\documentclass[lettersize,onecolumn]{IEEEtran}

\usepackage[T1]{fontenc}

\usepackage{tcolorbox}
\usepackage{amsmath}
\usepackage{import}
\usepackage{stackengine}
\usepackage{algorithm}
\usepackage{algpseudocode}
\usepackage{ushort}
\usepackage{cite}

\usepackage{booktabs} 
\usepackage{todonotes}
\usepackage{nicefrac} 
\usepackage{hyperref}

\usepackage{microtype}     
\usepackage{xcolor}        

\usepackage{enumitem}
\usepackage{bm}
\usepackage{amsmath}
\usepackage{amsthm}
\usepackage{caption}
\usepackage{amssymb}
\usepackage{subcaption}
\usepackage{wrapfig}
\usepackage{import}
\usepackage[usenames,dvipsnames]{pstricks}
\usepackage{pstricks-add}
\usepackage{epsfig}
\usepackage{pst-grad} 
\usepackage{pst-plot} 
\usepackage[space]{grffile}
\usepackage{etoolbox}

\usepackage{tikz}
\usepackage{amsmath, amssymb}
\usetikzlibrary{arrows.meta, positioning, shapes.geometric}

\theoremstyle{plain}
\newtheorem{definition}{Definition}
\newtheorem{theorem}{Theorem}
\theoremstyle{definition}

\theoremstyle{definition}
\newtheorem{remark}{Remark}
\theoremstyle{plain}

\newtheorem{proposition}{Proposition}
\newtheorem{lemma}{Lemma}
\newtheorem{corollary}{Corollary}

\newtheorem{assumption}{Assumption}
\newtheorem{example}{Example}

\def\enc{{u_\textrm{enc}}}
\def\dec{{u_\textrm{dec}}}

\def\encstar{{u^*_\textrm{enc}}}
\def\decstar{{u^*_\textrm{dec}}}

\def\encdagger{{u^\dagger_\textrm{enc}}}

\def\enclamb{\lambda_\textrm{e}}
\def\declamb{\lambda_\textrm{d}}

\def\lenc{\mathcal{L}_\textrm{enc}}
\def\ldec{\mathcal{L}_\textrm{dec}}

\def\gcc{\texttt{GCC}}
\def\stset{\mathcal{F}}
\def\fhat{\widehat{f}}

\def\objwc{J_{\textrm{wc}}}
\def\objp{J_{\textrm{p}}}

\newcommand\delmax[1]{\bar{\delta}_{#1}}
\newcommand\delmin[1]{\ushort{\delta}_{#1}}

\newcommand{\norm}[1]{\left\lVert#1\right\rVert}
\newcommand{\lp}[1]{{L}^{#1}\left(\Omega; \mathbb{R}\right)}
\newcommand{\lpm}[1]{{L}^{#1}\left(\Omega; \mathbb{R}^M\right)}
\newcommand{\lpd}[2]{{L}^{#1}\left(\Omega; \mathbb{R}^{#2}\right)}

\newcommand{\smp}{\mathbb{W}^{m,q}\left(\Omega; \mathbb{R}^M\right)}

\newcommand{\smpeq}{\widetilde{\mathbb{W}}^{m,q}\left(\Omega; \mathbb{R}^M\right)}

\newcommand{\soblocm}[2]{\mathbb{W}_{\textrm{loc}}^{#1,#2}\left(\Omega; \mathbb{R}^M\right)}
\newcommand{\soblocmd}[3]{\mathbb{W}_{\textrm{loc}}^{#1,#2}\left(\Omega; \mathbb{R}^{#3}\right)}

\newcommand{\sobeq}[2]{\widetilde{\mathbb{W}}^{#1,#2}\left(\Omega; \mathbb{R}\right)}
\newcommand{\sobm}[2]{\mathbb{W}^{#1,#2}\left(\Omega; \mathbb{R}^M\right)}

\newcommand{\sobmeq}[2]{\widetilde{\mathbb{W}}^{#1,#2}\left(\Omega; \mathbb{R}^M\right)}

\newcommand{\hil}[1]{\mathcal{H}^{#1}\left(\Omega;\mathbb{R}\right)}
\newcommand{\hilz}[1]{\mathcal{H}_0^{#1}\left(\Omega;\mathbb{R}\right)}
\newcommand{\hilm}[1]{\mathcal{H}^{#1}\left(\Omega;\mathbb{R}^M\right)}
\newcommand{\hild}[2]{\mathcal{H}^{#1}\left(\Omega;\mathbb{R}^{#2}\right)}

\newcommand{\hiltilde}[1]{\widetilde{\mathcal{H}}^{#1}\left(\Omega;\mathbb{R}\right)}
\newcommand{\hilmtilde}[1]{\widetilde{\mathcal{H}}^{#1}\left(\Omega;\mathbb{R}^M\right)}

\newcommand{\splineCC}[3]{\mathbf{S}_{#1,#2,#3}}
\newcommand{\lec}[2]{\stackrel{\text{#1}}{#2}}

\newtcolorbox{mathbox}[1][]{colback=gray!20,  #1}
\tcbset{width=(\linewidth-4mm),
before=,after=\hfill,colframe=black!75!black,colback=black!5!gray}

\begin{document}

\title{Learning-Theoretic Foundation for General Coded Computing: The Straggler Setting}


\author{
    \IEEEauthorblockN{Parsa Moradi}
    \IEEEauthorblockA{
    University of Minnesota, Twin Cities
    moradi@umn.edu\\}
    \and
    \IEEEauthorblockN{Behrooz Tahmasebi}
    \IEEEauthorblockA{
    Harvard University
    behrooz\_tahmasebi@seas.harvard.edu\\}
    \and
    \IEEEauthorblockN{Mohammad Ali Maddah-Ali}
    \IEEEauthorblockA{
    University of Minnesota, Twin Cities
    maddah@umn.edu}

\thanks{
  The work has been partially supported by the National Science Foundation under Grant CIF-2348638.
  }
  \thanks{
  This work was partially presented at the 2024 Conference on Neural Information Processing Systems (NeurIPS 2024), held in Vancouver, Canada, from December 10 to 15, 2024, at the 2025 IEEE International Symposium on Information Theory (ISIT 2025), held in Ann Arbor, Michigan, USA, from June 22 to 27, 2025, and in a survey paper in~\cite{GCCBITS}. 
    }
}


\maketitle

\begin{abstract}
Coded computing has emerged as a powerful paradigm for mitigating the impact of straggling workers in distributed computing systems. However, existing coded-computing schemes are predominantly designed for the exact recovery of highly structured computations, such as polynomial evaluation and matrix multiplication, and typically rely on strict recovery thresholds. These assumptions significantly limit their applicability to modern machine-learning workloads, particularly deep neural networks (DNNs), whose computations generally lack rigid algebraic structure and, in many applications, require only accurate approximations rather than exact recovery.

To address this gap, we revisit coded computing from a learning-theoretic perspective and introduce General Coded Computing (\gcc{}). Rather than adopting existing algebraic tools, \gcc{} formulates coded computing through a natural end-to-end mean-squared error loss that directly measures the discrepancy between the desired computations and their recovered estimates. By deriving suitable upper bounds and restricting the encoder and decoder to a reproducing kernel Hilbert space (RKHS) with mild smoothness constraints, we show that both the encoder and decoder admit specific representations as linear combinations of RKHS kernel functions. This representation allows the corresponding coefficients to be computed efficiently. Moreover, this framework enables us to establish theoretical performance guarantees for \gcc{} under two complementary straggler regimes. In the worst-case setting with $N$ worker nodes, and at most $S$ stragglers, we show that the end-to-end loss decays at least at rate $\mathcal{O}(S^3N^{-3})$ for standard configurations. We then study a probabilistic setting in which each worker independently straggles with probability $p$. We prove that the expected loss can still converge at rate $\mathcal{O}(\log_{1/p}^3(N)N^{-3})$. Extensive experiments on complex DNN architectures with millions of parameters validate the theoretical predictions and show that \gcc{} achieves lower reconstruction error and faster convergence than state-of-the-art coded-computing schemes.
\end{abstract}


\IEEEpeerreviewmaketitle

\section{Introduction}\label{sec:introduction}

Distributed computing has become an indispensable technology for managing massive datasets and evaluating increasingly complex functions, such as Machine Learning (ML) models. Typically, these systems operate under a master-worker architecture, where a central \emph{master node} delegates computational tasks to a cluster of \emph{worker nodes}. However, distributing these tasks across large clusters introduces systemic vulnerabilities, most notably the presence of \emph{stragglers}, worker nodes that experience unpredictable delays or fail to return their results within a prescribed deadline. Traditional naive distribution strategies, which assign data partitions to individual workers without redundancy, are bottlenecked by the slowest nodes, leading to severe resource underutilization.

To address this bottleneck, \emph{coded computing} was introduced, drawing inspiration from error-correcting codes in communications~\cite{Abraham1984, lee2017speeding, yu2019lagrange, yu2020straggler,short,  opt-recovery, jia2019capacity}. By deliberately injecting structured redundancy into the distributed data, the master node can recover the exact computational results using only the outputs from a subset of the workers. While highly successful, classical coded computing was fundamentally developed for algebraic tasks over finite fields and highly structured computations, such as matrix multiplication and polynomial evaluations~\cite{lee2017speeding, yu2019lagrange,yu2020straggler,yu2017polynomial,short, jahani2018codedsketch, opt-recovery}. Additionally, these methods guarantee exact recovery but impose a strict \emph{recovery threshold}: if the number of non-straggling worker nodes falls below this threshold, the decoding procedure fails entirely. This rigidity renders classical schemes ill-suited for more general and complex computations, such as deep neural networks (DNNs), which lack the structured algebraic properties and only require good approximation. 

To extend the reach of coded computing to a broader class of functions, several research directions have emerged. Initial efforts attempted to fit general computations into the classical framework by approximating non-polynomial target functions with polynomial surrogates~\cite{so2020scalable,codedpri}. However, because high-degree polynomial interpolation over the real field is susceptible to numerical instability, subsequent studies focused heavily on refining the underlying coding mechanisms, such as utilizing alternative polynomial bases or analog codes~\cite{AnooshehRobust,RamamoorthyCir,RamamoorthyConv,soleymani2020analog,fahim2019numerically}. More recently, recognizing that many applications, including machine learning workloads, inherently tolerate small computational perturbations, a paradigm shift toward \emph{approximate coded computing} has gained traction~\cite{jahani2018codedsketch, e-app-coded, overSketch,overSketchN, SuccApp}. By intentionally trading exact recovery for bounded approximation errors, these approaches eliminate the rigid recovery threshold constraint. This yields a more flexible system where the master node can compute a good approximation from any subset of workers, and the accuracy of the final computation improves as more worker nodes return their results.

Despite this paradigm shift, these attempts ultimately fail to bridge the gap between classical coded computing and general distributed computing systems. The root cause of this disconnect lies in the foundational reliance of classical coded computing on algebraic coding theory, a paradigm inherently incompatible with general, high-dimensional functions. For instance, while polynomial approximations are useful for simple non-linear functions, such as the one used in~\cite{codedpri} for the sigmoid function in the logistic regression setting, it is both theoretically and practically infeasible to approximate the complex, highly non-linear behavior of a deep neural network using a single global polynomial with acceptable accuracy. As a result, although approximating with polynomials or refining polynomial bases provides some flexibility, these methods still rely on algebraic structures that simply do not scale to arbitrary real-valued functions.

This mismatch motivates a re-examination of the design principles underlying coded computing. The work in~\cite{jahani2022berrut} takes an initial step in this direction by moving away from strictly algebraic codes and leveraging tools from approximation theory to handle general computation tasks. However, its encoder and decoder are not explicitly designed to optimize the end-to-end coded-computing objective. Instead, their design is largely guided by off-the-shelf approximation-theoretic tools without a direct mechanism to minimize the error between the recovered function and the true computational target. Consequently, developing a unified approach that directly optimizes the approximation error for arbitrary, high-dimensional continuous functions remains an open challenge.

To fill this critical gap, we propose General Coded Computing (\gcc{}), a novel framework that bridges the divide between coded computing and general distributed computing. Unlike previous approaches that rely on rigid algebraic structures or off-the-shelf approximation techniques, \gcc{} is built around an explicit end-to-end loss function aligned with the central objective of coded computing: accurately approximating the target function values on a batch of input data from the results returned by a subset of worker nodes.

To systematically minimize this objective, we restrict the encoder and decoder to a reproducing kernel Hilbert space (RKHS) of second-order Sobolev functions. This functional-analytic formulation enables the use of tools from learning theory and spline approximation. Instead of solving the original joint optimization problem directly, we derive a tractable upper bound on the end-to-end loss. This bound decomposes into two interpretable terms: a decoder approximation term, which measures how accurately the decoder reconstructs the target computation from the returned worker outputs, and an encoder approximation term, which measures how faithfully the encoder preserves the original input batch. The decoder approximation term, however, depends jointly on both the encoder and the decoder. By invoking the Representer Theorem, we show that both the encoder and decoder admit explicit representations as linear combinations of RKHS kernel functions, with coefficients that can be computed efficiently in closed form. Consequently, although the code design is grounded in learning theory, neither the encoder nor the decoder requires training.

In summary, the main contributions of this paper are as follows:
\begin{itemize}
    \item \textbf{A learning-theoretic formulation of coded computing:}
    We introduce General Coded Computing (\gcc{}), a framework for coded computing beyond algebraically structured tasks. The framework is built around an explicit end-to-end loss function that directly captures the central coded-computing objective: accurately approximating the target function values on a batch of inputs from the results returned by a subset of worker nodes. By formulating the encoder and decoder as functions in an RKHS of second-order Sobolev functions, we obtain a principled approach for designing efficient encoding and decoding procedures through tractable upper bounds on this loss.

    \item \textbf{Theoretical guarantees and convergence rates:}
    We characterize the end-to-end loss of \gcc{} under two complementary straggler regimes. In the worst-case setting with at most $S$ stragglers, we prove that, under standard node configurations, the loss decays at least at rate $\mathcal{O}(S^3/N^3)$. We also analyze a probabilistic setting in which each worker independently straggles with probability $p$. We prove that the expected loss still can vanish at rate $\mathcal{O}(\log_{1/p}^3(N)/N^3)$.

    \item \textbf{Extensive Empirical Evaluation:} We validate the \gcc{} framework across a diverse spectrum of computational tasks, ranging from classical structured target functions, such as polynomials, to highly complex, state-of-the-art deep neural networks, including Vision Transformers (ViTs)~\cite{dosovitskiy2020image}. The numerical evaluations strongly support our theoretical derivations, demonstrating that \gcc{} achieves superior recovery accuracy and faster convergence compared to existing state-of-the-art coded computing schemes.
\end{itemize}

The remainder of this paper is organized as follows. Section~\ref{sec:prob_form} formulates the distributed computing problem and introduces the \gcc{} framework. Section~\ref{sec:prelim} presents the mathematical preliminaries used throughout the paper. Sections~\ref{sec:worst_case} and~\ref{sec:prob_straggler} develop the main theoretical results under the worst-case and probabilistic straggler models, respectively. Section~\ref{sec:proofs} provides the proofs of the main results. Section~\ref{sec:exp_result} presents the experimental setup and results. Section~\ref{sec:conclusion} concludes the paper and discusses future research directions. Appendices~\ref{app:defs}, \ref{app:sobolev_props}, and~\ref{app:smoothspline} collect supplementary definitions and mathematical tools used in the theoretical analysis, while Appendix~\ref{app:proof_lemmas} provides the proofs of the auxiliary lemmas used in the proof of the main theorems.

\subsection{Notation}
Throughout this paper, scalars are denoted by non-boldface lowercase letters
(e.g., $x$), vectors by boldface lowercase letters (e.g., $\mathbf{x}$), and
matrices by boldface uppercase letters (e.g., $\mathbf{A}$). Sets are denoted
by calligraphic uppercase letters (e.g., $\mathcal{S}$). Coded quantities are
indicated by a tilde; for example, $\widetilde{\mathbf{x}}$ denotes a coded
vector and $\widetilde{\mathbf{A}}$ denotes a coded matrix. For
$n\in\mathbb{N}$, we use the shorthand $[n]:=\{1,2,\ldots,n\}$. For a finite
set $\mathcal{S}$, its cardinality is denoted by $|\mathcal{S}|$.

For a differentiable scalar-valued function
$f:\mathbb{R}^d\to\mathbb{R}$, the gradient at $\mathbf{x}$ is denoted by
$\nabla f(\mathbf{x})\in\mathbb{R}^d$, and the Hessian is denoted by
$D^2 f(\mathbf{x})\in\mathbb{R}^{d\times d}$. For a vector-valued function
$f=[f_1,\ldots,f_m]^T:\mathbb{R}^d\to\mathbb{R}^m$, $\nabla f_j(\mathbf{x})$
and $D^2 f_j(\mathbf{x})$ denote the gradient and Hessian of its $j$-th
component, respectively.

For a function $g:\Omega\to\mathbb{R}^M$ defined on an interval
$\Omega\subseteq\mathbb{R}$, we denote by $g^{(i)}$ its $i$-th weak derivative (see Definition~\ref{def:weak}) with respect to the scalar argument, taken component-wise. In particular,
$g^{(0)}:=g$, and we use the shorthand $g':=g^{(1)}$ and
$g'':=g^{(2)}$. The Euclidean norm is denoted by $\|\cdot\|_2$, and the operator norm of a matrix $\mathbf{A}$ is
\(
    \|\mathbf{A}\|_{\mathrm{op}}
    :=
    \sup_{\|\mathbf{x}\|_2=1}\|\mathbf{A}\mathbf{x}\|_2 .
\)
For a measurable set $\Omega\subseteq\mathbb{R}$, its Lebesgue measure is
denoted by $\operatorname{Leb}(\Omega)$.

\section{Problem Formulation}\label{sec:prob_form}

Consider a distributed computing system consisting of a master node and $N$ worker nodes. The master node aims to compute a batch of function evaluations $\{f(\mathbf{x}_k)\}_{k=1}^K$, where each input data point satisfies $\mathbf{x}_k \in \mathbb{R}^d$ and the target function is $f:\mathbb{R}^d \to \mathbb{R}^m$. The parameters $K,d,m\in\mathbb{N}$ denote the batch size, input dimension, and output dimension, respectively.

A key challenge in such systems is the presence of \emph{straggling} workers, namely workers that fail to complete their assigned computations within a prescribed deadline. Consequently, a naive strategy that assigns each input data point to a single worker is highly vulnerable to delays and can lead to inefficient utilization of the distributed resources.

To mitigate this issue, the master node introduces redundancy through coding. Instead of directly assigning the original input data points to the worker nodes, the master node employs an encoding procedure to generate $N$ coded data points, where each coded data point is constructed from the input batch. Each worker then evaluates the function $f(\cdot)$ on its assigned coded data point and transmits the computed result, referred to as the \emph{coded output}, back to the master node. Upon receiving the outputs from the non-straggling workers, the master node applies a decoding procedure to recover approximations $\fhat(\mathbf{x}_k)$ of the desired evaluations $f(\mathbf{x}_k)$ for $k\in[K]$. The redundancy induced by the encoding procedure enables the master node to estimate the desired function evaluations $\{f(\mathbf{x}_k)\}_{k=1}^K$ even when a subset of workers are stragglers and their computation results are unavailable.

\subsection{General Coded Computing (\gcc{})}\label{sec:framework}

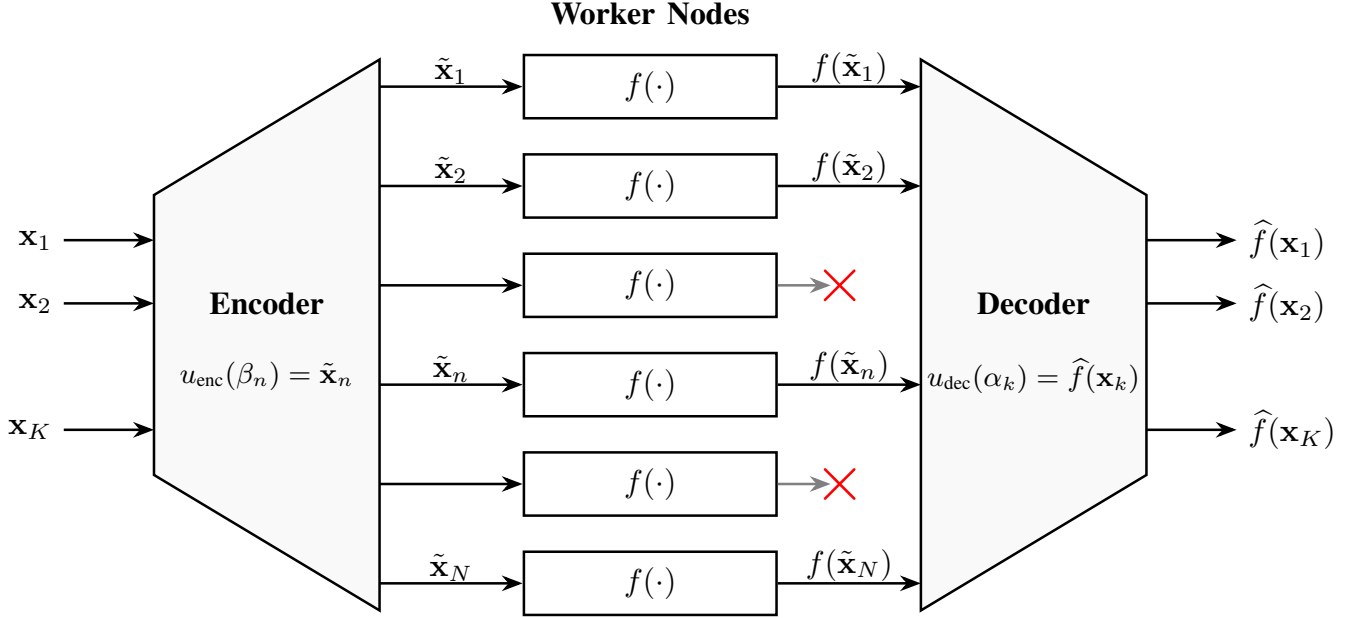
\begin{figure}[!t]
\centering
\resizebox{\columnwidth}{!}{
\begin{tikzpicture}[
    >=Stealth,
    node distance=1cm,
    block/.style={rectangle, draw, minimum width=2.8cm, minimum height=0.7cm, align=center, thick},
    label style/.style={font=\small}
]

    \draw[thick, fill=gray!5] (0,0) -- (2.5,1.5) -- (2.5,-4.6) -- (0,-3.1) -- cycle;
    \node at (1.25, -1.2) {\textbf{Encoder}};
    \node at (1.25, -2) {\small $u_{\text{enc}}(\beta_n) = \tilde{\mathbf{x}}_n$};

    \node[block] (w1) at (5.5, 1.2)  {$f(\cdot)$};
    \node[block] (w2) at (5.5, 0.1)  {$f(\cdot)$};
    \node[block] (w3) at (5.5, -1.0) {$f(\cdot)$};
    \node[block] (wn) at (5.5, -2.1) {$f(\cdot)$};
    \node[block] (w5) at (5.5, -3.2) {$f(\cdot)$};
    \node[block] (wN) at (5.5, -4.3) {$f(\cdot)$};

    \node[above=0.2cm of w1, font=\bfseries] {Worker Nodes};

    \begin{scope}[shift={(8.5,0)}]
        \draw[thick, fill=gray!5] (0,1.5) -- (2.5,0) -- (2.5,-3.1) -- (0,-4.6) -- cycle;
        \node at (1.25, -1.2) {\textbf{Decoder}};
        \node at (1.25, -2) {\small $u_{\text{dec}}(\alpha_k) = \widehat{f}(\mathbf{x}_k)$};
    \end{scope}

    \draw[->, thick] (-1, -0.5) node[left]{$\mathbf{x}_1$} -- (0, -0.5);
    \draw[->, thick] (-1, -1.2) node[left]{$\mathbf{x}_2$} -- (0, -1.2);
    \draw[->, thick] (-1, -2.6) node[left]{$\mathbf{x}_K$} -- (0, -2.6);

    \draw[->, thick] (2.5, 1.2)  -- node[above, yshift=-1mm]{$\tilde{\mathbf{x}}_1$} (w1.west);
    \draw[->, thick] (2.5, 0.1)  -- node[above, yshift=-1mm]{$\tilde{\mathbf{x}}_2$} (w2.west);
    \draw[->, thick] (2.5, -1.0) -- (w3.west);
    \draw[->, thick] (2.5, -2.1) -- node[above, yshift=-1mm]{$\tilde{\mathbf{x}}_n$} (wn.west);
    \draw[->, thick] (2.5, -3.2) -- (w5.west);
    \draw[->, thick] (2.5, -4.3) -- node[above, yshift=-1mm]{$\tilde{\mathbf{x}}_N$} (wN.west);

    \draw[->, thick] (w1.east) -- node[above, yshift=-1mm]{$f(\tilde{\mathbf{x}}_1)$} (8.5, 1.2);
    \draw[->, thick] (w2.east) -- node[above, yshift=-1mm]{$f(\tilde{\mathbf{x}}_2)$} (8.5, 0.1);
    
    \draw[->, thick, gray] (w3.east) -- (7.5, -1.0);
    \node[red, thick] at (7.6, -1.0) {\huge $\times$};
    
    \draw[->, thick] (wn.east) -- node[above, yshift=-1mm]{$f(\tilde{\mathbf{x}}_n)$} (8.5, -2.1);
    
    \draw[->, thick, gray] (w5.east) -- (7.5, -3.2);
    \node[red, thick] at (7.6, -3.2) {\huge $\times$};
    
    \draw[->, thick] (wN.east) -- node[above, yshift=-1mm]{$f(\tilde{\mathbf{x}}_N)$} (8.5, -4.3);

    \draw[->, thick] (11, -0.5) -- (12, -0.5) node[right]{$\widehat{f}(\mathbf{x}_1)$};
    \draw[->, thick] (11, -1.2) -- (12, -1.2) node[right]{$\widehat{f}(\mathbf{x}_2)$};
    \draw[->, thick] (11, -2.6) -- (12, -2.6) node[right]{$\widehat{f}(\mathbf{x}_K)$};

\end{tikzpicture}
 }
\caption{System model for coded computing with $N$ worker nodes and $K$ input tasks, where some nodes act as stragglers.}
\label{fig:framework}
\end{figure}

We now introduce \gcc{}, a straggler-resilient distributed computing framework for a general high-dimensional computation. As illustrated in Figure~\ref{fig:framework}, \gcc{} consists of three stages: an \emph{encoding} stage at the master node, a \emph{computation} stage across the worker nodes, and a \emph{decoding} stage at the master node. The main idea is to embed the input batch $\{\mathbf{x}_k\}_{k=1}^K$ into a continuous curve, evaluate this curve at worker-specific points to generate coded data points, and then reconstruct the desired outputs from the returned worker results through a decoding procedure. 

Formally, the scheme operates as follows.
\begin{enumerate}[label={(\arabic*)}]
\setlength\itemsep{0.1em}
    \item \emph{Encoding:} The master node selects $K$ points $\alpha_1 < \cdots < \alpha_K$ in a bounded interval $\Omega \subset \mathbb{R}$, referred to as the \emph{encoder design points}. It then constructs an encoder function
    \(
        \enc:\Omega \to \mathbb{R}^d
    \)
    such that
    \begin{align}
        \label{eq:enc_approx}
        \enc(\alpha_k) \approx \mathbf{x}_k, \quad k \in [K].
    \end{align}
    The approximation error in \eqref{eq:enc_approx} will be characterized in the subsequent theoretical analysis (see Section~\ref{sec:worst_code_design}). Next, the master node selects another set of $N$ points $\beta_1 < \cdots < \beta_N$ in $\Omega$, referred to as the \emph{decoder design points}. It then evaluates the encoder function at these points to generate the coded data points
    \[
        \widetilde{\mathbf{x}}_n := \enc(\beta_n) \in \mathbb{R}^d, \quad n \in [N].
    \]
    The master node assigns the coded data point $\widetilde{\mathbf{x}}_n$ to worker node $n$. Since $\enc(\cdot)$ is constructed from $\{\mathbf{x}_k\}_{k=1}^K$, each coded data point $\widetilde{\mathbf{x}}_n$ generally depends on the entire input batch.

    \item \emph{Computation:} Worker node $n \in [N]$ evaluates the target function at its assigned coded data point and returns
    \[
        f(\widetilde{\mathbf{x}}_n)=f(\enc(\beta_n)) \in \mathbb{R}^m
    \]
    to the master node.

    \item \emph{Decoding:} Let $\mathcal{F} \subseteq [N]$ denote the set of non-straggling workers for the current batch. Using the received results $\{f(\enc(\beta_v))\}_{v \in \mathcal{F}}$, the master node constructs a decoder function
    \(
        \dec:\Omega \to \mathbb{R}^m
    \)
    intended to approximate the composite map $z \mapsto f(\enc(z))$, namely,
    \begin{align}\label{eq:dec_approx}
        \dec(z) \approx f(\enc(z)), \quad z \in \Omega.
    \end{align}
    The approximation error in \eqref{eq:dec_approx} will also be characterized in the subsequent theoretical analysis (see Section~\ref{sec:worst_code_design}). Consequently, given the non-straggler set $\mathcal{F}$, the master node forms the estimate $\fhat_{\mathcal{F}}(\mathbf{x}_k)$ by evaluating the decoder at the encoder design points $\alpha_k$ for $k \in [K]$:
\begin{equation}\label{eq:gcc_approx}
        \fhat_{\mathcal{F}}(\mathbf{x}_k)
        :=
        \dec(\alpha_k)
        \overset{\mathrm{(a)}}{\approx}
        f(\enc(\alpha_k))
        \overset{\mathrm{(b)}}{\approx}
        f(\mathbf{x}_k),
    \end{equation}
    where step (a) follows from the decoder approximation in \eqref{eq:dec_approx}, and step (b) follows from the encoder approximation in \eqref{eq:enc_approx}.
\end{enumerate}

Our goal is to make the resulting approximations as accurate as possible. To this end, for a non-straggler set $\mathcal{F}$, we quantify the performance of the scheme as the mean-squared error (MSE) of the approximations:
\begin{align}\label{eq:gen_obj}
    \mathcal{L}( \enc, \dec)
    =
    \frac{1}{K}
    \sum_{k=1}^K
    \left\|
        \fhat_{\mathcal{F}}(\mathbf{x}_k) - f(\mathbf{x}_k)
    \right\|_2^2.
\end{align}

In the subsequent sections, after establishing the necessary preliminaries (Section~\ref{sec:prelim}), we study two distinct straggler models:
\begin{enumerate}
    \item \textbf{Worst-Case Straggler Setting} (Section~\ref{sec:worst_case}): In this setting, we assume that at most $S \leq N$ workers straggle. The goal is to design the encoder and decoder so as to minimize the objective in \eqref{eq:gen_obj} under the worst-case straggling pattern over all non-straggler sets $\mathcal{F} \subseteq [N]$ satisfying $|\mathcal{F}| \geq N-S$. Formally, the objective is
    \begin{align}\label{eq:obj_opt_worst_0}
 \inf_{\enc}\ 
        \max_{|\mathcal{F}| \geq N-S} \inf_{\dec}\
        \mathcal{L}( \enc, \dec).
    \end{align}
    
    \item \textbf{Probabilistic Straggler Setting} (Section~\ref{sec:prob_straggler}): We assume that each worker independently becomes a straggler with probability $p \in (0,1)$. Equivalently, each worker belongs to the non-straggler set independently with probability $1-p$. The goal is to minimize the expected MSE over the distribution of the non-straggling worker set:
    \begin{align}
\label{eq:obj_opt_prob_0}
\inf_{\enc}\ 
        \mathbb{E}_{\mathcal{F} \sim \mathsf{F}_{p,N}}
        \big[ \inf_{\dec}
            \mathcal{L}( \enc, \dec)
        \big],
    \end{align}
    where $\mathsf{F}_{p,N}$ denotes the distribution over subsets of $[N]$ from independent worker failures with probability $p$.
\end{enumerate}
Thus, the main objective is to systematically design the encoder and decoder functions, $\enc(\cdot)$ and $\dec(\cdot)$, and to characterize the resulting approximation errors in \eqref{eq:enc_approx} and \eqref{eq:dec_approx} according to the performance objectives of worst-case and probabilistic settings characterized in \eqref{eq:obj_opt_worst_0} and \eqref{eq:obj_opt_prob_0}, respectively. 

\begin{remark}
    The order of the optimizations in~\eqref{eq:obj_opt_worst_0} and~\eqref{eq:obj_opt_prob_0} reflects the fact that the decoder may depend on the set of non-straggling workers $\mathcal{F}$, whereas the encoder cannot. Indeed, at the decoding stage, $\mathcal{F}$ is known because the master has observed which workers have responded. In contrast, the encoder is designed before the straggler pattern is realized and therefore must be independent of $\mathcal{F}$.
\end{remark}

\section{Preliminaries}\label{sec:prelim}

Let $\Omega \subset \mathbb{R}$ be a bounded interval, and let $M\in\mathbb{N}$ be fixed.  For $1 \le \rho < \infty$, let $\lpm{\rho}$ denote the space of measurable functions $g:\Omega \to \mathbb{R}^M$ whose components are $\rho$-integrable over $\Omega$. Similarly, let $\lpm{\infty}$ denote the space of essentially bounded functions on $\Omega$; see Appendix~\ref{app:defs} for formal definitions.

Next, we define the Sobolev space $\sobm{r}{\rho}$, which consists of functions whose weak derivatives up to order $r$ belong to the corresponding $L^\rho$ space.

\begin{definition}[Sobolev space $\sobm{r}{\rho}$]\label{def:sobolev}
Let $r \in \mathbb{N}$ with $r \ge 1$, let $1 \le \rho \le \infty$, and let $\Omega \subset \mathbb{R}$ be a bounded interval. The Sobolev space $\sobm{r}{\rho}$ consists of all functions $g \in \lpm{\rho}$ whose weak derivatives $g^{(i)}$, $i=1,\dots,r$, exist and also belong to $\lpm{\rho}$; see Definition~\ref{def:weak} for the definition of weak derivatives. For notational convenience, we set $g^{(0)} := g$.

For $1 \le \rho < \infty$, the space $\sobm{r}{\rho}$ is equipped with the norm
\begin{align}\label{eq:sob_norm1}
    \norm{g}_{\sobm{r}{\rho}}
    :=
    \left(
        \sum_{i=0}^r
        \norm{g^{(i)}}_{\lpm{\rho}}^\rho
    \right)^{1/\rho}.
\end{align}
For $\rho=\infty$, it is equipped with the norm
\begin{align}\label{eq:sob_norm2}
    \norm{g}_{\sobm{r}{\infty}}
    :=
    \max_{i \in \{0,\dots,r\}}
    \norm{g^{(i)}}_{\lpm{\infty}}.
\end{align}
\end{definition}

For $\rho=2$, the Sobolev space $\sobm{r}{2}$ is a Hilbert space. Under the conditions stated in Proposition~\ref{prop:sob_hilb} below, this Hilbert space is also a reproducing kernel Hilbert space (RKHS).
For background on RKHS, the reader is referred to~\cite{leoni2024first,adams2003sobolev}. The RKHS structure is important in our setting because it enables the use of the representer theorem, stated in Theorem~\ref{thm:representer}, which plays a key role in the design of the encoder and decoder functions.

\begin{proposition}\label{prop:sob_hilb}
\cite[Section~7.2]{leoni2024first}, \cite[Theorem~121]{berlinet2011reproducing}, \cite{wahba1990spline}
For any bounded interval $\Omega \subseteq \mathbb{R}$ and any $r,M \in \mathbb{N}$ with $r\ge 1$, the Sobolev space
\begin{align}
    \hilm{r} := \sobm{r}{2}
\end{align}
is a reproducing kernel Hilbert space.
\end{proposition}
Throughout the paper, we use $\hilm{r}$ to denote the Sobolev space of order $r$. 
As an RKHS, $\hilm{r}$ has an associated positive-definite reproducing kernel~\cite{leoni2024first,adams2003sobolev} 
\begin{align}
\Phi:\Omega\times\Omega \to \mathbb{R}^{M\times M}.
\end{align}
Here, the kernel is matrix-valued because the functions in $\hilm{r}$ take values in $\mathbb{R}^M$. The representer theorem states that, for a broad class of regularized optimization problems over an RKHS, the minimizer can be written as a finite linear combination of kernel evaluations at the design points. We state this result next.

\begin{theorem}[Representer theorem \cite{scholkopf2001generalized,micchelli2005learning}]\label{thm:representer}
Let $\mathcal{H}$ be an RKHS of functions from $\Omega$ to $\mathbb{R}^M$ with reproducing kernel $\Phi:\Omega \times \Omega \to \mathbb{R}^{M \times M}$. Let $t_1,\dots,t_n \in \Omega$ be given design points, and let
\[
    J(g)
    =
    \Psi\big(g(t_1),\dots,g(t_n)\big)
    +
    \lambda\,\Theta\big(\norm{g}_{\mathcal H}\big),
\]
where $\lambda>0$, $\Psi:(\mathbb{R}^M)^n \to \mathbb{R}\cup\{+\infty\}$ is a functional depending only on the values of $g$ at $\{t_i\}_{i=1}^n$, and $\Theta:[0,\infty)\to\mathbb{R}$ is strictly increasing. Then every minimizer  of
\[
    \min_{g\in\mathcal H} J(g)
\]
admits a representation of the form
\begin{align}\label{eq:representer}
    g^*(t)
    =
    \sum_{i=1}^n
    \Phi(t,t_i)\mathbf{w}_i,
\end{align}
for some coefficient vectors $\mathbf{w}_1,\dots,\mathbf{w}_n \in \mathbb{R}^M$.
\end{theorem}

\begin{remark}\label{rem:rep_vec}
When $\mathcal{H}=\hild{2}{M}$, the corresponding matrix-valued reproducing kernel takes the form
\begin{align}
    \Phi(t,s)
    =
    \phi(t,s)\,I_M,
\end{align}
where $\phi$ is the scalar reproducing kernel associated with $\hild{2}{}$, and $I_M$ is the $M\times M$ identity matrix~\cite{aubin2000applied,li2024towards}. Hence, the representation in Theorem~\ref{thm:representer} reduces to
\begin{align}
    g^*(t)
    =
    \sum_{i=1}^n
    \phi(t,t_i)\mathbf{w}_i,
\end{align}
for some coefficient vectors $\mathbf{w}_1,\dots,\mathbf{w}_n \in \mathbb{R}^M$.
\end{remark}

Finally, we introduce two quantities that characterize the spacing of an ordered finite set of points in an interval.

\begin{definition}\label{def:max_dist_st}
Let $\Omega = (a,b) \subset \mathbb{R}$, and let $\mathcal{T} = \{t_1,t_2,\dots,t_J\} \subset \Omega$ be an ordered finite set with $J\ge 2$ and
\[
    a < t_1 < t_2 < \cdots < t_J < b.
\]
Define the augmented sequence by setting $t_0=a$ and $t_{J+1}=b$. The maximum and minimum consecutive distances of $\mathcal{T}$ are defined, respectively, as
\begin{align}
    \delmax{\mathcal{T}}
    &:=
    \max_{j \in \{0,\dots,J\}}
    \left( t_{j+1} - t_j \right), \\
    \delmin{\mathcal{T}}
    &:=
    \min_{j \in \{1,\dots,J-1\}}
    \left( t_{j+1} - t_j \right).
\end{align}
\end{definition}

\begin{example}\label{ex:spacing_points}
For the domain $\Omega=(-1,1)$, two common choices of design points are equidistant points and Chebyshev points. For $J$ equidistant points, we have
\[
    t_j=-1+\frac{2j}{J+1}, \qquad j\in[J].
\]
Then $\delmax{\mathcal{T}}=\delmin{\mathcal{T}}=2/(J+1)$. For $J$ Chebyshev points of the first kind, ordered increasingly, we have
\[
    t_j=-\cos\!\left(\frac{(2j-1)\pi}{2J}\right),
    \qquad j\in[J].
\]
Then one can verify that $\delmax{\mathcal{T}}=\mathcal{O}(J^{-1})$ and $\delmin{\mathcal{T}}=\mathcal{O}(J^{-2})$.
\end{example}

In the remainder of the paper, we restrict both the encoder and decoder to the second-order Sobolev spaces. Specifically, we assume
\[
    \enc \in \hild{2}{d},
    \qquad
    \dec \in \hild{2}{m}.
\]
Equivalently, the encoder and decoder, together with their weak derivatives up to second order, are square-integrable on $\Omega$ (see Definition~\ref{def:sobolev}). This design choice imposes a mild regularity condition while providing powerful theoretical tools for encoder and decoder design and error analysis through the associated RKHS structure.

We further set
\begin{align}
    \Omega = (-1,1)
\end{align}
as the common domain for both the encoder and decoder. This normalization is without loss of generality, since any bounded interval can be mapped to $(-1,1)$ through an affine change of variables. Such a rescaling may change the constants in the theoretical bounds, but it does not affect the convergence rates. Therefore, the encoder design points $\{\alpha_k\}_{k=1}^K$ and decoder design points $\{\beta_n\}_{n=1}^N$ are chosen from $(-1,1)$.

\section{Worst-Case Straggler Setting}\label{sec:worst_case}

In this section, we study \gcc{} under the worst-case straggler model, in which at most $S<N$ worker nodes straggle. The objective is to design the encoder and decoder functions according to the following the worst-case optimization:
\begin{align}\label{eq:obj_opt_worst_case}
    \objwc
    :=
    \inf_{\substack{
        \enc \in \hild{2}{d}
    }}
    \ 
    \max_{|\mathcal{F}| \geq N-S}  \inf_{\substack{
        \dec \in \hild{2}{m}
    }}
    \mathcal{L}( \enc, \dec).
\end{align}
where $\mathcal{F}$ denotes the set of non-straggling workers.

In what follows, we first design the encoder and decoder for \gcc{} in the worst-case straggler setting in Section~\ref{sec:worst_code_design}. In Section~\ref{sec:worst_comp_analysis}, we analyze the master-side computational complexity of the resulting encoder and decoder and compare it with that of existing coded-computing schemes, including Berrut Approximation Coded Computing (BACC)~\cite{jahani2022berrut} and Lagrange Coded Computing (LCC)~\cite{yu2019lagrange}. Finally, in Section~\ref{sec:worst_asymp}, we present an asymptotic analysis of the worst-case end-to-end loss and derive an upper bound on the convergence rate of $\objwc$ as the number of workers $N$ tends to infinity.

\subsection{Code Design}\label{sec:worst_code_design}

In this subsection, we present the encoder and decoder design for the worst-case optimization problem in \eqref{eq:obj_opt_worst_case}. This optimization is highly challenging because it requires jointly optimizing the encoder and decoder over infinite-dimensional function spaces. A central step in our approach is to replace this intractable objective with a carefully constructed upper bound. This bound is not merely a technical relaxation; it reveals the structure of the problem by decomposing the loss into two interpretable, yet coupled, components. 

The first component captures the decoder approximation error, namely the error incurred when reconstructing the target computation from the returned worker outputs. The second component captures the encoder approximation error, namely the error caused when the encoder does not exactly reproduce the original input data at the encoder design points. We then optimize this upper bound over $\enc \in \hild{2}{d}$ and $\dec \in \hild{2}{m}$. To this end, for any non-straggler set $\mathcal{F}\subseteq[N]$, we have
\begin{align}
    \mathcal{L}( \enc, \dec)
    &=
    \frac{1}{K}
    \sum_{k=1}^K
    \left\|
        \fhat_{\mathcal{F}}(\mathbf{x}_k)-f(\mathbf{x}_k)
    \right\|_2^2
    \nonumber\\
    &\overset{\mathrm{(a)}}{=}
    \frac{1}{K}
    \sum_{k=1}^K
    \left\|
        \dec(\alpha_k)-f(\mathbf{x}_k)
    \right\|_2^2
    \nonumber\\
    &=
    \frac{1}{K}
    \sum_{k=1}^K
    \left\|
        \dec(\alpha_k)-f(\enc(\alpha_k))
        +
        f(\enc(\alpha_k))-f(\mathbf{x}_k)
    \right\|_2^2
    \nonumber\\
    &\overset{\mathrm{(b)}}{\leq}
    \frac{2}{K}
    \sum_{k=1}^K
    \left\|
        f(\enc(\alpha_k))-f(\mathbf{x}_k)
    \right\|_2^2
    +
    \frac{2}{K}
    \sum_{k=1}^K
    \left\|
        \dec(\alpha_k)-f(\enc(\alpha_k))
    \right\|_2^2
    \nonumber\\
    &=
    \underbrace{
    \frac{2}{K}
    \sum_{k=1}^K
    \left\|
        f(\enc(\alpha_k))-f(\mathbf{x}_k)
    \right\|_2^2
    }_{\lenc} + 
    \underbrace{
    \frac{2}{K}
    \sum_{k=1}^K
    \left\|
        \dec(\alpha_k)-f(\enc(\alpha_k))
    \right\|_2^2
    }_{\ldec}.
\label{eq:loss_decomposition_fixed_F}
\end{align}
Here, step (a) follows from the definition
$\fhat_{\mathcal{F}}(\mathbf{x}_k)=\dec(\alpha_k)$, and step (b) follows from
\(
    \|\mathbf{u}+\mathbf{v}\|_2^2
    \leq
    2\|\mathbf{u}\|_2^2+2\|\mathbf{v}\|_2^2 
\). 
Thus, we obtain
\begin{align}\label{eq:worst_case_obj_upperbound}
    \inf_{\dec \in \hild{2}{m}}
    \mathcal{L}(\enc,\dec)
    \le
    \lenc
    +
    \inf_{\dec \in \hild{2}{m}}
    \ldec .
\end{align}
Here, the infimum on the right-hand side applies only to $\ldec$, since $\lenc$ is independent of the decoder $\dec$.

\begin{remark}
The terms $\ldec$ and $\lenc$ in \eqref{eq:loss_decomposition_fixed_F} correspond to the two approximation steps in the \gcc{} construction in \eqref{eq:gcc_approx}. The first term, $\ldec$, captures the decoding approximation error associated with
\(
    \dec(\alpha_k) \approx f(\enc(\alpha_k)),
\)
namely, how accurately the decoder estimates the composite map $z\mapsto f(\enc(z))$ at the encoder design points $\{\alpha_k\}_{k=1}^K$. Note that although the decoder is constructed using the returned worker results at the decoder design points $\{\beta_v\}_{v\in\mathcal{F}}$, its accuracy is evaluated at the encoder design points $\{\alpha_k\}_{k=1}^K$. The second term, $\lenc$, captures the encoding approximation error associated with
\(
    f(\enc(\alpha_k)) \approx f(\mathbf{x}_k),
\)
namely, the effect of approximating the original input points $\{\mathbf{x}_k\}_{k=1}^K$ by the encoder values $\{\enc(\alpha_k)\}_{k=1}^K$ after applying the target function $f(\cdot)$.
\end{remark}

\begin{remark}\label{rem:coupling}
The two terms $\ldec$ and $\lenc$ in \eqref{eq:loss_decomposition_fixed_F} are interpretable but not independent. Both depend on the encoder $\enc(\cdot)$. The encoding term depends directly on the discrepancy between $\enc(\alpha_k)$ and $\mathbf{x}_k$, while the decoding term depends on $\enc(\cdot)$ through the composite function $f\circ \enc$ that the decoder must reconstruct.
\end{remark}

As discussed in Remark~\ref{rem:coupling}, the optimization problem in \eqref{eq:obj_opt_worst_case} remains coupled because both $\lenc$ and $\ldec$ depend on $\enc(\cdot)$. A systematic way to address this coupling is to formulate the design as a nested optimization problem. Specifically, for a fixed encoder, we first construct the decoder that minimizes $\ldec$. Substituting this decoder into the upper bound in \eqref{eq:loss_decomposition_fixed_F} yields an objective that depends only on the encoder, which can then be minimized in an outer optimization step.

In the following subsections, we first design the decoder for a given encoder and then design the encoder.

\subsubsection{\bf Decoder Design}\label{sec:worst_decoder_design}
We first design the decoder for a fixed encoder $\enc\in\hild{2}{d}$ and a fixed non-straggler set $\mathcal{F}\subseteq[N]$ satisfying $|\mathcal{F}|\ge N-S$. Ideally, our goal is to minimize
\begin{align}\label{eq:dec_obj_raw}
    \frac{2}{K}
    \sum_{k=1}^K
    \left\|
        \dec(\alpha_k)-f(\enc(\alpha_k))
    \right\|_2^2 .
\end{align}
However, the decoder must be constructed from the available worker outputs,
\[
    \{(\beta_v, f(\enc(\beta_v)))\}_{v\in\mathcal{F}} .
\]
A natural alternative is therefore to choose $\ldec$ as a minimizer of the empirical objective
\[
    \frac{1}{|\mathcal{F}|}
    \sum_{v \in \mathcal{F}}
    \left\|
        u(\beta_v) - f(\enc(\beta_v))
    \right\|_2^2 .
\]
The performance of the resulting scheme, however, is ultimately evaluated at the encoder design points $\{\alpha_k\}_{k=1}^K$, as reflected in \eqref{eq:dec_obj_raw}. Thus, the decoder must approximate the composite map
\[
    z \mapsto f(\enc(z))
\]
at points that are generally \emph{different} from those used in its construction. Consequently, a small approximation error at the available decoder design points $\{\beta_v\}_{v\in\mathcal{F}}$ does not, by itself, guarantee a small error at the target points $\{\alpha_k\}_{k=1}^K$.

This mismatch is analogous to the \emph{generalization gap} in learning theory: a model is trained on observed samples but evaluated on unseen test points. In statistical learning, such gaps are commonly mitigated by regularizing the training objective~\cite{hastie2009elements,evgeniou2000regularization}. This observation motivates us to design the decoder through a regularized objective that balances fidelity to the available worker outputs with a smoothness constraint. The smoothness regularization encourages the decoder to vary in a controlled manner over $\Omega$, which can improve its accuracy at unobserved target points when the underlying composite map is sufficiently smooth.

Accordingly, for the fixed encoder $\enc$ and non-straggler set $\mathcal{F}$, the decoder is defined as
\begin{align}\label{eq:decoder_opt}
\decstar
=
\operatorname*{arg\,min}_{u \in \hild{2}{m}}
\left\{
    \frac{1}{|\mathcal{F}|}
    \sum_{v \in \mathcal{F}}
    \left\|
        u(\beta_v) - f(\enc(\beta_v))
    \right\|_2^2
    +
    \declamb
    \norm{u^{\prime\prime}}^2_{\lpd{2}{m}}
\right\},
\end{align}
where $\declamb\ge 0$ is the decoder smoothing parameter. The first term in
\eqref{eq:decoder_opt} measures the empirical 
MSE over the
available worker outputs, while the second term penalizes the $L^2$ norm of the
second derivative of the decoder.

\begin{remark}[Cubic smoothing spline for $\declamb>0$]
\label{rem:smoothspline_def}
For $\declamb>0$, the optimization problem in \eqref{eq:decoder_opt} has a unique solution. This solution is the cubic smoothing spline associated with the data
\(
    \left\{
        \left(\beta_v,f(\enc(\beta_v))\right)
    \right\}_{v\in\mathcal{F}}
\)
and the smoothing parameter $\declamb$~\cite{wahba1975smoothing,wahba1990spline}. In particular, it is a piecewise cubic polynomial with continuous first and second derivatives. Additional details are provided in Appendix~\ref{app:smoothspline}.
\end{remark}

\begin{remark}[Natural cubic spline for $\declamb=0$]
\label{rem:natural_spline_decoder}
We extend the decoder construction in \eqref{eq:decoder_opt} to the limiting case $\declamb=0$ by taking the limit of the optimal smoothing-spline decoder as $\declamb\to0$. Specifically, the solution $\lim_{\declamb\to0} \decstar$ is defined in the minimum-curvature interpolation sense: among all functions $u \in \hild{2}{m}$ satisfying
\[
    u(\beta_v) = f(\enc(\beta_v)), \qquad v \in \mathcal{F},
\]
we select the interpolant that minimizes the curvature $\|u''\|_{\lpd{2}{m}}^2$. The resulting function is a natural cubic spline interpolating the returned worker outputs.
\end{remark}

Since the squared Euclidean norm is additive across output dimensions, the
optimization in \eqref{eq:decoder_opt} decouples into $m$ scalar problems, one
for each component of the decoder. Specifically, for each $j\in[m]$, the
decoder component $u_j:\Omega\to\mathbb{R}$ is constructed from the scalar
responses $\{f_j(\enc(\beta_v))\}_{v\in\mathcal{F}}$.

Using the Representer Theorem (Theorem~\ref{thm:representer}) for smoothing spline objective in \eqref{eq:decoder_opt}
\cite{wahba1975smoothing,wahba1990spline,scholkopf2001generalized}, each
scalar solution lies in the finite-dimensional space spanned by kernel functions
centered at the available decoder design points, together with the null space of
the second-derivative penalty. For this penalty, the null space is $\operatorname{span}\{1,t\}$.

Formally, let $\mathcal{F}=\{i_1,\ldots,i_{|\mathcal{F}|}\}$ denote the ordered set of
non-straggling worker indices. The decoder therefore admits the representation
\begin{align}\label{eq:dec_sol_form}
    \decstar(t)
    =
    \mathbf{d}_0+\mathbf{d}_1 t
    +
    \sum_{\ell=1}^{|\mathcal{F}|}
    \mathbf{c}_{\ell}\,\phi_0(t,\beta_{i_\ell}),
\end{align}
where $\mathbf{d}_0,\mathbf{d}_1\in\mathbb{R}^m$ determine the linear component,
and $\mathbf{c}_{\ell}\in\mathbb{R}^m$ are coefficient vectors associated with
the returned worker outputs. The function $\phi_0(\cdot,\cdot)$ is the reproducing kernel associated with the space $\hilz{2}$, which is the second-order Sobolev space with compact support (see Definition~\ref{def:sobz} and Appendix~\ref{app:sobolev_props} for details).

\begin{remark}[Explicit kernel form]\label{rem:kernel}
For the second-order Sobolev setting on $\Omega=(-1,1)$, the kernel
$\phi_0(t,s)$ admits the closed-form expression
\begin{align} \label{eq:sobolev_kernel_explicit}
    \phi_0(t,s)
    =
    \begin{cases}
      \displaystyle
      \frac{1}{3}+\frac{t+s}{2}+ts+\frac{ts^2}{2}-\frac{s^3}{6},
      & -1 \le s \le t \le 1, \\[0.7em]
      \displaystyle
      \frac{1}{3}+\frac{t+s}{2}+ts+\frac{st^2}{2}-\frac{t^3}{6},
      & -1 \le t \le s \le 1.
   \end{cases}
\end{align}
\end{remark}

Consequently, the coefficients
$\mathbf{C}=[\mathbf{c}_1,\ldots,\mathbf{c}_{|\mathcal{F}|}]^T
\in\mathbb{R}^{|\mathcal{F}|\times m}$ and
$\mathbf{D}=[\mathbf{d}_0,\mathbf{d}_1]^T\in\mathbb{R}^{2\times m}$ are
computed by solving the standard linear system for cubic smoothing splines
\cite{wahba1975smoothing,wahba1990spline,duchon1977splines}. 

For a fixed encoder, we next derive an upper bound on the worst-case decoder term
\(
    \max_{|\mathcal{F}|\ge N-S}
    \inf_{\dec}
    \ldec(\mathcal{F}),
\)
using the decoder construction in \eqref{eq:dec_sol_form}.

\begin{theorem}[Non-asymptotic upper bound on $\ldec$ in the worst-case setting]
\label{th:ldec_worstcase}
Consider the \gcc{} framework with $N$ worker nodes, and let
$\mathcal{F}\subseteq[N]$ be a non-straggler set satisfying
$|\mathcal{F}|\ge N-S$. Let
$\bm{\beta}:=\{\beta_1,\dots,\beta_N\}$ denote the ordered decoder design
points. Assume that the target function
$f=[f_1,\dots,f_m]^T:\mathbb{R}^d\to\mathbb{R}^m$ is twice differentiable and
satisfies the uniform bounds
\(
    \|\nabla f_j(\mathbf{x})\|_2 \le q_j
\)
and
\(
    \|D^2 f_j(\mathbf{x})\|_{\mathrm{op}} \le \mu_j
\)
for all $\mathbf{x}\in\mathbb{R}^d$ and $j\in[m]$. Let
\(
    q:=(\sum_{j=1}^m q_j^2)^{1/2}
\)
and
\(
    \mu:=(\sum_{j=1}^m \mu_j^2)^{1/2}.
\)
Then
\begin{equation}\label{eq:th_ldec_worstcase}
    \max_{|\mathcal{F}| \geq N-S}  \inf_{\dec}\ldec
    \le
    C_1 \max\{\mu^2,q^2\}
    A_{S,N}^{3/4}
    \left(1+\frac{A_{S,N}}{16}\right)^{1/4}
    \psi\!\left(\|\enc\|_{\hild{2}{d}}^2\right),
\end{equation}
where $C_1>0$ is a constant independent of $N$ and $S$,
$\psi(t):= t + 4t^2$ is a fixed, strictly increasing function, and
\begin{align}\label{eq:const_th2}
    A_{S,N}
    :=
    (S+1)^4\delmax{\bm{\beta}}^4
    +
    N\declamb\,
    (S+1)^3
    \frac{\delmax{\bm{\beta}}^3}{\delmin{\bm{\beta}}^2}.
\end{align}
\end{theorem}

\subsubsection{\bf Encoder Design}\label{sec:worst_encoder_design}

For a fixed encoder $u\in\hild{2}{d}$, Theorem~\ref{th:ldec_worstcase} provides an upper bound on the decoding term $\ldec$. On the other hand, since $\|\nabla f_j(\mathbf{x})\|_2\le q_j$ for all $\mathbf{x}\in\mathbb{R}^d$ and $j\in[m]$, and since $q:=(\sum_{j=1}^m q_j^2)^{1/2}$, we have
\begin{align}
    \lenc
    &=
    \frac{2}{K}
    \sum_{k=1}^K
    \left\|
        f(\enc(\alpha_k))-f(\mathbf{x}_k)
    \right\|_2^2
    \nonumber\\
    &\le
    \frac{2q^2}{K}
    \sum_{k=1}^K
    \left\|
        \enc(\alpha_k)-\mathbf{x}_k
    \right\|_2^2,
    \label{eq:lenc_bound}
\end{align}
using the Lipschitz continuity implied by the gradient bounds on $f$. 
Combining \eqref{eq:lenc_bound} with Theorem~\ref{th:ldec_worstcase}, the total loss is upper-bounded as
\begin{align}\label{eq:l_unified_bound}
     \ 
    \max_{|\mathcal{F}| \geq N-S}  \inf_{\substack{
        \dec \in \hild{2}{m}
    }}  \mathcal{L}( \enc, \dec)
    &\lec{(a)}{\le}
     \lenc + 
    \max_{|\mathcal{F}| \geq N-S}  \inf_{\substack{
        \dec \in \hild{2}{m}
    }} \ldec
    \nonumber\\
    &\le
    2q^2\left[
    \frac{1}{K}
    \sum_{k=1}^K
    \left\|
        \enc(\alpha_k)-\mathbf{x}_k
    \right\|_2^2+
    \Gamma_{S,N}
    \psi\!\left(\|\enc\|_{\hild{2}{d}}^2\right)\right],
\end{align}
where (a) follows by the fact that $\lenc$ does not depend on the non-straggler set $\stset$ and 
\begin{align}\label{eq:lamb_enc}
    \Gamma_{S,N}
    :=
    \frac{C_1}{2q^2}\max\{\mu^2,q^2\}
    A_{S,N}^{3/4}
    \left(1+\frac{A_{S,N}}{16}\right)^{1/4}.
\end{align}
We therefore choose the encoder as a minimizer of this upper bound:
\begin{align}\label{eq:enc_opt}
    \enc^*
    =
    \operatorname*{arg\,min}_{u\in\hild{2}{d}}
    \left\{
        \frac{1}{K}
        \sum_{k=1}^K
        \left\|
            u(\alpha_k)-\mathbf{x}_k
        \right\|_2^2
        +
        \Gamma_{S,N}
        \psi\!\left(\|u\|_{\hild{2}{d}}^2\right)
    \right\}.
\end{align}
Using the representer theorem and the  properties of vector-valued Sobolev representation in Remark~\ref{rem:rep_vec}, any minimizer of \eqref{eq:enc_opt}, admits the finite-dimensional representation
\begin{align}\label{eq:final_encoder_form}
    \encstar(t)
    =
    \sum_{k=1}^K
    \phi(t,\alpha_k)\mathbf{z}_k,
\end{align}
for some coefficient vectors $\mathbf{z}_1,\ldots,\mathbf{z}_K\in\mathbb{R}^d$, where $\phi(\cdot,\cdot)$ is the scalar reproducing kernel of $\hild{2}{}$ given by $\phi(t,s)=1+ts+\phi_0(t,s)$, with $\phi_0$ defined in Remark~\ref{rem:kernel}. Thus, the encoder design reduces to a finite-dimensional optimization over the coefficients $\{\mathbf{z}_k\}_{k=1}^K$.

The encoder objective in \eqref{eq:enc_opt} contains the nonlinear term
$\psi(\|u\|_{\hild{2}{d}}^2)$, where $\psi(t)=t+4t^2$. Although this objective admits a finite-dimensional RKHS representation by \eqref{eq:final_encoder_form}, directly calculating the representation coefficients can still be computationally expensive (see Section~\ref{sec:worst_comp_analysis}). The following theorem shows that the encoder can instead be constructed through a cubic smoothing-spline objective without losing the validity of the loss upper bound.

\begin{theorem}\label{th:efficient_encoder}
Under the assumptions of Theorem~\ref{th:ldec_worstcase}, there exist constants $m_1,m_2>0$ and 
\[ 
R_{\mathrm e} = R_{\mathrm e} \left( \{\alpha_k\}_{k=1}^K, \{\mathbf{x}_k\}_{k=1}^K \right) >0, \]
such that 
\begin{align}\label{eq:efficient_encoder_loss_bound}
     \objwc
\le
2q^2 \left[ \frac{1}{K}
\sum_{k=1}^K
\left\|
    \encdagger(\alpha_k)-\mathbf{x}_k
\right\|_2^2
+
\lambda_{\mathrm e}
\left(
    R_{\mathrm e}
    +
    \|(\encdagger)''\|_{\lpd{2}{d}}^2
\right) \right],
\end{align}
where encoder smoothing parameter
\(
    \lambda_{\mathrm e}
    :=
    \frac{\Gamma_{S,N}(m_1+m_2R_{\mathrm e})}{2q^2},
\)
$\Gamma_{S,N}$ is defined in \eqref{eq:lamb_enc}, and  $\encdagger$ is a minimizer of
\begin{align}\label{eq:efficient_encoder_opt}
    \encdagger
    =
    \operatorname*{arg\,min}_{u\in\hild{2}{d}}
    \left\{
        \frac{1}{K}
        \sum_{k=1}^K
        \left\|
            u(\alpha_k)-\mathbf{x}_k
        \right\|_2^2
        +
        \lambda_{\mathrm e}
        \|u''\|_{\lpd{2}{d}}^2
    \right\}.
\end{align}
\end{theorem}

Based on Theorem~\ref{th:efficient_encoder}, $\encdagger$ is a vector-valued cubic smoothing spline and admits the representation
\begin{align}\label{eq:enc_final_rep}
    \encdagger(t)
    =
    \mathbf{a}_0+\mathbf{a}_1t
    +
    \sum_{k=1}^K
    \mathbf{b}_k\,\phi_0(t,\alpha_k),
\end{align}
where $\mathbf{a}_0,\mathbf{a}_1,\{\mathbf{b}_i\}^K_{i=1}\in\mathbb{R}^d$, and $\phi_0$ is defined in Remark~\ref{rem:kernel}. Therefore, the coefficients are computed efficiently by solving the standard linear system for cubic smoothing splines with smoothing parameter $\lambda_{\mathrm e}$.

\subsubsection{\gcc{} Scheme}
In the previous two subsections, we described the encoder and decoder designs, the intuition behind them, and the techniques used to bound their errors. These constructions lead to the practical \gcc{} scheme summarized in Algorithm~\ref{alg:gcc}. Recall that the optimization problems in \eqref{eq:decoder_opt} and \eqref{eq:efficient_encoder_opt}, corresponding to Steps~1 and~4 of the algorithm, admit the special structure given by the representer theorem stated in Theorem~\ref{thm:representer}. Moreover, the corresponding coefficients can be computed efficiently, as discussed in the next subsection.

The algorithm has two smoothing parameters: the encoder smoothing parameter $\lambda_{\mathrm e}$ and the decoder smoothing parameter $\declamb$. The decoder parameter $\declamb$ is a regularization parameter in the smoothing-spline decoder, while the encoder parameter $\lambda_{\mathrm e}$ arises from the upper-bound-based encoder design and incorporates several constants from the analysis. In practice, these constants depend on problem-specific quantities that are not directly available, and therefore $\lambda_{\mathrm e}$ and $\declamb$ should be viewed as two scalar knobs of the algorithm that can be tuned. Even one can choose limiting choice $\lambda_{\mathrm e}\to0$ and $\declamb\to0$, which reduces the encoder and decoder constructions to their minimum-curvature interpolation forms, namely natural cubic splines, as discussed in Remark~\ref{rem:natural_spline_decoder}. As shown in the experimental results in Section~\ref{sec:exp_result}, this hyperparameter-free choice already outperforms state-of-the-art coded-computing schemes.

\begin{remark}[The encoder and decoder are not trained]
The proposed code design is motivated by ideas from learning theory, particularly the notion of a generalization gap. However, the encoder and decoder are not obtained through an iterative training process. Instead, they have prescribed functional forms whose parameters can be computed efficiently, as described in Algorithm~\ref{alg:gcc} and further detailed in Subsection~\ref{sec:worst_comp_analysis}.
\end{remark}

\begin{algorithm}[t]
\caption{General Coded Computing (\gcc{})}
\label{alg:gcc}
\begin{algorithmic}[1]
\Require Input batch $\{\mathbf{x}_k\}_{k=1}^K$, target function $f(\cdot)$, encoder and decoder design points $\{\alpha_k\}_{k=1}^K$ and $\{\beta_n\}_{n=1}^N$, smoothing parameters $\lambda_{\mathrm e}$ and $\declamb$.
\Ensure Estimates $\{\fhat_{\mathcal{F}}(\mathbf{x}_k)\}_{k=1}^K$.
\State The master node constructs the encoder $\encdagger$ by solving the smoothing-spline optimization 
\[
    \encdagger
=
    \operatorname*{arg\,min}_{u\in\hild{2}{d}}
    \left\{
        \frac{1}{K}
        \sum_{k=1}^K
        \|u(\alpha_k)-\mathbf{x}_k\|_2^2
        +
        \lambda_{\mathrm e}
        \|u''\|_{\lpd{2}{d}}^2
    \right\},
\]
which has a closed-form solution given in \eqref{eq:final_encoder_form}.
\State The master nodes computes and sends coded inputs $\widetilde{\mathbf{x}}_n=\encdagger(\beta_n)$ to worker $n\in[N]$.

\State Each worker node $n \in [N]$ computes and returns $f(\widetilde{\mathbf{x}}_n)=f(\encdagger(\beta_n))$. 

\State  Let $\mathcal{F}$ be the set of non-stragglers. 
The master node constructs the decoder $\decstar$ from $\{(\beta_v,f(\encdagger(\beta_v)))\}_{v\in\mathcal{F}}$ by solving
\[
    \decstar
    =
    \operatorname*{arg\,min}_{u\in\hild{2}{m}}
    \left\{
        \frac{1}{|\mathcal{F}|}
        \sum_{v\in\mathcal{F}}
        \|u(\beta_v)-f(\encdagger(\beta_v))\|_2^2
        +
        \declamb
        \|u''\|_{\lpd{2}{m}}^2
    \right\},
\]
with a closed-form solution given in \eqref{eq:dec_sol_form}.
\State The master nodes outputs $\fhat_{\mathcal{F}}(\mathbf{x}_k)=\decstar(\alpha_k)$ for all $k\in[K]$.

\end{algorithmic}
\end{algorithm}

\subsection{Computational Analysis}\label{sec:worst_comp_analysis}

A key practical requirement for any coded-computing approach is that the master-side computational overhead remain small. This overhead includes constructing the coded inputs during encoding and aggregating the returned worker outputs during decoding. If the master spends as much time on these operations as it would take to compute $\{f(\mathbf{x}_k)\}_{k=1}^K$ locally, then the distributed framework loses its practical advantage.

In the following analysis, we divide the master-side overhead into two parts: encoder complexity and decoder complexity. The encoder complexity includes (i) constructing the encoder and (ii) evaluating the encoder at the decoder points to generate the coded inputs. The decoder complexity includes (iii) constructing the decoder from the returned worker outputs and (iv) evaluating the decoder at the encoder points to recover the estimates.

\subsubsection{\bf Encoder Complexity}
A useful distinction is between the encoder obtained from the nonlinear surrogate in \eqref{eq:enc_opt} and the efficient encoder in Theorem~\ref{th:efficient_encoder}. The nonlinear encoder admits the finite-dimensional kernel representation in \eqref{eq:final_encoder_form}. However, its coefficients do not generally admit a closed-form expression because of the nonlinear regularization term $\psi(\cdot)$. Since the objective in \eqref{eq:enc_opt} is smooth and convex, standard iterative solvers, such as limited-memory Broyden--Fletcher--Goldfarb--Shanno (L-BFGS)~\cite{nocedal2006numerical}, are used to compute the coefficient vectors. Formally, let
\(
    \mathbf{K}_{\alpha}
    :=
    \big[\phi(\alpha_i,\alpha_j)\big]_{i,j=1}^K
    \in\mathbb{R}^{K\times K},
\)
and
\(
    \mathbf{Z}
    :=
    [\mathbf{z}_1,\ldots,\mathbf{z}_K]^T
    \in\mathbb{R}^{K\times d}.
\)
Then the encoder values at the encoder design points are given by $\mathbf{K}_{\alpha}\mathbf{Z}$, and the RKHS norm reduces to the quadratic form
\(
    \|\enc\|_{\hild{2}{d}}^2
    =
    \sum_{i=1}^K
    \sum_{j=1}^K
    \mathbf{z}_i^\top
    \phi(\alpha_i,\alpha_j)
    \mathbf{z}_j
    =
    \operatorname{tr}(\mathbf{Z}^T\mathbf{K}_{\alpha}\mathbf{Z}).
\)
Thus, evaluating the objective in \eqref{eq:enc_opt} and its gradient requires matrix operations involving $\mathbf{K}_{\alpha}\mathbf{Z}$, which cost $\mathcal{O}(K^2d)$ per iteration. If the iterative solver uses $I$ iterations, the coefficient-computation cost is $\mathcal{O}(IK^2d)$. After the coefficients are computed, direct evaluation of the nonlinear encoder at the $N$ decoder design points requires computing
\(
    \widetilde{\mathbf{X}}
    =
    \mathbf{\Phi}_{\beta,\alpha}\mathbf{Z},
\)
where
\(
    \mathbf{\Phi}_{\beta,\alpha}
    :=
    \big[\phi(\beta_n,\alpha_k)\big]_{n,k}
    \in\mathbb{R}^{N\times K}.
\)
This direct kernel evaluation costs $\mathcal{O}(NKd)$. Therefore, the direct implementation of the nonlinear encoder incurs
\[
    \mathcal{O}(IK^2d+NKd)
\]
master-side encoding overhead.

In contrast, the efficient encoder in \eqref{eq:efficient_encoder_opt} is a vector-valued cubic smoothing spline. Using a B-spline basis, the corresponding linear system is banded and is solved in time linear in the number of design points~\cite{eilers1996flexible,de2001calculation}. Therefore, the efficient encoder is constructed in $\mathcal{O}(Kd)$ operations and evaluated at the $N$ decoder design points in $\mathcal{O}(Nd)$ operations. Hence, the total efficient encoding overhead is
\(
    \mathcal{O}((K+N)d),
\)
which becomes $\mathcal{O}(Nd)$ in the common regime where $N\ge K$.

\subsubsection{\bf Decoder Complexity} The decoder has the same smoothing-spline structure. Constructing the decoder from points
\(
    \{(\beta_v,f(\enc(\beta_v)))\}_{v\in\mathcal{F}}
\)
requires $\mathcal{O}(|\mathcal{F}|\cdot m)$ operations using the same B-spline implementation, and evaluating the decoder at the $K$ encoder design points requires $\mathcal{O}(Km)$ operations. Thus, the total decoding overhead is
\(
    \mathcal{O}((|\mathcal{F}|+K)m),
\)
which becomes $\mathcal{O}(|\mathcal{F}|\cdot m)$ when $|\mathcal{F}|\ge K$. Overall, with the efficient encoder, the master-side complexity of \gcc{} is
\begin{align}\label{eq:gcc_overhead}
    \mathcal{O}((K+N)d)
    +
    \mathcal{O}((|\mathcal{F}|+K)m).
\end{align}
Thus, \gcc{} scales linearly with the input dimension $d$, the output dimension $m$, the number of workers, and the number of input data points.

\subsubsection{\bf Comparison with existing schemes}
Lagrange Coded Computing (LCC) relies on polynomial interpolation and fast multipoint evaluation~\cite{yu2019lagrange,borodin1974fast}. Its encoding and decoding computational complexities scale as
\(
    \mathcal{O}\!\left(N\log^2K\log\log K\cdot d\right)
\)
and
\(
    \mathcal{O}\!\left(|\mathcal{F}|\log^2|\mathcal{F}|\log\log |\mathcal{F}|\cdot m\right),
\)
respectively. Therefore, LCC has super-linear dependence on the number of input data points and workers. 

Berrut Approximation Coded Computing (BACC) uses barycentric rational interpolation~\cite{jahani2022berrut}. With direct evaluation, its encoding and decoding computational costs scale as $\mathcal{O}(NKd)$ and $\mathcal{O}(K|\mathcal{F}|\cdot m)$, respectively. 

Comparing the master-side computational overhead of \gcc{} in \eqref{eq:gcc_overhead} with those of LCC and BACC shows that \gcc{} has lower encoding and decoding overhead.

\subsection{Asymptotic Analysis}\label{sec:worst_asymp}

In this section, we analyze the asymptotic dependence of the worst-case loss on the number of workers $N$ and the straggler budget $S$. The non-asymptotic bounds derived in the Section~\ref{sec:worst_comp_analysis} depend on the geometry of the decoder design points through $\delmax{\bm{\beta}}$ and $\delmin{\bm{\beta}}$. To state the asymptotic results we impose the following assumption.

\begin{assumption}[Decoder design point spacing] \label{ass:decoder_spacing}
 There exist constants $B_1,B_2,\gamma_1,\gamma_2>0$, independent of $N$ and $S$, such that the ordered decoder design points $\bm{\beta}=\{\beta_1,\ldots,\beta_N\}$ satisfy
\[
    \frac{\delmax{\bm{\beta}}}{\delmin{\bm{\beta}}}
    \le
    B_1N^{\gamma_1},
    \qquad
    \delmax{\bm{\beta}}
    \le
    B_2N^{-\gamma_2}.
\]
\end{assumption}

\begin{remark}
The spacing parameters in Assumption~\ref{ass:decoder_spacing} are mild and hold for the standard choices used in coded computing. For equidistant design points, one may take $\gamma_1=0$ and $\gamma_2=1$. For Chebyshev design points of the first or second kind, the maximum gap satisfies $\delmax{\bm{\beta}}=\mathcal{O}(N^{-1})$, while the ratio satisfies $\delmax{\bm{\beta}}/\delmin{\bm{\beta}} = \mathcal{O}(N)$; hence one may take $\gamma_1=1$ and $\gamma_2=1$.
\end{remark}

\begin{theorem}[Worst-case asymptotic bound for $\ldec$]
\label{th:ldec_worst_asymp}
Under the assumptions of Theorem~\ref{th:ldec_worstcase} and Assumption~\ref{ass:decoder_spacing}, suppose that 
\[
    \declamb
    \le
    N^{-(2\gamma_1+3\gamma_2+1)}.
\]
Then, for every non-straggler set $\mathcal{F}\subseteq[N]$ with $|\mathcal{F}|\ge N-S$, the decoder $\decstar$ defined by \eqref{eq:decoder_opt} satisfies
\begin{align}\label{eq:ldec_worst_asymp}
        \max_{|\mathcal{F}| \geq N-S}  \inf_{\dec}\ldec
    \le
    C_2
    \max\{\mu^2,q^2\}
    \left(
        \frac{S+1}{N^{\gamma_2}}
    \right)^3
    \left[
        1
        +
        \left(
            \frac{S+1}{N^{\gamma_2}}
        \right)^4
    \right]^{1/4}
    \psi\!\left(\|\enc\|_{\hild{2}{d}}^2\right),
\end{align}
where $\psi(t)=t+4t^2$, and $C_2>0$ is a constant independent of $N$, $S$, and $\mathcal{F}$.
\end{theorem}

\begin{corollary}[Worst-case convergence rate]
\label{cor:worst_convergence}
Under the assumptions of Theorem~\ref{th:ldec_worst_asymp}, suppose that the decoder is constructed according to \eqref{eq:decoder_opt}, and that the encoder $\encstar$ is designed either according to \eqref{eq:enc_opt} or using the efficient encoder construction in Theorem~\ref{th:efficient_encoder}. If $S=o(N^{\gamma_2})$, then, for all sufficiently large $N$, the worst-case end-to-end loss satisfies
\begin{align}\label{eq:worst_total_convergence}
    \objwc
    =
    \mathcal{O}\!\left(
        \frac{(S+1)^3}{N^{3\gamma_2}}
    \right).
\end{align}
\end{corollary}

\begin{remark}\label{rem:worst_convergence}
Corollary~\ref{cor:worst_convergence} shows that the worst-case end-to-end loss converges to zero whenever $S=o(N^{\gamma_2})$. For the standard equidistant and Chebyshev design points discussed in Example~\ref{ex:spacing_points}, we have $\gamma_2=1$, and the asymptotic rate is $\mathcal{O}(S^3N^{-3})$. Thus, in these common settings, the end-to-end loss converges to zero as long as $S=o(N)$. This behavior contrasts with traditional exact coded-computing schemes, which typically rely on a recovery threshold: if the number of non-straggling workers falls below this threshold, the decoder may fail to recover the desired computation. In contrast, \gcc{} produces an approximation from any valid non-straggler set, with an error that degrades smoothly as the number of stragglers increases.
\end{remark}

\section{Probabilistic Straggler Setting}\label{sec:prob_straggler}

Next, we study \gcc{} under a probabilistic straggler setting, in which each worker node independently behaves as a straggler with probability $p\in(0,1)$. Unlike the worst-case setting, we do not impose a deterministic budget on the number of stragglers. Instead, the set of non-straggling workers is random.

Let $\mathsf{F}_{p,N}$ denote the distribution over subsets of $[N]$ induced by independent worker failures with probability $p$. Similar to the worst-case setting, the objective is to design the encoder and decoder so as to minimize the mean-squared error. Here, however, the loss is averaged over the probabilistic straggler distribution. Thus, the probabilistic design objective is
\begin{align}\label{eq:obj_opt_prob}
    \objp
    :=\inf_{\enc\in \hild{2}{d}}\ 
        \mathbb{E}_{\mathcal{F} \sim \mathsf{F}_{p,N}}
        \big[ \inf_{\dec \in \hild{2}{m}}
            \mathcal{L}( \enc, \dec)
        \big],
\end{align}

The analysis follows the same overall structure as the worst-case setting. We first derive an upper bound on the expected end-to-end loss and use it to construct the encoder and decoder (Section~\ref{sec:prob_code_design}). We then discuss the computational complexity of the resulting encoding and decoding procedures (Section~\ref{sec:prob_comp_analysis}). Finally, we characterize the asymptotic behavior of the expected loss as the number of workers $N$ increases (Section~\ref{sec:prob_asymp}).

\subsection{Code Design}\label{sec:prob_code_design}

The encoder and decoder design in the probabilistic straggler setting are the same as in the worst-case setting once the non-straggler set $\mathcal{F}$ is realized. The difference is in the performance criterion: instead of taking the maximum over all non-straggler sets satisfying $|\mathcal{F}|\ge N-S$, we take expectation over $\mathcal{F}\sim\mathsf{F}_{p,N}$.

For any realized non-straggler set $\mathcal{F}\subseteq[N]$, the loss decomposition in \eqref{eq:loss_decomposition_fixed_F} gives
\begin{align}\label{eq:prob_loss_decomposition_fixed_F}
    \mathcal{L}(\enc, \dec)
    \le
    \lenc + 
    \ldec(\mathcal{F}),
\end{align}
where
\(
    \ldec(\mathcal{F})
    :=
    \frac{2}{K}
    \sum_{k=1}^K
    \left\|
        \dec(\alpha_k)-f(\enc(\alpha_k))
    \right\|_2^2
\)
and
\(
    \lenc
    :=
    \frac{2}{K}
    \sum_{k=1}^K
    \left\|
        f(\enc(\alpha_k))-f(\mathbf{x}_k)
    \right\|_2^2 .
\)
Here, the notation $\ldec(\mathcal{F})$ is used to emphasize that the decoder is constructed from the worker outputs indexed by the realized non-straggler set $\mathcal{F}$. Since $\lenc$ is independent of $\mathcal{F}$ and $\dec$, taking the infimum over the decoder and then expectation over $\mathcal{F}\sim\mathsf{F}_{p,N}$ gives
\begin{align}\label{eq:prob_loss_decomposition}
    \mathbb{E}_{\mathcal{F}\sim\mathsf{F}_{p,N}}
    \!\left[
        \inf_{\dec\in\hild{2}{m}}
        \mathcal{L}(\fhat_{\mathcal{F}})
    \right]
    \le
    \lenc
    +
    \mathbb{E}_{\mathcal{F}\sim\mathsf{F}_{p,N}}
    \!\left[
        \inf_{\dec\in\hild{2}{m}}
        \ldec(\mathcal{F})
    \right].
\end{align}

\subsubsection{\bf Decoder Design} For a fixed encoder $\enc$ and a realized non-straggler set $\mathcal{F}$, the decoder is constructed exactly as in the worst-case setting:
\begin{align}\label{eq:prob_decoder_opt}
\decstar
=
\operatorname*{arg\,min}_{u \in \hild{2}{m}}
\left\{
    \frac{1}{|\mathcal{F}|}
    \sum_{v \in \mathcal{F}}
    \left\|
        u(\beta_v) - f(\enc(\beta_v))
    \right\|_2^2
    +
    \declamb
    \norm{u^{\prime\prime}}^2_{\lpd{2}{m}}
\right\}.
\end{align}
Thus, for each realization $\mathcal{F}$, the decoder remains a vector-valued cubic smoothing spline with the same representation as in \eqref{eq:dec_sol_form}.

Using the decoder construction in \eqref{eq:prob_decoder_opt}, we now derive an upper bound on
\(
\mathbb{E}_{\mathcal{F}\sim\mathsf{F}_{p,N}}
    [
        \inf_{\dec}
        \ldec(\mathcal{F})
    ].
\) 
The bound is expressed in terms of a new random variable $R_{\mathcal{F}}$ that captures the geometry of the realized straggler pattern. Let
\(
    \mathcal{S}_{\mathcal{F}}
    :=
    [N]\setminus\mathcal{F}
\)
be the corresponding straggler set, and let $R_{\mathcal{F}}$ denote the maximum number of consecutive indices in $\mathcal{S}_{\mathcal{F}}$, with the convention that $R_{\mathcal{F}}=0$ when $\mathcal{S}_{\mathcal{F}}=\emptyset$.

Unlike the worst-case setting, where the upper bound is expressed in terms of the deterministic quantity $S+1$ in \eqref{eq:th_ldec_worstcase}, the upper bound here is expressed in terms of $R_{\mathcal{F}}+1$, which depends on the \emph{longest consecutive run} of stragglers in the realized failure pattern. This replacement leads to a somewhat surprising implication, which is highlighted later in Remark~\ref{rem:log}.

\begin{theorem}[Non-asymptotic upper bound on $\mathbb{E}_{\mathcal{F}\sim\mathsf{F}_{p,N}}\inf_{\dec} \ldec(\mathcal{F})$ in the probabilistic setting]
\label{th:ldec_prob_nonasymp}

Consider the \gcc{} framework with $N$ worker nodes under the probabilistic straggler setting and let $\mathcal{F}\sim\mathsf{F}_{p,N}$ denote the realized set of non-straggling workers. . Let $\bm{\beta}:=\{\beta_1,\dots,\beta_N\}$ denote the ordered decoder design points. Assume that the target function $f=[f_1,\dots,f_m]^T:\mathbb{R}^d\to\mathbb{R}^m$ is twice differentiable and satisfies
\(
    \|\nabla f_j(\mathbf{x})\|_2\le q_j
\)
and
\(
    \|D^2 f_j(\mathbf{x})\|_{\mathrm{op}}\le \mu_j
\)
for all $\mathbf{x}\in\mathbb{R}^d$ and $j\in[m]$. Let
\(
    q:=(\sum_{j=1}^m q_j^2)^{1/2}
\)
and
\(
    \mu:=(\sum_{j=1}^m \mu_j^2)^{1/2}.
\)
Then, 
\begin{align}\label{eq:prob_ldec_nonasymp}
    \mathbb{E}_{\mathcal{F}\sim\mathsf{F}_{p,N}}[\inf_{\dec} \ldec(\mathcal{F})]
    \le
    C_1\max\{\mu^2,q^2\}\cdot
    \mathbb{E}_{\mathcal{F}\sim\mathsf{F}_{p,N}}
    \left[
        A_{\mathcal{F},N}^{3/4}
        (
            1+\frac{A_{\mathcal{F},N}}{16}
        )^{1/4}
    \right]\cdot
    \psi\!\left(\|\enc\|_{\hild{2}{d}}^2\right),
\end{align}
where $C_1>0$ is independent of $N$ and $\mathcal{F}$,
$\psi(t):=t+4t^2$, and
\begin{align}\label{eq:A_prob_F_N}
    A_{\mathcal{F},N}
    :=
    (R_{\mathcal{F}}+1)^4
    \delmax{\bm{\beta}}^4
    +
    N\declamb\,
    (R_{\mathcal{F}}+1)^3
    \frac{\delmax{\bm{\beta}}^3}{\delmin{\bm{\beta}}^2}.
\end{align}
\end{theorem}



\subsubsection{\bf Encoder Design} For the encoder design, we first bound the encoding loss $\lenc$. As in the worst-case setting, the gradient bounds on $f$ imply
\begin{align}\label{eq:prob_lenc_bound}
    \lenc
    \le
    \frac{2q^2}{K}
    \sum_{k=1}^K
    \left\|
        \enc(\alpha_k)-\mathbf{x}_k
    \right\|_2^2 .
\end{align}
Combining \eqref{eq:prob_loss_decomposition}, \eqref{eq:prob_ldec_nonasymp}, and \eqref{eq:prob_lenc_bound}, the expected loss is upper-bounded as
\begin{align}\label{eq:prob_unified_bound}
    \mathbb{E}_{\mathcal{F}\sim\mathsf{F}_{p,N}}
    \left[
        \inf_{\dec}
        \ldec(\mathcal{F})
    \right]
    \le 2q^2\left[
    \frac{1}{K}
    \sum_{k=1}^K
    \left\|
        \enc(\alpha_k)-\mathbf{x}_k
    \right\|_2^2 + 
    \Gamma_{p,N}\cdot
    \psi\!\left(\|\enc\|_{\hild{2}{d}}^2\right)
    \right],
\end{align}
where $\Gamma_{p,N}:=\frac{C_1}{2q^2}\max\{\mu^2,q^2\}\cdot
    \mathbb{E}_{\mathcal{F}\sim\mathsf{F}_{p,N}}
    [
        A_{\mathcal{F},N}^{3/4}
        (
            1+\frac{A_{\mathcal{F},N}}{16}
        )^{1/4}
    ]$.
We therefore choose the encoder as a minimizer of the surrogate upper bound
\begin{align}\label{eq:prob_enc_opt}
    \encstar
    =
    \operatorname*{arg\,min}_{u\in\hild{2}{d}}
    \left\{
        \frac{1}{K}
        \sum_{k=1}^K
        \left\|
            u(\alpha_k)-\mathbf{x}_k
        \right\|_2^2
        +
        \Gamma_{p,N}\cdot
        \psi\!\left(\|u\|_{\hild{2}{d}}^2\right)
    \right\}.
\end{align}
By the same representer-theorem argument used in Section~\ref{sec:worst_encoder_design}, any minimizer of \eqref{eq:prob_enc_opt}, admits the finite-dimensional representation in \eqref{eq:final_encoder_form}.

\paragraph{\bf Efficient encoder design}
As in the worst-case setting, the nonlinear regularizer in \eqref{eq:prob_enc_opt} can be replaced by a quadratic smoothing-spline surrogate. Specifically, define
\(
    \enclamb
    :=
    \Gamma_{p,N}(m_1+m_2R_{\mathrm e}),
\)
where $R_{\mathrm e}$, $m_1$, and $m_2$ are from Theorem~\ref{th:efficient_encoder}. As a result, the efficient encoder in the probabilistic setting is obtained by solving
\begin{align}\label{eq:prob_efficient_encoder_opt}
    \encdagger
    =
    \operatorname*{arg\,min}_{u\in\hild{2}{d}}
    \left\{
        \frac{1}{K}
        \sum_{k=1}^K
        \left\|
            u(\alpha_k)-\mathbf{x}_k
        \right\|_2^2
        +
        \enclamb
        \|u''\|_{\lpd{2}{d}}^2
    \right\}.
\end{align}
Thus, $\encdagger$ remains a vector-valued cubic smoothing spline and is computed using the same linear-system method as in the worst-case setting. Consequently, Algorithm~\ref{alg:gcc} applies without modification in the probabilistic straggler setting. 

\subsection{Computational Analysis}\label{sec:prob_comp_analysis}

The probabilistic straggler setting does not change the master-side computations of \gcc{} for a realized non-straggler set $\mathcal{F}$. Therefore, conditioned on $\mathcal{F}$, the master-side computational overhead is
\begin{align}\label{eq:prob_gcc_overhead}
    \mathcal{O}((K+N)d)
    +
    \mathcal{O}((|\mathcal{F}|+K)m).
\end{align}
Since each worker is non-straggling with probability $1-p$, we have
\(
    |\mathcal{F}|\sim \mathrm{Binomial}(N,1-p)
\),
and hence
\(
    \mathbb{E}[|\mathcal{F}|]=(1-p)N=\mathcal{O}(N).
\)
Because the decoding overhead is linear in $|\mathcal{F}|$, the expected master-side overhead is
\begin{align}\label{eq:prob_expected_gcc_overhead}
    \mathcal{O}((K+N)d)
    +
    \mathcal{O}((N+K)m).
\end{align}
Thus, \gcc{} preserves the same linear computational scaling in the probabilistic straggler setting as in the worst-case setting.
\subsection{Asymptotic Analysis}\label{sec:prob_asymp}

In this subsection, we derive an asymptotic rate for the expected loss in the probabilistic straggler setting. A direct use of the worst-case result in Corollary~\ref{cor:worst_convergence} gives a pessimistic prediction for the convergence of the end-to-end loss. Indeed, under the probabilistic model, the expected number of stragglers is $pN$. However, when $\gamma_2=1$, Remark~\ref{rem:worst_convergence} guarantees convergence in the worst-case setting only if $S/N\to 0$ as $N\to\infty$. Substituting the average value $S\approx pN$ into the worst-case rate would therefore suggest that convergence may not occur. Perhaps surprisingly, we show that the expected loss still converges to zero, with only a logarithmic degradation relative to the worst-case rate. 

\begin{theorem}[Asymptotic bound in the probabilistic straggler setting]
\label{th:prob_dec_bound}
Suppose the assumptions of Theorem~\ref{th:ldec_prob_nonasymp} and Assumption~\ref{ass:decoder_spacing} hold. Assume that each worker independently straggles with probability $p\in(0,1)$, and let $\mathcal{F}\sim\mathsf{F}_{p,N}$ be the resulting non-straggler set. For each realization $\mathcal{F}$, let $\decstar$ be the decoder constructed by solving the smoothing-spline problem in \eqref{eq:prob_decoder_opt}. If the decoder smoothing parameter satisfies
\(
    \declamb
    \le
    N^{-(2\gamma_1+3\gamma_2+1)},
\)
then there exist constants $C_3>0$ and $n_0\in\mathbb{N}$, independent of $N$, such that for all $N>n_0$,
\begin{align}\label{eq:th_l_dec_prob}
    \mathbb{E}_{\mathcal{F}\sim\mathsf{F}_{p,N}}
    \left[
        \inf_{\dec}
        \ldec(\mathcal{F})
    \right]
    \le
    C_3
    \max\{\mu^2,q^2\}
    \frac{
        \left(\log_{1/p}\!\big((1-p)N\big)\right)^3
    }{
        N^{3\gamma_2}
    }
    \psi\!\left(\|\enc\|_{\hild{2}{d}}^2\right),
\end{align}
where $\psi(t)=t+4t^2$.
\end{theorem}

\begin{corollary}[Probabilistic convergence rate]
\label{cor:prob_convergence}
Under the assumptions of Theorem~\ref{th:prob_dec_bound}, let the decoder be constructed according to \eqref{eq:prob_decoder_opt}, and let the encoder be constructed using the efficient smoothing-spline design in \eqref{eq:prob_efficient_encoder_opt}. Then, for sufficiently large $N$, the expected end-to-end loss satisfies
\begin{align}\label{eq:prob_total_convergence}
    \objp
    =
    \mathcal{O}\!\left(
        \frac{\log_{1/p}^3(N)}{N^{3\gamma_2}}
    \right).
\end{align}
\end{corollary}

\begin{remark}
\label{rem:log}
Corollary~\ref{cor:prob_convergence} shows that the expected end-to-end loss converges to zero for every fixed worker failure probability $p\in(0,1)$. For the standard equidistant and Chebyshev design points discussed in Example~\ref{ex:spacing_points}, we have $\gamma_2=1$, and the convergence rate becomes
\[
    \mathcal{O}\!\left(
        \frac{\log_{1/p}^3(N)}{N^3}
    \right).
\]
This result explains why the expected loss in the probabilistic setting still converges, even though the number of stragglers may be of order $N$. The key idea is that, although the expected number of stragglers is $pN$, the dominant factor is not the total number of stragglers, but the \emph{longest consecutive run of stragglers}, which grows only logarithmically with $N$. Thus, the straggler patterns that dominate the worst-case bound occur rarely under independent worker failures.
\end{remark}

\section{Proofs of the Main Results}\label{sec:proofs}

In this section, we provide the proofs of the theoretical results stated in Sections~\ref{sec:worst_case} and~\ref{sec:prob_straggler}. In each proof, we may introduce auxiliary lemmas. To keep the main arguments focused, the proofs of these auxiliary lemmas are provided in Appendix~\ref{app:proof_lemmas}.

\subsection{Proof of Theorem~\ref{th:ldec_worstcase}}

The decoder-error bound in Theorem~\ref{th:ldec_worstcase}  involves the second derivative of the composite function $f\circ\enc$. The following lemma ensures that this composite function is belong to the second-order Sobolev space.

\begin{lemma}\label{lem:composition_sobolev}
Let $\enc\in\hild{2}{d}$. Let
$f=[f_1,\ldots,f_m]^T:\mathbb{R}^d\to\mathbb{R}^m$ be twice continuously differentiable, and assume that, for every $\mathbf{x}\in\mathbb{R}^d$ and every $j\in[m]$,
\[
    \|\nabla f_j(\mathbf{x})\|_2\le q_j,
    \qquad
    \|D^2 f_j(\mathbf{x})\|_{\mathrm{op}}\le \mu_j .
\]
Then $f\circ\enc\in\hild{2}{m}$.
\end{lemma}

The detailed proof is provided in Appendix~\ref{app:proof_lemma_omposition_sobolev}. Fix an encoder $\enc\in\hild{2}{d}$ and define
\[
    g(t):=f(\enc(t)).
\]
By Lemma~\ref{lem:composition_sobolev}, we have $g\in\hild{2}{m}$. Fix a non-straggler set $\mathcal{F}\subseteq[N]$ satisfying $|\mathcal{F}|\ge N-S$, and let
\[
    \bm{\beta}_{\mathcal F}:=\{\beta_i\}_{i\in\mathcal F}
\]
denote the corresponding ordered set of returned decoder design points.
Let $\decstar$ be the decoder constructed by solving \eqref{eq:decoder_opt}, and let $\ldec^*$ denote the corresponding decoder loss. Define the decoder error function
\[
    h(t):=\decstar(t)-g(t).
\]
Since $\decstar\in\hild{2}{m}$ and $g\in\hild{2}{m}$, it follows that $h\in\hild{2}{m}$.

We first bound the decoding loss by the supremum norm of $h$:
\begin{align}
    \ldec^*
    &=
    \frac{2}{K}
    \sum_{k=1}^K
    \left\|
        \decstar(\alpha_k)-f(\enc(\alpha_k))
    \right\|_2^2
    \nonumber\\
    &=
    \frac{2}{K}
    \sum_{k=1}^K
    \|h(\alpha_k)\|_2^2
    \nonumber\\
    &\le
    2\|h\|_{\lpd{\infty}{m}}^2 .
    \label{eq:ldec_linf_step}
\end{align}

In the next lemma, we derive an upper bound for $\|h\|_{\lpd{\infty}{m}}$ based on $\|h\|_{\lpd{2}{m}}$ and $\|h'\|_{\lpd{2}{m}}$ leveraging Sobolev interpolation inequalities~\cite{leoni2024first}.

\begin{lemma}\label{lem:decoder_error_supnorm}
For the decoder error function $h$ on $\Omega=(-1,1)$, the following bounds hold:
\begin{align}\label{eq:decoder_error_supnorm}
    \|h\|_{\lpd{\infty}{m}}
    \le
    2
    \|h\|_{\lpd{2}{m}}^{1/2}
    \|h'\|_{\lpd{2}{m}}^{1/2}.
\end{align}
\end{lemma}

Using Lemma~\ref{lem:decoder_error_supnorm} in \eqref{eq:ldec_linf_step}, we obtain
\begin{align}\label{eq:ldec_h_hp_step}
    \ldec^*
    \le
    2\|h\|_{\lpd{\infty}{m}}^2
    \le
    8
    \|h\|_{\lpd{2}{m}}
    \|h'\|_{\lpd{2}{m}} .
\end{align}

See Appendix~\ref{app:proof_lem:decoder_error_supnorm} for the proof. We next bound $\|h\|_{\lpd{2}{m}}$ and $\|h'\|_{\lpd{2}{m}}$ using smoothing-spline approximation theory~\cite{ragozin1983error,utreras1988convergence}. 



\begin{lemma}[Second-order smoothing-spline error bound]
\label{lem:lit_spline_noiseless}
Let $g\in H^2(\Omega;\mathbb{R})$ be a scalar-valued function, and let
$\mathcal{T}=\{t_1,\ldots,t_n\}\subset\Omega$ be an ordered set with $n\ge2$.
For $\lambda\ge0$, let $g_\lambda$ denote the second-order smoothing-spline estimator of $g$ over $\mathcal{T}$, defined by
\begin{align}
    g_\lambda
    =
    \operatorname*{arg\,min}_{u\in\hil{2}}
    \left\{
        \frac{1}{n}
        \sum_{i=1}^n
        \left|
            u(t_i)-g(t_i)
        \right|^2
        +
        \lambda
        \|u''\|_{\lp{2}}^2
    \right\}.
\end{align}
Then there exist constants $G_0,G_1>0$ such that
\begin{align}\label{eq:spline_error_j0}
    \|g-g_\lambda\|_{\lp{2}}^2
    \le
    G_0
    L_{\mathcal{T}}
    \|g''\|_{\lp{2}}^2,
\end{align}
and
\begin{align}\label{eq:spline_error_j1}
    \|(g-g_\lambda)'\|_{\lp{2}}^2
    \le
    G_1
    L_{\mathcal{T}}^{1/2}
    \left(
        1+\frac{L_{\mathcal{T}}}{16}
    \right)^{1/2}
    \|g''\|_{\lp{2}}^2,
\end{align}
where
\(
    L_{\mathcal{T}}
    :=
    p_2\!\left(
        \frac{\delmax{\mathcal{T}}}{\delmin{\mathcal{T}}}
    \right)
    \frac{n\delmax{\mathcal{T}}}{4}
    \lambda
    +
    D\,\delmax{\mathcal{T}}^4.
\)
Here, $p_2(\cdot)$ is a polynomial of degree two with positive coefficients, and $D>0$ is an absolute constant.
\end{lemma}

The proof is provided in Appendix~\ref{app:proof_lem:lit_spline_noiseless}. Applying Lemma~\ref{lem:lit_spline_noiseless} to each output component and summing over $j\in[m]$ gives
\begin{align}
    \|h\|_{\lpd{2}{m}}^2
    &\le
    G_0
    L
    \|g''\|_{\lpd{2}{m}}^2,
    \label{eq:h_l2_spline_fixedF}\\
    \|h'\|_{\lpd{2}{m}}^2
    &\le
    G_1
    L^{1/2}
    \left(
        1+\frac{L}{16}
    \right)^{1/2}
    \|g''\|_{\lpd{2}{m}}^2,
    \label{eq:hp_l2_spline_fixedF}
\end{align}
where
\begin{align}\label{eq:L_F_definition}
    L
    :=
    p_2\!\left(
        \frac{\delmax{\bm{\beta}_{\mathcal F}}}
             {\delmin{\bm{\beta}_{\mathcal F}}}
    \right)
    \frac{
        |\mathcal F|\delmax{\bm{\beta}_{\mathcal F}}
    }{4}
    \declamb
    +
    D
    \delmax{\bm{\beta}_{\mathcal F}}^4 .
\end{align}

Since there are at most $S$ stragglers among the worker nodes, the largest node gap expands by at most $S$ missing nodes, meaning $\delmax{\bm{\beta}_{\mathcal F}}\leqslant (S+1) \cdot \delmax{\bm{\beta}}$. Additionally, since $\bm{\beta}_{\mathcal F} \subseteq \bm{\beta}$, we have $\delmin{\bm{\beta}_{\mathcal F}}\ge\delmin{\bm{\beta}}$.

Since $|\mathcal F|\le N$, and since $p_2(\cdot)$ is a degree-two polynomial with positive coefficients, there exists a constant $J>0$ such that $p_2(x)\le Jx^2$ for all $x\ge1$. Hence,
\begin{align}
    L
    &\le
    J
    \left(
        (S+1)
        \frac{\delmax{\bm{\beta}}}{\delmin{\bm{\beta}}}
    \right)^2
    \frac{
        N(S+1)\delmax{\bm{\beta}}
    }{4}
    \declamb
    +
    D(S+1)^4
    \delmax{\bm{\beta}}^4
    \nonumber\\
    &\le
    C_A
    \left[
        N\declamb
        (S+1)^3
        \frac{\delmax{\bm{\beta}}^3}{\delmin{\bm{\beta}}^2}
        +
        (S+1)^4
        \delmax{\bm{\beta}}^4
    \right]
    \nonumber\\
    &=
    C_A A_{S,N},
    \label{eq:LF_bound_by_ASN}
\end{align}
for a constant $C_A>0$ independent of $N$, $S$, and $\mathcal F$.

Combining \eqref{eq:ldec_h_hp_step}, \eqref{eq:h_l2_spline_fixedF}, and \eqref{eq:hp_l2_spline_fixedF}, we obtain
\begin{align}
    \max_{|\mathcal{F}| \geq N-S} \ldec^*
    &\le
    8
    G_0^{1/2}
    G_1^{1/2}
    \|g''\|_{\lpd{2}{m}}^2
    L^{3/4}
    \left(
        1+\frac{L}{16}
    \right)^{1/4}.
    \label{eq:ldec_LF_bound}
\end{align}
Using \eqref{eq:LF_bound_by_ASN}, and absorbing the resulting constants into the leading constant, gives
\begin{align}
    \max_{|\mathcal{F}| \geq N-S}\ldec^*
    \le
    C
    \|g''\|_{\lpd{2}{m}}^2
    A_{S,N}^{3/4}
    \left(
        1+\frac{A_{S,N}}{16}
    \right)^{1/4},
    \label{eq:ldec_ASN_intermediate}
\end{align}
where $C>0$ is independent of $N$, $S$, and $\mathcal F$.

It remains to bound $\|g''\|_{\lpd{2}{m}}^2$ in terms of the Sobolev norm of the encoder. This is done in the following lemma.

\begin{lemma}
\label{lem:composition_second_derivative_bound}
Under the assumptions of Lemma~\ref{lem:composition_sobolev}, let $g=f\circ\enc$. Then there exists a constant $C_{\mathrm{comp}}>0$ such that
\begin{align}\label{eq:composition_second_derivative_final_bound}
    \|g''\|_{\lpd{2}{m}}^2
    \le
    C_{\mathrm{comp}}
    \max\{\mu^2,q^2\}
    \psi\!\left(
        \|\enc\|_{\hild{2}{d}}^2
    \right),
\end{align}
where $\psi(t)=t+4t^2$.
\end{lemma}

See Appendix~\ref{app:proof_lem:composition_second_derivative_bound} for the proof. Substituting Lemma~\ref{lem:composition_second_derivative_bound} into \eqref{eq:ldec_ASN_intermediate} and absorbing constants into $C_1$ yields
\begin{align}\label{eq:proof_th_2_ldec}
    \max_{|\mathcal{F}| \geq N-S}\ldec^*
    \le
    C_1
    \max\{\mu^2,q^2\}
    A_{S,N}^{3/4}
    \left(
        1+\frac{A_{S,N}}{16}
    \right)^{1/4}
    \psi\!\left(
        \|\enc\|_{\hild{2}{d}}^2
    \right).
\end{align}
Since \eqref{eq:proof_th_2_ldec} is an upper bound for $\dec = \decstar$, we have:
\begin{align}
    \max_{|\mathcal{F}| \geq N-S} \inf_{\dec} \ldec \leq C_1
    \max\{\mu^2,q^2\}
    A_{S,N}^{3/4}
    \left(
        1+\frac{A_{S,N}}{16}
    \right)^{1/4}
    \psi\!\left(
        \|\enc\|_{\hild{2}{d}}^2
    \right).
\end{align}

which completes the proof

\subsection{Proof of Theorem~\ref{th:efficient_encoder}}

We first state an auxiliary lemma showing that the minimizer of the smoothing-spline encoder objective remains in a fixed Sobolev ball, uniformly over the smoothing parameter.

\begin{lemma}
\label{lem:efficient_encoder_uniform_bound}
Let $\{\alpha_k\}_{k=1}^K$ be fixed distinct encoder design points and let $\{\mathbf{x}_k\}_{k=1}^K$ be fixed input data points. For any $\lambda>0$, let
\[
    u_\lambda
    =
    \operatorname*{arg\,min}_{u\in\hild{2}{d}}
    \left\{
        \frac{1}{K}
        \sum_{k=1}^K
        \|u(\alpha_k)-\mathbf{x}_k\|_2^2
        +
        \lambda
        \|u''\|_{\lpd{2}{d}}^2
    \right\}.
\]
Then there exists a constant $R_{\mathrm e}>0$, depending only on
$\{\alpha_k\}_{k=1}^K$ and $\{\mathbf{x}_k\}_{k=1}^K$, such that
\[
    \|u_\lambda\|_{\hild{2}{d}}^2
    \le
    R_{\mathrm e}
\]
for every $\lambda>0$.
\end{lemma}

See Appendix~\ref{app:proof_uniform_encoder_bound} for the proof. We now prove Theorem~\ref{th:efficient_encoder}. Define
\[
    \mathcal{E}(u)
    :=
    \frac{1}{K}
    \sum_{k=1}^K
    \|u(\alpha_k)-\mathbf{x}_k\|_2^2 .
\]
From the loss decomposition and the bounds derived in Theorem~\ref{th:ldec_worstcase} and \eqref{eq:lenc_bound}, for any encoder $u\in\hild{2}{d}$ we have
\begin{align}\label{eq:efficient_encoder_start}
\max_{|\mathcal{F}| \geq N-S}  \inf_{\dec}\mathcal{L}(u,\dec) \leq 
2q^2\mathcal{E}(u) + \max_{|\mathcal{F}| \geq N-S}  \inf_{\dec}\ldec  
    \le
    2q^2\left[\mathcal{E}(u)
    +
    \Gamma_{S,N}
    \psi\!\left(\|u\|_{\hild{2}{d}}^2\right)\right],
\end{align}
where $\Gamma_{S,N}$ is defined in \eqref{eq:lamb_enc} and
\(
    \psi(t)=t+4t^2.
\)

By Lemma~\ref{lem:efficient_encoder_uniform_bound}, there exists a constant $R_{\mathrm e}>0$, depending only on the encoder design points and the input data, such that the minimizer $\encdagger$ of \eqref{eq:efficient_encoder_opt} satisfies
\(
    \|\encdagger\|_{\hild{2}{d}}^2
    \le
    R_{\mathrm e}.
\)
Therefore, for
\(
    t=\|\encdagger\|_{\hild{2}{d}}^2,
\)
we have $0\le t\le R_{\mathrm e}$, and hence
\begin{align}\label{eq:psi_efficient_linear}
    \psi(t)
    =
    t+4t^2
    \le
    t(1+4R_{\mathrm e})
    \le
    (1+4R_{\mathrm e})
    \left(
        R_{\mathrm e}
        +
        \|(\encdagger)''\|_{\lpd{2}{d}}^2
    \right).
\end{align}
Thus, by setting
\(
    m_1:=1
\)
and
\(
    m_2:=4,
\)
we obtain
\[
    \psi\!\left(\|\encdagger\|_{\hild{2}{d}}^2\right)
    \le
    (m_1+m_2R_{\mathrm e})
    \left(
        R_{\mathrm e}
        +
        \|(\encdagger)''\|_{\lpd{2}{d}}^2
    \right).
\]
Multiplying both sides by $\Gamma_{S,N}$ and using the definition
\[
    \lambda_{\mathrm e}
    :=
    \Gamma_{S,N}(m_1+m_2R_{\mathrm e}),
\]
we get
\begin{align}\label{eq:efficient_encoder_regularizer_bound}
    \Gamma_{S,N}
    \psi\!\left(\|\encdagger\|_{\hild{2}{d}}^2\right)
    \le
    \lambda_{\mathrm e}\cdot
    \left(
        R_{\mathrm e}
        +
        \|(\encdagger)''\|_{\lpd{2}{d}}^2
    \right).
\end{align}

Finally, substituting $u=\encdagger$ into \eqref{eq:efficient_encoder_start} and applying \eqref{eq:efficient_encoder_regularizer_bound} yields
\begin{align}
    \max_{|\mathcal{F}| \geq N-S}  \inf_{\dec}\mathcal{L}(\encdagger,\dec) \leq 
    2q^2\left[
    \frac{1}{K}
    \sum_{k=1}^K
    \left\|
        \encdagger(\alpha_k)-\mathbf{x}_k
    \right\|_2^2
    + 
    \lambda_{\mathrm e}
    \left(
        R_{\mathrm e}
        +
        \|(\encdagger)''\|_{\lpd{2}{d}}^2
    \right)\right].
\end{align}
Using the fact that
\[
    \inf_{\substack{
        \enc
    }}
    \ 
    \max_{|\mathcal{F}| \geq N-S}  \inf_{\substack{
        \dec 
    }}
    \mathcal{L}( \enc, \dec) \leq \max_{|\mathcal{F}| \geq N-S}  \inf_{\substack{
        \dec 
    }}
    \mathcal{L}( \encdagger, \dec)
\]
completes the proof.

\subsection{Proof of Theorem~\ref{th:ldec_worst_asymp}}

Recall from \eqref{eq:const_th2} that
\[
    A_{S,N}
    =
    (S+1)^4\delmax{\bm{\beta}}^4
    +
    N\declamb
    (S+1)^3
    \frac{\delmax{\bm{\beta}}^3}{\delmin{\bm{\beta}}^2}.
\]
By Assumption~\ref{ass:decoder_spacing} we have
\(
    \delmax{\bm{\beta}}
    \le
    B_2N^{-\gamma_2}
    \) and \(
    \frac{\delmax{\bm{\beta}}}{\delmin{\bm{\beta}}}
    \le
    B_1N^{\gamma_1}.
\)
Thus,
\begin{align}
    (S+1)^4\delmax{\bm{\beta}}^4
    &\le
    B_2^4
    (S+1)^4
    N^{-4\gamma_2},
    \label{eq:A_asymp_first_term}
\end{align}
and 
\begin{align}
    N\declamb
    (S+1)^3
    \frac{\delmax{\bm{\beta}}^3}{\delmin{\bm{\beta}}^2}
    &=
    N\declamb
    (S+1)^3
    \delmax{\bm{\beta}}
    \left(
        \frac{\delmax{\bm{\beta}}}{\delmin{\bm{\beta}}}
    \right)^2
    \nonumber\\
    &\le
    B_1^2B_2
    (S+1)^3
    N^{1-\gamma_2+2\gamma_1}
    \declamb \nonumber \\
    &\lec{(a)}{\le}
    B_1^2B_2
    (S+1)^3
    N^{-4\gamma_2},
    \nonumber \\
    &\lec{(b)}{\le}
    B_1^2B_2 \frac{(S+1)^4}{N^{4\gamma_2}}.
    \label{eq:A_asymp_final_bound}
\end{align}
where (a) follows from the assumption $\declamb \le N^{-(2\gamma_1+3\gamma_2+1)}$, and (b) follows by the fact that $(S+1)^3\le (S+1)^4$.

Combining \eqref{eq:A_asymp_first_term} and \eqref{eq:A_asymp_final_bound}  and defining $C_A:=\max\{B_1^2B_2, B_2^4\}$, we obtain:
\begin{align}
    A_{S,N} \le C_A\frac{(S+1)^4}{N^{4\gamma_2}}.
\end{align}

Substituting \eqref{eq:A_asymp_final_bound} into Theorem~\ref{th:ldec_worstcase} gives
\begin{align}
    \max_{|\mathcal{F}| \geq N-S}  \inf_{\dec}\ldec
    &\le
    C_1
    \max\{\mu^2,q^2\}
    \left(
        C_A\frac{(S+1)^4}{N^{4\gamma_2}}
    \right)^{3/4}
    \left(
        1+
        C_A
        \frac{(S+1)^4}{N^{4\gamma_2}}
    \right)^{1/4}
    \psi\!\left(\|\enc\|_{\hild{2}{d}}^2\right)
    \nonumber\\
    &\le
    C_2
    \max\{\mu^2,q^2\}
    \left(
        \frac{S+1}{N^{\gamma_2}}
    \right)^3
    \left(
        1+
        \left(
            \frac{S+1}{N^{\gamma_2}}
        \right)^4
    \right)^{1/4}
    \psi\!\left(\|\enc\|_{\hild{2}{d}}^2\right),
    \label{eq:ldec_asymp_with_extra_factor}
\end{align}
where $C_2:=C_1C_A^{3/4}\cdot \max\{1,C_A\}^{1/4}$.

\subsection{Proof of Corollary~\ref{cor:worst_convergence}}

Let $\widetilde{\enc}$ denote the natural cubic spline interpolant satisfying
\begin{align}\label{eq:nat_enc}
    \widetilde{\enc}(\alpha_k)=\mathbf{x}_k,
    \qquad k\in[K].
\end{align}
Since the encoder design points and input data are fixed, both
$\|\widetilde{\enc}\|_{\hild{2}{d}}$ and
$\|\widetilde{\enc}''\|_{\lpd{2}{d}}$ are independent of $N$. Now consider the encoder $\encstar$ designed according to \eqref{eq:enc_opt}. From the loss decomposition and the encoder-design upper bound in \eqref{eq:l_unified_bound}, we have
\begin{align}
    \max_{|\mathcal{F}| \geq N-S}  \inf_{\substack{
        \dec 
    }}  \mathcal{L}( \encstar, \dec)
    \le
    2q^2\left[
    \frac{1}{K}
    \sum_{k=1}^K
    \left\|
        \encstar(\alpha_k)-\mathbf{x}_k
    \right\|_2^2
    +
    \Gamma_{S,N}
    \psi\!\left(
        \|\encstar\|_{\hild{2}{d}}^2
    \right)\right].
\end{align}
 By the optimality of $\encstar$ in \eqref{eq:enc_opt}, comparing with $\widetilde{\enc}$ gives
\begin{align}
    \frac{1}{K}
    \sum_{k=1}^K
    \left\|
        \encstar(\alpha_k)-\mathbf{x}_k
    \right\|_2^2
    +
    \Gamma_{S,N}
    \psi\!\left(
        \|\encstar\|_{\hild{2}{d}}^2
    \right)
    &\le
    \frac{1}{K}
    \sum_{k=1}^K
    \left\|
        \widetilde{\enc}(\alpha_k)-\mathbf{x}_k
    \right\|_2^2
    +
    \Gamma_{S,N}
    \psi\!\left(
        \|\widetilde{\enc}\|_{\hild{2}{d}}^2
    \right)
    \nonumber\\
    &=
    \Gamma_{S,N}
    \psi\!\left(
        \|\widetilde{\enc}\|_{\hild{2}{d}}^2
    \right).
\end{align}
For $S=o(N^{\gamma_2})$, Theorem~\ref{th:ldec_worst_asymp} implies
\[
    \Gamma_{S,N}\cdot \psi\!\left(
        \|\widetilde{\enc}\|_{\hild{2}{d}}^2
    \right)
    =
    \mathcal{O}\!\left(
        \frac{(S+1)^3}{N^{3\gamma_2}}
    \right),
\]
because
\(
    (
        1+
        \left(
            \frac{S+1}{N^{\gamma_2}}
        \right)^4
    )^{1/4} \le 2
\)
for sufficiently large $N$. Thus, we obtain
\begin{align}\label{eq:wc_con_proof_1}
    \max_{|\mathcal{F}| \geq N-S}  \inf_{\substack{
        \dec 
    }}  \mathcal{L}( \encstar, \dec)
    =
    \mathcal{O}\!\left(
        \frac{(S+1)^3}{N^{3\gamma_2}}
    \right).
\end{align}

Now consider the efficient encoder design in Theorem~\ref{th:efficient_encoder}. From \eqref{eq:efficient_encoder_loss_bound},
\begin{align}
    \max_{|\mathcal{F}| \geq N-S}  \inf_{\substack{
        \dec 
    }}  \mathcal{L}( \encdagger, \dec)
    \le
    2q^2\left[
    \frac{1}{K}
    \sum_{k=1}^K
    \left\|
        \encdagger(\alpha_k)-\mathbf{x}_k
    \right\|_2^2
    +
    \lambda_{\mathrm e}
    \left(
        R_{\mathrm e}
        +
        \|(\encdagger)''\|_{\lpd{2}{d}}^2
    \right)\right].
\end{align}
By the optimality of $\encdagger$ in \eqref{eq:efficient_encoder_opt} and the interpolation property  $\widetilde{\enc}(\alpha_k)=\mathbf{x}_k$ in \eqref{eq:nat_enc}, we obtain
\begin{align}
    \frac{1}{K}
    \sum_{k=1}^K
    \left\|
        \encdagger(\alpha_k)-\mathbf{x}_k
    \right\|_2^2
    +
    \lambda_{\mathrm e}
    \|(\encdagger)''\|_{\lpd{2}{d}}^2
    &\le
    \frac{1}{K}
    \sum_{k=1}^K
    \left\|
        \widetilde{\enc}(\alpha_k)-\mathbf{x}_k
    \right\|_2^2
    +
    \lambda_{\mathrm e}
    \|\widetilde{\enc}''\|_{\lpd{2}{d}}^2 \nonumber \\
    &\lec{}{=}
    \lambda_{\mathrm e}
    \|\widetilde{\enc}''\|_{\lpd{2}{d}}^2.
\end{align}
Therefore,
\begin{align}
    \max_{|\mathcal{F}| \geq N-S}  \inf_{\substack{
        \dec 
    }}  \mathcal{L}( \encdagger, \dec)
    \le
    \lambda_{\mathrm e}
    \left(
        R_{\mathrm e}
        +
        \|\widetilde{\enc}''\|_{\lpd{2}{d}}^2
    \right).
\end{align}
Since $R_{\mathrm e}$ and $\|\widetilde{\enc}''\|_{\lpd{2}{d}}^2$ are independent of $N$, and since
\[
    \lambda_{\mathrm e}
    =
    \Gamma_{S,N}(m_1+m_2R_{\mathrm e})
    =
    \mathcal{O}\!\left(
        \frac{(S+1)^3}{N^{3\gamma_2}}
    \right)
\]
for $S=o(N^{\gamma_2})$, the same rate follows:
\begin{align}\label{eq:wc_con_proof_2}
    \max_{|\mathcal{F}| \geq N-S}  \inf_{\substack{
        \dec 
    }}  \mathcal{L}( \encdagger, \dec)
    =
    \mathcal{O}\!\left(
        \frac{(S+1)^3}{N^{3\gamma_2}}
    \right).
\end{align}

Combining \eqref{eq:wc_con_proof_1} and \eqref{eq:wc_con_proof_2}, we obtain the desired convergence rate. Indeed, since $\objwc$ is the infimum over all admissible encoders, it is upper-bounded by the performance of either $\encstar$ or $\encdagger$. Hence,
\begin{align}
    \objwc
    &\le
    \min\left\{
        \max_{|\mathcal{F}|\ge N-S}
        \inf_{\dec}
        \mathcal{L}(\encstar,\dec),
        \,
        \max_{|\mathcal{F}|\ge N-S}
        \inf_{\dec}
        \mathcal{L}(\encdagger,\dec)
    \right\} =
    \mathcal{O}\!\left(
        \frac{(S+1)^3}{N^{3\gamma_2}}
    \right).
\end{align}

\subsection{Proof of Theorem~\ref{th:ldec_prob_nonasymp}}

Fix a realization of the non-straggler set $\mathcal{F}\subseteq[N]$ with $|\mathcal{F}|\ge2$, and let
\(
    \mathcal{S}_{\mathcal F}:=[N]\setminus\mathcal{F}
\)
be the corresponding straggler set. Let $R_{\mathcal F}$ denote the maximum number of consecutive indices in $\mathcal{S}_{\mathcal F}$. Also define the available decoder design points
\(
    \bm{\beta}_{\mathcal F}:=\{\beta_i\}_{i\in\mathcal F},
\)
ordered increasingly.

The proof follows the same steps as the proof of Theorem~\ref{th:ldec_worstcase}. The only difference is the way we control the geometry of the available at the decoder design points. In the worst-case setting, if at most $S$ workers straggle, then the maximum number of consecutive straggler worker nodes is at most $S+1$. Here, for the realized straggler pattern $\mathcal{S}_{\mathcal F}$, the largest number of consecutive straggler worker nodes is $R_{\mathcal F}$. Therefore,
\begin{align}\label{eq:prob_active_spacing_max}
    \delmax{\bm{\beta}_{\mathcal F}}
    \le
    (R_{\mathcal F}+1)\delmax{\bm{\beta}},
    \qquad 
    \delmin{\bm{\beta}_{\mathcal F}}
    \ge
    \delmin{\bm{\beta}}.
\end{align}

Now define
\(
    g(t):=f(\enc(t))
\)
and
\(
    h(t):=\decstar(t)-g(t).
\)
By Lemma~\ref{lem:composition_sobolev}, $g\in\hild{2}{m}$, and since $\decstar\in\hild{2}{m}$, we have $h\in\hild{2}{m}$. Repeating the same supremum-norm and smoothing-spline error arguments used in \eqref{eq:ldec_linf_step} to \eqref{eq:ldec_ASN_intermediate} in the proof of Theorem~\ref{th:ldec_worstcase}, but with the  set $\bm{\beta}_{\mathcal F}$, yields
\begin{align}\label{eq:prob_ldec_intermediate}
    \ldec(\mathcal{F})
    \le
    C
    \|g''\|_{\lpd{2}{m}}^2
    L_{\mathcal F}^{3/4}
    \left(
        1+\frac{L_{\mathcal F}}{16}
    \right)^{1/4},
\end{align}
where
\begin{align}\label{eq:prob_LF_def}
    L_{\mathcal F}
    :=
    p_2\!\left(
        \frac{\delmax{\bm{\beta}_{\mathcal F}}}
             {\delmin{\bm{\beta}_{\mathcal F}}}
    \right)
    \frac{
        |\mathcal F|\delmax{\bm{\beta}_{\mathcal F}}
    }{4}
    \declamb
    +
    D
    \delmax{\bm{\beta}_{\mathcal F}}^4 .
\end{align}
Here, $p_2(\cdot)$ is a degree-two polynomial with positive coefficients, and $C,D>0$ are constants independent of $N$ and $\mathcal F$.

Using $|\mathcal F|\le N$, \eqref{eq:prob_active_spacing_max}, and using the fact that $p_2(x)\le Jx^2$ for all $x\ge1$ for some constant $J>0$, we obtain
\begin{align}
    L_{\mathcal F}
    &\le
    J
    \left(
        (R_{\mathcal F}+1)
        \frac{\delmax{\bm{\beta}}}{\delmin{\bm{\beta}}}
    \right)^2
    \frac{
        N(R_{\mathcal F}+1)\delmax{\bm{\beta}}
    }{4}
    \declamb
    +
    D(R_{\mathcal F}+1)^4
    \delmax{\bm{\beta}}^4
    \nonumber\\
    &\le
    C_P
    \left[
        N\declamb
        (R_{\mathcal F}+1)^3
        \frac{\delmax{\bm{\beta}}^3}{\delmin{\bm{\beta}}^2}
        +
        (R_{\mathcal F}+1)^4
        \delmax{\bm{\beta}}^4
    \right]
    \nonumber\\
    &=
    C_P A_{\mathcal F,N},
    \label{eq:prob_LF_bound_by_AFN}
\end{align}
for a constant $C_P>0$ independent of $N$ and $\mathcal F$.

Substituting \eqref{eq:prob_LF_bound_by_AFN} into \eqref{eq:prob_ldec_intermediate}, using Lemma~\ref{lem:composition_second_derivative_bound} and absorbing constants into $C_1$ yields
\[
    \ldec(\mathcal{F})
    \le
    C_1
    \max\{\mu^2,q^2\}
    A_{\mathcal F,N}^{3/4}
    \left(
        1+\frac{A_{\mathcal F,N}}{16}
    \right)^{1/4}
    \psi\!\left(
        \|\enc\|_{\hild{2}{d}}^2
    \right).
\]
Taking expectation over $\mathcal{F}\sim F_{p, N}$ from both sides proves Theorem~\ref{th:ldec_prob_nonasymp}.

\subsection{Proof of Theorem~\ref{th:prob_dec_bound}}
\label{sec:proof_prob_dec_bound}

We begin from the non-asymptotic bound in Theorem~\ref{th:ldec_prob_nonasymp}. For every realized non-straggler set $\mathcal{F}$, we have
\begin{align}\label{eq:proof_prob_start}
    \ldec(\mathcal{F})
    \le
    C_1
    \max\{\mu^2,q^2\}
    A_{\mathcal{F},N}^{3/4}
    \left(
        1+\frac{A_{\mathcal{F},N}}{16}
    \right)^{1/4}
    \psi\!\left(\|\enc\|_{\hild{2}{d}}^2\right),
\end{align}
where
\begin{align}\label{eq:proof_prob_A_def}
    A_{\mathcal{F},N}
    =
    (R_{\mathcal{F}}+1)^4
    \delmax{\bm{\beta}}^4
    +
    N\declamb
    (R_{\mathcal{F}}+1)^3
    \frac{\delmax{\bm{\beta}}^3}{\delmin{\bm{\beta}}^2}.
\end{align}
Here $R_{\mathcal{F}}$ is the longest run of consecutive stragglers in the realized straggler pattern.

We first upper-bound $A_{\mathcal{F},N}$ using Assumption~\ref{ass:decoder_spacing}. Since
\(
    \delmax{\bm{\beta}}
    \le
    B_2N^{-\gamma_2}
\) and 
\(
    \frac{\delmax{\bm{\beta}}}{\delmin{\bm{\beta}}}
    \le
    B_1N^{\gamma_1},
\)
we have:
\begin{align}
\label{eq:asymp_dec_bound_eq}
    (R_{\mathcal{F}}+1)^4
    \delmax{\bm{\beta}}^4
    +
    N\declamb
    (R_{\mathcal{F}}+1)^3
    \frac{\delmax{\bm{\beta}}^3}{\delmin{\bm{\beta}}^2} &\lec{}{\le} B_2^4
    (R_{\mathcal{F}}+1)^4
    N^{-4\gamma_2} + B_1^2B_2
    (R_{\mathcal{F}}+1)^3
    N^{1-\gamma_2+2\gamma_1}
    \declamb
    \nonumber \\ 
    &\lec{(a)}{\le }
    B_2^4
    (R_{\mathcal{F}}+1)^4
    N^{-4\gamma_2} + B_1^2B_2
    (R_{\mathcal{F}}+1)^3
    N^{-4\gamma_2} \nonumber \\
    &\lec{(b)}{\le}
    C_A\frac{(R_{\mathcal{F}}+1)^4}{N^{4\gamma_2}},
\end{align}
where $C_A:=\max\{B_2^4, B_2^2B_2\}$, (a) comes from $\declamb \le N^{-(2\gamma_1+3\gamma_2+1)}$, and (b) follows by $(R_{\mathcal{F}}+1)^3 \le (R_{\mathcal{F}}+1)^4$.

Substituting \eqref{eq:asymp_dec_bound_eq} into \eqref{eq:proof_prob_start}, and absorbing constants into $C_5>0$, gives
\begin{align}\label{eq:proof_prob_ldec_R}
    \ldec(\mathcal{F})
    \le
    C_5
    \max\{\mu^2,q^2\}
    \frac{(R_{\mathcal{F}}+1)^3}{N^{3\gamma_2}}
    \left[
        1+
        \frac{(R_{\mathcal{F}}+1)^4}{N^{4\gamma_2}}
    \right]^{1/4}
    \psi\!\left(\|\enc\|_{\hild{2}{d}}^2\right).
\end{align}

We now take expectation over $\mathcal{F}\sim\mathsf{F}_{p,N}$. To control the moments of $R_{\mathcal{F}}$, we use the following standard longest-run estimate.

\begin{lemma}\label{lem:longest_run_moment}
Let $R_{\mathcal{F}}$ be the longest run of consecutive stragglers in a length-$N$ Bernoulli sequence, where each worker straggles independently with probability $p\in(0,1)$. Then there exist constants $C_R>0$ and $n_0\in\mathbb{N}$, independent of $N$, such that for all $N>n_0$,
\begin{align}\label{eq:longest_run_fourth_moment}
    \mathbb{E}_{\mathcal{F}\sim\mathsf{F}_{p,N}}
    \!\left[
        (R_{\mathcal{F}}+1)^4
    \right]
    \le
    C_R
    \left(
        \log_{1/p}\!\big((1-p)N\big)
    \right)^4 .
\end{align}
\end{lemma}

The proof is given in Appendix~\ref{app:proof_longest_run_2}. By the Jensen inequality, we obtain
\begin{align}\label{eq:longest_run_third_moment}
    \mathbb{E}
    \!\left[
        (R_{\mathcal{F}}+1)^3
    \right] = \mathbb{E}
    \!\left[\left(
        (R_{\mathcal{F}}+1)^4\right)^{3/4}
    \right]
    \le
    \left(
        \mathbb{E}
        \!\left[
            (R_{\mathcal{F}}+1)^4
        \right]
    \right)^{3/4}
    \le
    C_R^4
    \left(
        \log_{1/p}\!\big((1-p)N\big)
    \right)^3 .
\end{align}

Thus, the expectation of the random factor in \eqref{eq:proof_prob_ldec_R} satisfies
\begin{align}
&\mathbb{E}
\left[
    (R_{\mathcal{F}}+1)^3
    \left(
        1+
        \frac{(R_{\mathcal{F}}+1)^4}{N^{4\gamma_2}}
    \right)^{1/4}
\right]
\lec{(a)}{\le}
\mathbb{E}
\left[
    (R_{\mathcal{F}}+1)^3
\right]
+
\frac{1}{N^{\gamma_2}}
\mathbb{E}
\left[
    (R_{\mathcal{F}}+1)^4
\right],
\end{align}
where (a) follows from $(1+x)^{1/4}\le 1+x^{1/4}$ for $x\ge0$.
By \eqref{eq:longest_run_fourth_moment} and \eqref{eq:longest_run_third_moment}, we obtain
\begin{align}
&\mathbb{E}
\left[
    (R_{\mathcal{F}}+1)^3
    \left(
        1+
        \frac{(R_{\mathcal{F}}+1)^4}{N^{4\gamma_2}}
    \right)^{1/4}
\right]
\nonumber\\
&\qquad\le
C_R^4
\left(
    \log_{1/p}\!\big((1-p)N\big)
\right)^3
+
\frac{C_R^4}{N^{\gamma_2}}
\left(
    \log_{1/p}\!\big((1-p)N\big)
\right)^4 .
\end{align}
Since $\gamma_2>0$, the factor
\(
    \frac{\log_{1/p}((1-p)N)}{N^{\gamma_2}}
\)
is bounded for sufficiently large $N$. Therefore, there exists $n_0\in \mathbb{N}$, such that for $N> n_0$ we have:
\begin{align}\label{eq:proof_prob_random_factor_final}
&\mathbb{E}
\left[
    (R_{\mathcal{F}}+1)^3
    \left(
        1+
        \frac{(R_{\mathcal{F}}+1)^4}{N^{4\gamma_2}}
    \right)^{1/4}
\right]
\nonumber\\
&\qquad\le
C_R^4
\left(
    \log_{1/p}\!\big((1-p)N\big)
\right)^3 .
\end{align}

Taking expectation in \eqref{eq:proof_prob_ldec_R}, applying \eqref{eq:proof_prob_random_factor_final}, and absorbing all constants into $C_3$, we obtain, for all $N>n_0$,
\begin{align}
    \mathbb{E}_{\mathcal{F}\sim\mathsf{F}_{p,N}}
    \!\left[
        \ldec(\mathcal{F})
    \right]
    \le
    C_3
    \max\{\mu^2,q^2\}
    \frac{
        \left(
            \log_{1/p}\!\big((1-p)N\big)
        \right)^3
    }{
        N^{3\gamma_2}
    }
    \psi\!\left(\|\enc\|_{\hild{2}{d}}^2\right),
\end{align}
where $C_3>0$ is independent of $N$. This proves Theorem~\ref{th:prob_dec_bound}.

\subsection{Proof of Corollary~\ref{cor:prob_convergence}}

Recall that, in the probabilistic setting, the efficient encoder in
\eqref{eq:prob_efficient_encoder_opt} is constructed as
\begin{align}\label{eq:proof_prob_eff_enc}
    \encdagger
    =
    \operatorname*{arg\,min}_{u\in\hild{2}{d}}
    \left\{
        \frac{1}{K}
        \sum_{k=1}^K
        \left\|
            u(\alpha_k)-\mathbf{x}_k
        \right\|_2^2
        +
        \enclamb
        \|u''\|_{\lpd{2}{d}}^2
    \right\},
\end{align}
where the encoder smoothing parameter is
\(
    \enclamb
    =
    \Gamma_{p,N}(m_1+m_2R_{\mathrm e}).
\)
Here $R_{\mathrm e},m_1,m_2$ are the constants from
Theorem~\ref{th:efficient_encoder}, and $\Gamma_{p,N}:=\frac{C_1}{2q^2}\max\{\mu^2,q^2\}\cdot
    \mathbb{E}_{\mathcal{F}\sim\mathsf{F}_{p,N}}
    [
        A_{\mathcal{F},N}^{3/4}
        (
            1+\frac{A_{\mathcal{F},N}}{16}
        )^{1/4}
    ]$ is the coefficient multiplying
$\psi(\|\enc\|_{\hild{2}{d}}^2)$ in the upper bound.

By Theorem~\ref{th:prob_dec_bound}, for sufficiently large $N$,
\[
    \Gamma_{p,N}
    =
    \mathcal{O}\!\left(
        \frac{\log_{1/p}^3(N)}{N^{3\gamma_2}}
    \right).
\]
Since $R_{\mathrm e},m_1,m_2$ are independent of $N$, it follows that
\begin{align}\label{eq:proof_prob_lambda_rate}
    \enclamb
    =
    \mathcal{O}\!\left(
        \frac{\log_{1/p}^3(N)}{N^{3\gamma_2}}
    \right).
\end{align}

Let $\widetilde{\enc}$ be the natural cubic spline interpolant of the input data, i.e.,
\[
    \widetilde{\enc}(\alpha_k)=\mathbf{x}_k,
    \qquad k\in[K].
\]
Since the encoder design points and input data are fixed,
$\|\widetilde{\enc}''\|_{\lpd{2}{d}}^2$ is independent of $N$.

By the optimality of $\encdagger$ in \eqref{eq:proof_prob_eff_enc}, comparing
$\encdagger$ with $\widetilde{\enc}$ gives
\begin{align}
    \frac{1}{K}
    \sum_{k=1}^K
    \left\|
        \encdagger(\alpha_k)-\mathbf{x}_k
    \right\|_2^2
    +
    \enclamb
    \|(\encdagger)''\|_{\lpd{2}{d}}^2
    &\le
    \frac{1}{K}
    \sum_{k=1}^K
    \left\|
        \widetilde{\enc}(\alpha_k)-\mathbf{x}_k
    \right\|_2^2
    +
    \enclamb
    \|\widetilde{\enc}''\|_{\lpd{2}{d}}^2
    \nonumber\\
    &=
    \enclamb
    \|\widetilde{\enc}''\|_{\lpd{2}{d}}^2 .
    \label{eq:proof_prob_eff_opt_compare}
\end{align}

On the other hand, the efficient-encoder surrogate bound gives
\begin{align}\label{eq:proof_prob_eff_loss_bound}
    \mathbb{E}_{\mathcal{F}\sim\mathsf{F}_{p,N}}
    \!\left[
        \mathcal{L}(\fhat_{\mathcal{F}})
    \right]
    \le
    2q^2\left[
    \frac{1}{K}
    \sum_{k=1}^K
    \left\|
        \encdagger(\alpha_k)-\mathbf{x}_k
    \right\|_2^2
    +
    \enclamb
    \left(
        R_{\mathrm e}
        +
        \|(\encdagger)''\|_{\lpd{2}{d}}^2
    \right)\right].
\end{align}
Using \eqref{eq:proof_prob_eff_opt_compare} in
\eqref{eq:proof_prob_eff_loss_bound}, we obtain
\begin{align}
    \mathbb{E}_{\mathcal{F}\sim\mathsf{F}_{p,N}}
    \!\left[
        \mathcal{L}(\fhat_{\mathcal{F}})
    \right]
    &\le
    \enclamb
    \left(
        R_{\mathrm e}
        +
        \|\widetilde{\enc}''\|_{\lpd{2}{d}}^2
    \right).
\end{align}
The term
\(
    R_{\mathrm e}
    +
    \|\widetilde{\enc}''\|_{\lpd{2}{d}}^2
\)
depends only on the fixed encoder design points and input data, and is therefore independent of
$N$. Combining this with \eqref{eq:proof_prob_lambda_rate} yields
\[
    \mathbb{E}_{\mathcal{F}\sim\mathsf{F}_{p,N}}
    \!\left[
        \mathcal{L}(\fhat_{\mathcal{F}})
    \right]
    =
    \mathcal{O}\!\left(
        \frac{\log_{1/p}^3(N)}{N^{3\gamma_2}}
    \right).
\]
This proves Corollary~\ref{cor:prob_convergence}.
\section{Experimental Results}\label{sec:exp_result}

In this section, we present the empirical evaluation of the proposed \gcc{} framework. We first detail our experimental methodology, including the selected computing tasks, baseline methods, and evaluation metrics. Subsequently, we present distinct sets of experiments designed to validate our theoretical convergence bounds and demonstrate the performance superiority of \gcc{}. All experiments are implemented using PyTorch~\cite{paszke2019pytorch} and executed on a single-GPU environment.

\subsection{Experimental Setup}

\textbf{Models and Workloads:} To evaluate the versatility of the proposed framework, we consider two classes of target computation tasks $f(\cdot)$: deep neural networks (DNNs) and polynomial functions.
\begin{itemize}
    \item \textit{Deep Neural Networks:} We evaluate \gcc{} on two architectures representing different levels of model complexity. As a relatively shallow architecture, we consider LeNet-5~\cite{lecun1998gradient}, which contains approximately $6\times10^4$ parameters, on the MNIST dataset~\cite{lecun2010mnist} for handwritten-digit classification. To evaluate the scalability of \gcc{} to large, high-dimensional models, we consider the Vision Transformer ViT-B/16~\cite{dosovitskiy2020image}, which contains approximately 80 million parameters, on the ImageNet-1K dataset~\cite{deng2009imagenet}, a large-scale benchmark for 1,000-class image classification. For both architectures, the output of the final softmax layer is treated as the output of the target function. Accordingly, the target mappings are
\(
    f:\mathbb{R}^{32^2}\to\mathbb{R}^{10}
\)
for LeNet-5 and
\(
    f:\mathbb{R}^{3\times224^2}\to\mathbb{R}^{1000}
\)
for ViT-B/16.
    \item \textit{Polynomial Functions:} To benchmark against classical exact coded computing schemes, we also evaluate \gcc{} on a high-dimensional multivariate polynomial computation task.
\end{itemize}

\textbf{Baselines for Comparison:} We evaluate \gcc{} against two primary baselines:
\begin{itemize}
    \item \textit{Berrut Approximate Coded Computing (\texttt{BACC}):}
Introduced in~\cite{jahani2022berrut}, \texttt{BACC} serves as a baseline coded-computing scheme for approximating general nonlinear functions.

\item \textit{Lagrange Coded Computing (\texttt{LCC}):}
Proposed in~\cite{yu2019lagrange}, \texttt{LCC} is an exact-recovery coded-computing scheme designed for polynomial computations and is provably optimal for this class of tasks.

\end{itemize}

\textbf{Hyper-parameters:} 
For \gcc{}, there are two scalar hyper-parameters, $\enclamb$ and $\declamb$, that must be selected. We determine these parameters via cross-validation using a logarithmic grid search. Empirically, however, the performance of \gcc{} is not particularly sensitive to their precise values. In particular, the limiting choice $\enclamb\to0$ and $\declamb\to0$, which yields minimum-curvature natural cubic spline interpolants for both the encoder and decoder (see Remark~\ref{rem:natural_spline_decoder}), outperforms the baseline methods in all experiments.
Accordingly, we use this parameter-free limiting choice in some experiments, while selecting $\enclamb$ and $\declamb$ through cross-validation in the remaining experiments.

\begin{figure*}[t]
     \centering
     \begin{subfigure}[b]{0.49\textwidth}
         \centering
         \includegraphics[width=1\textwidth]{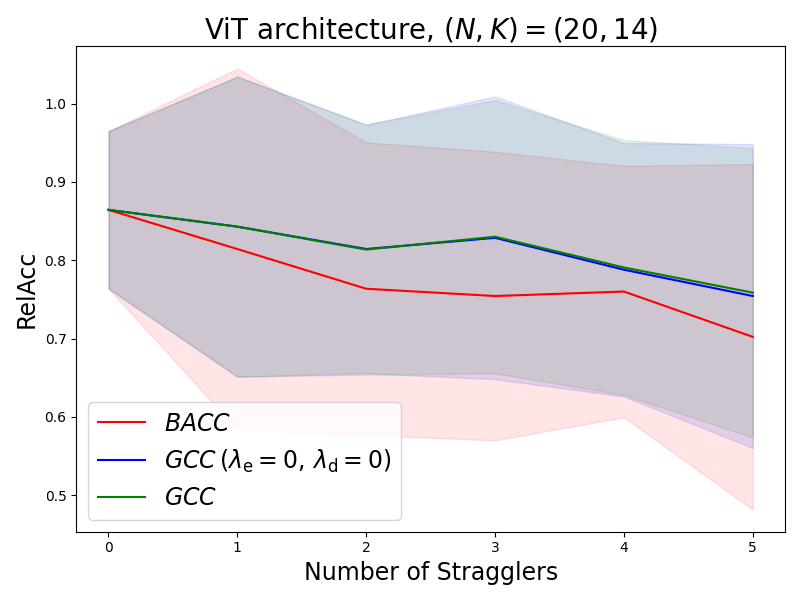}
         \caption{ViT: Relative Accuracy}
         \label{fig:racc_vit}
     \end{subfigure}
     \hfill
     \begin{subfigure}[b]{0.49\textwidth}
         \centering
         \includegraphics[width=\textwidth]{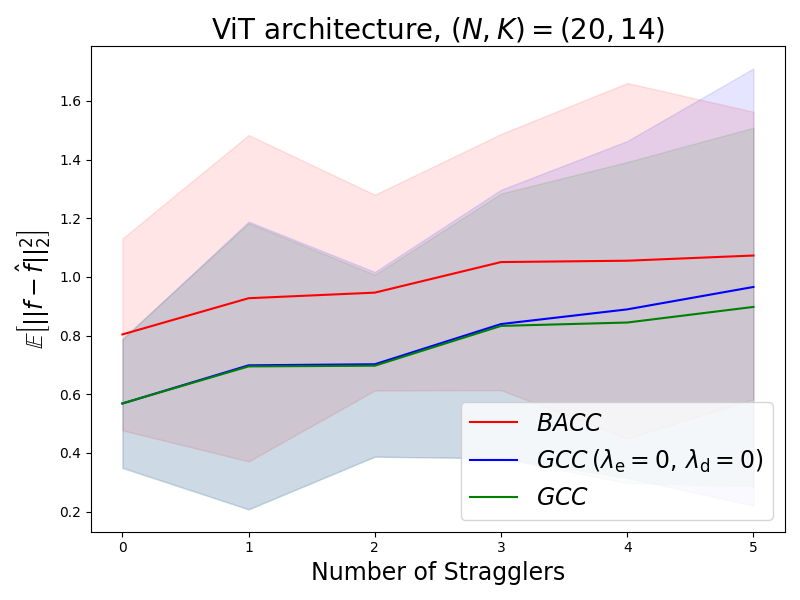}
         \caption{ViT: Mean Squared Error}
         \label{fig:mse_vit}
     \end{subfigure}
     
     \vspace{0.3cm}
     
     \begin{subfigure}[b]{0.49\textwidth}
         \centering
         \includegraphics[width=1\textwidth]{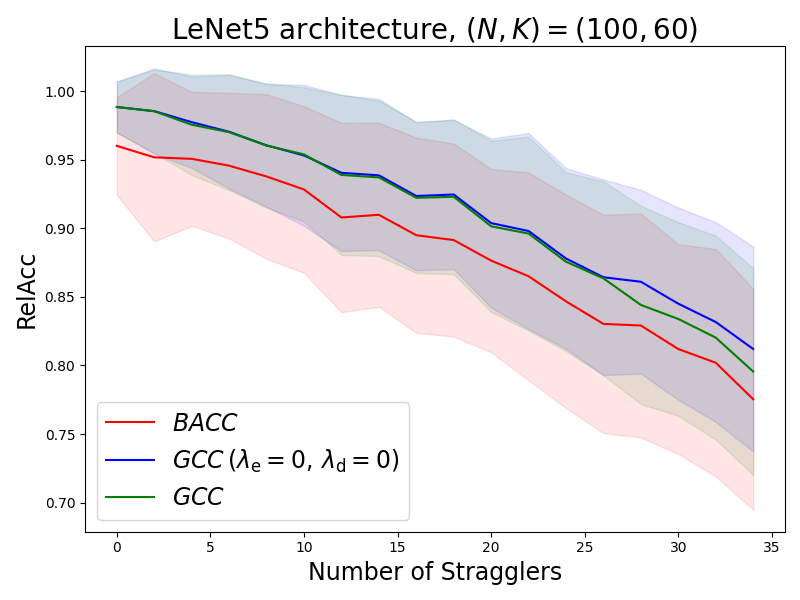}
         \caption{LeNet5: Relative Accuracy}
         \label{fig:racc_lenet}
     \end{subfigure}
     \hfill
     \begin{subfigure}[b]{0.49\textwidth}
         \centering
         \includegraphics[width=1\textwidth]{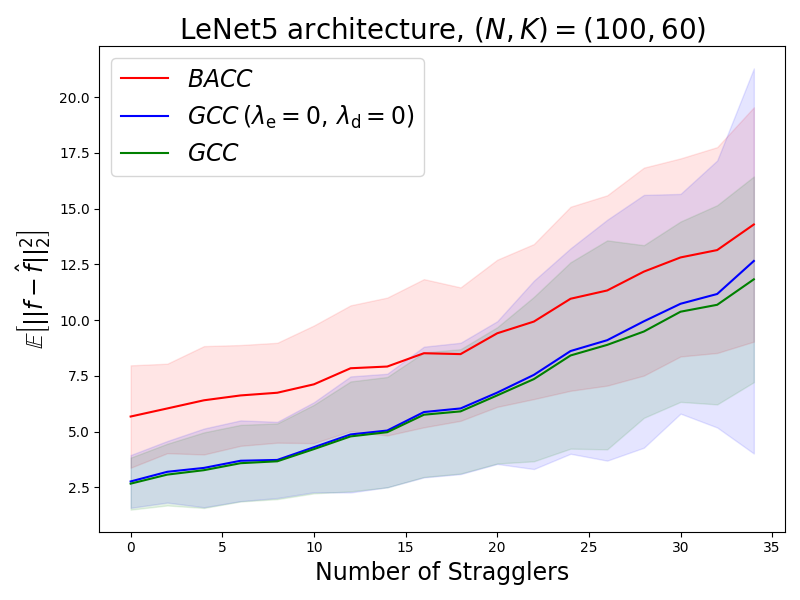}
         \caption{LeNet5: Mean Squared Error}
         \label{fig:mse_lenet}
     \end{subfigure}
     \caption{Performance comparison of \gcc{}, natural spline interpolating \gcc{} ($\enclamb=\declamb=0$), and \texttt{BACC} across varying numbers of stragglers ($S$). The top row displays results for the ViT architecture, while the bottom row shows results for LeNet5. Shaded regions denote the 95\% confidence intervals.}
     \label{fig:perf_comp}
\end{figure*}

\textbf{Interpolation Points:} Following \cite{jahani2022berrut}, to ensure numerical stability, we select Chebyshev points for both the encoder design points $\{\alpha_k\}_{k=1}^K$ and the worker evaluation nodes $\{\beta_n\}_{n=1}^N$. Specifically, we use Chebyshev nodes of the first and second kind: $\alpha_k = -\cos\left(\frac{(2k-1)\pi}{2K}\right)$ and $\beta_n = -\cos\left(\frac{(n-1)\pi}{N-1}\right)$. Thus, the spacing parameters of the decoder design points are $\gamma_1=1$ and $\gamma_2=1$ (see Example~\ref{ex:spacing_points}). This choice guarantees optimal conditioning and provides a fair, identical setup for comparison with the \texttt{BACC} framework~\cite{jahani2022berrut}.

\textbf{Evaluation Metrics:} We quantify performance using two complementary metrics:
\begin{itemize}
    \item \textit{Mean Squared Error (MSE):} Serving as our primary theoretical loss defined in \eqref{eq:gen_obj}, MSE measures the empirical average of the $L_2$ distance between the true function output and our approximation over the input data distribution $\mathcal{X}$ and the non-straggler set $\stset$, i.e., $\mathbb{E}_{\mathbf{x}\sim \mathcal{X}, \stset}\left[\frac{1}{K}\sum^K_{k=1} \| f(\mathbf{x}_k) - \fhat(\mathbf{x}_k)\|^2_2\right]$. This provides an unbiased estimate of the expected loss defined in \eqref{eq:gen_obj}.
    \item \textit{Relative Accuracy (RelAcc):} To translate MSE into task-specific performance for classification models, RelAcc measures the ratio of the approximated model's prediction accuracy to the exact base model's prediction accuracy on the original dataset.
\end{itemize}

\subsection{Performance Evaluation}

\textbf{1. \gcc{} vs. \texttt{BACC} on DNNs across Straggler Regimes:}\\
We first evaluate the performance of \gcc{} for varying numbers of stragglers $S$ in the worst-case setting, using LeNet-5 and ViT-B/16 as the target functions. For each input batch, we independently generate 20 straggler realizations by selecting $S$ failed workers and average the resulting performance metrics. We repeat this procedure over 20 independently sampled input batches and report the overall averages together with the corresponding 95\% confidence intervals. Figure~\ref{fig:perf_comp} presents the resulting performance comparison.

We consider both \gcc{} with tuned smoothing parameters and natual spline interpolating variant obtained by $\enclamb=0$ and $\declamb=0$. Both variants consistently achieve lower MSE and higher RelAcc than \texttt{BACC} across nearly all straggler configurations. As shown in Figures~\ref{fig:mse_vit} and~\ref{fig:mse_lenet}, the two variants perform similarly when the number of stragglers is relatively small. In high-straggler regimes, however, the tuned smoothing parameters provide a clear advantage, and \gcc{} outperforms its interpolating variant.

\textbf{2. \gcc{} versus \texttt{LCC} for Polynomial Computation:}\\
Although \texttt{LCC} guarantees exact recovery for polynomial computations, it requires the number of non-straggling workers to meet a strict recovery threshold~\cite{yu2019lagrange}. For a polynomial target function $f$, this threshold is
\(
    \deg(f)(K-1)+1.
\)
If the number of available workers falls below this threshold, exact recovery is no longer guaranteed.

Figure~\ref{fig:perf_comp_lcc} compares \texttt{LCC} and \gcc{} on a polynomial computation task. As shown in Figure~\ref{fig:lag_comp_below_thresh}, when the number of non-straggling workers meets or exceeds the recovery threshold,
\(
    \deg(f)(K-1)+1
    =
    3\times4+1
    =
    13,
\)
or equivalently, when $S\le7$, \texttt{LCC} achieves zero reconstruction error, while \gcc{} incurs only a small approximation error of order $\mathcal{O}(10^{-3})$. However, when the number of stragglers increases and the number of available workers falls below the recovery threshold, as shown in Figure~\ref{fig:lag_comp_above_thresh}, \gcc{} continues to provide accurate approximations. In contrast, the reconstruction error of \texttt{LCC} increases sharply because its exact-recovery condition is no longer satisfied.
\begin{figure*}[t]
     \centering
     \begin{subfigure}[b]{0.49\textwidth}
         \centering
         \includegraphics[width=1\textwidth]{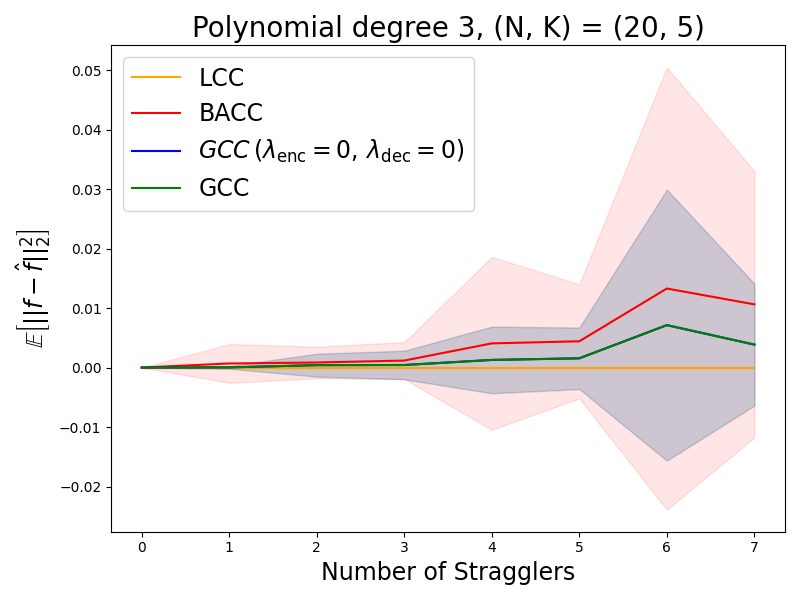}
         \caption{Performance within the \texttt{LCC} recovery threshold}
         \label{fig:lag_comp_below_thresh}
     \end{subfigure}
     \hfill
     \begin{subfigure}[b]{0.49\textwidth}
         \centering
         \includegraphics[width=1\textwidth]{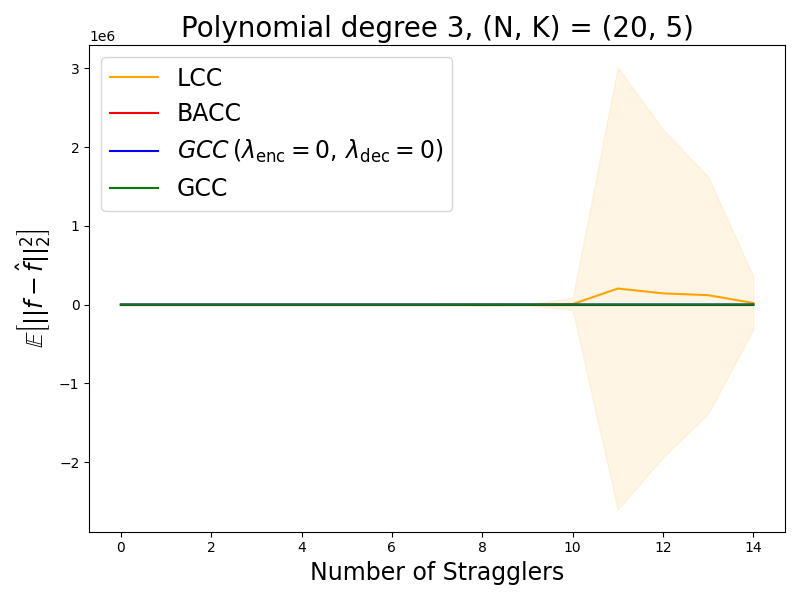}
         \caption{Performance across all straggler regimes}
         \label{fig:lag_comp_above_thresh}
     \end{subfigure}
     \caption{Comparison of Mean Squared Error (MSE) between \gcc{} and \texttt{LCC} for a polynomial task. (a) A zoomed-in view showing \texttt{LCC}'s exact recovery prior to exceeding the recovery threshold. (b) The full evaluation showing \texttt{LCC}'s instability beyond the threshold, contrasted with \gcc{}'s robustness.}
     \label{fig:perf_comp_lcc}
\end{figure*}

\textbf{3. Asymptotic Convergence Rate (\texttt{BACC} versus \gcc{}):}\\
To empirically evaluate the convergence guarantees established in Theorem~\ref{th:ldec_worst_asymp} and Corollary~\ref{cor:worst_convergence}, we fix the batch size $K$ and the number of stragglers $S$ while increasing the total number of workers $N$. Figure~\ref{fig:conv_comp_bacc} reports the resulting MSE for \gcc{} and \texttt{BACC} on a logarithmic scale. The empirical decay observed for \gcc{} is consistent with the theoretical upper bound
\(
    \mathcal{O}\!\left(N^{-3\gamma_2}\right),
\)
which reduces to $\mathcal{O}(N^{-3})$ for the Chebyshev design-point configuration used in the experiments, where $\gamma_2=1$. Moreover, the results show that the error of \gcc{} decreases more rapidly with $N$ than that of \texttt{BACC}.

\textbf{4. Convergence Rates under Worst-Case and Probabilistic Straggler Models:}\\
Finally, we examine how the straggler model affects the convergence behavior of \gcc{}. Figure~\ref{fig:conv_comp_prob} compares the worst-case setting with a fixed number of stragglers, $S=5$, against probabilistic settings with different worker-failure probabilities $p$. The results are consistent with Theorem~\ref{th:prob_dec_bound} and Corollary~\ref{cor:prob_convergence}. Although the MSE increases with $p$, reflecting the corresponding increase in the expected number of stragglers, the error continues to decrease as $N$ grows. In particular, the observed convergence behavior in the probabilistic settings is consistent with the theoretical upper bound
\(
    \mathcal{O}\!\left(
        \log_{1/p}^3(N)N^{-3}
    \right)
\)
established in Corollary~\ref{cor:prob_convergence}.

\begin{figure*}[t]
     \centering
     \begin{subfigure}[b]{0.49\textwidth}
         \centering
         \includegraphics[width=1\textwidth]{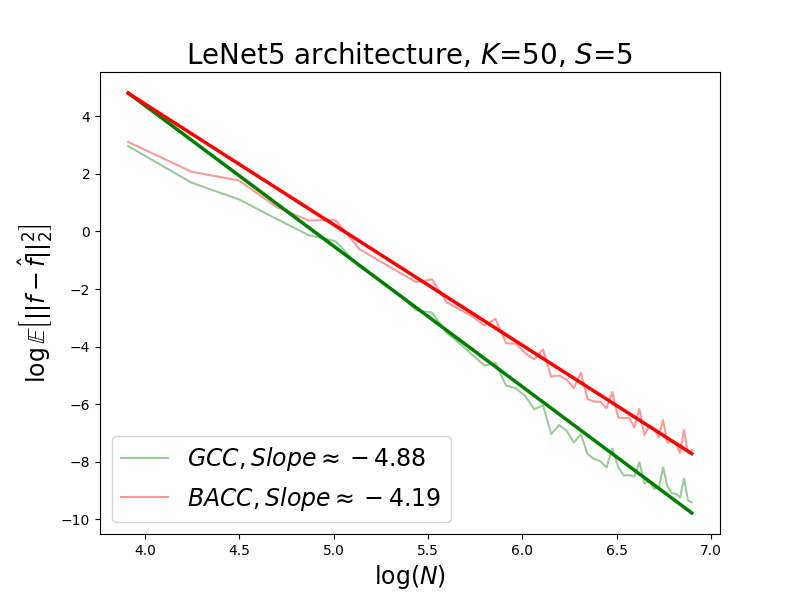}
         \caption{Convergence rate under worst-case stragglers}
         \label{fig:conv_comp_bacc}
     \end{subfigure}
     \hfill
     \begin{subfigure}[b]{0.49\textwidth}
         \centering
         \includegraphics[width=1\textwidth]{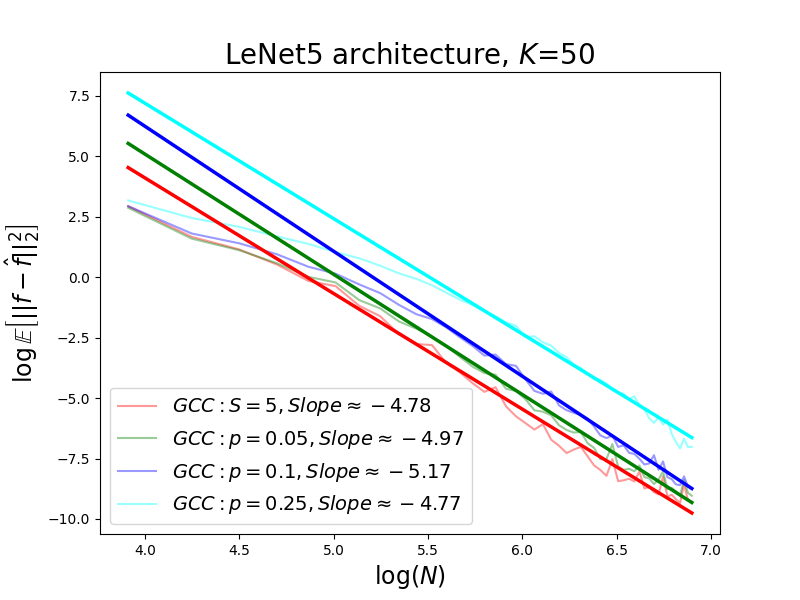}
         \caption{Convergence rate under probabilistic straggler settings}
         \label{fig:conv_comp_prob}
     \end{subfigure}
     \caption{Empirical validation of the asymptotic convergence rates as the number of worker nodes $N$ increases. (a) Comparison of the MSE decay between \texttt{BACC} and \gcc{} in a worst-case straggler environment. (b) The MSE decay for \gcc{} under various probabilistic worker failure rates alongside a worst-case straggler setting, validating the theoretical upper bound of $\mathcal{O}(N^{-3} \log_{1/p} N)$.}
     \label{fig:conv_comp}
\end{figure*}

\section{Conclusion and Future Directions}\label{sec:conclusion}

In this paper, we introduced General Coded Computing (\gcc{}), a unified framework that leverages learning theory to overcome the rigid algebraic limitations of classical coded computing. By reformulating distributed computation as an end-to-end loss minimization problem within second-order Sobolev spaces, \gcc{} enables the resilient, scalable execution of complex, non-linear workloads, such as deep neural networks. Through rigorous theoretical analysis and extensive empirical validation, we demonstrated that \gcc{} guarantees  function recovery and convergence under both worst-case and probabilistic straggler regimes. 

While this work establishes a new mathematical foundation for coded computing, it also opens several promising avenues for future exploration:

\begin{itemize}
    \item \textbf{Optimization of Design Points:} In our current formulations, the encoder and decoder design points, denoted by $\{\alpha_k\}_{k=1}^K$ and $\{\beta_n\}_{n=1}^N$, are selected \emph{a priori} using standard arrangements (e.g., equidistant or Chebyshev nodes). A natural extension is to develop algorithms that jointly optimize these points based on the specific target function or underlying data distribution.
    
    \item \textbf{Optimality of Convergence Bounds:} Although we established strong upper bounds on the approximation error and convergence rates for various straggler scenarios, an important theoretical direction is to tighten these bounds or rigorously prove their optimality.
    
    \item \textbf{Alternative Smoothness Assumptions:} The encoder and decoder designs in \gcc{} fundamentally rely on the assumption that both the encoding and decoding functions belong to a second-order Sobolev space. Expanding the framework to accommodate alternative mathematical notions of smoothness, such as Lipschitz continuity or broader Hölder classes of functions, would further generalize this theoretical foundation to an even wider class of functions.
\end{itemize}

\bibliographystyle{ieeetr}

\section{Biographies}
\begin{IEEEbiographynophoto}{Parsa Moradi}
is a Ph.D. candidate in the Department of Electrical and Computer Engineering at the University of Minnesota Twin Cities. He received his B.Sc. and M.Sc. degrees from the Department of Electrical Engineering, Sharif University of Technology, Tehran, Iran, in 2016 and 2019, respectively. His research interests include deep learning, generative AI,  distributed machine learning, and AI for healthcare.
\end{IEEEbiographynophoto}

\begin{IEEEbiographynophoto}{Behrooz Tahmasebi}
 is a Postdoctoral Fellow in Applied Mathematics and Computer Science at Harvard University, working in the Geometric Machine Learning Laboratory. He received his Ph.D. in Electrical Engineering and Computer Science (EECS) from the Massachusetts Institute of Technology (MIT), where he conducted research at the Computer Science and Artificial Intelligence Laboratory (CSAIL) under the supervision of Prof. Stefanie Jegelka. His research focuses on geometric deep learning, symmetries, and invariances in machine learning. He received a Best Paper Award at the HiLD Workshop at ICML 2024 and co-presented a NeurIPS 2025 tutorial on geometric machine learning. 
\end{IEEEbiographynophoto}

\begin{IEEEbiographynophoto}{Mohammad Ali Maddah-Ali}
(IEEE Fellow, 2023) is an Associate Professor at the University of Minnesota Twin Cities. He received his B.Sc. degree in Electrical Engineering from Isfahan University of Technology, his M.A.Sc. degree from the University of Tehran, and his Ph.D. in Electrical and Computer Engineering from the University of Waterloo, Canada, in 2007.

From 2007 to 2008, he was with the Wireless Technology Laboratories at Nortel Networks, Ottawa, ON, Canada. He then held a Postdoctoral Fellowship at the Department of Electrical Engineering and Computer Sciences, University of California, Berkeley, from 2008 to 2010. From September 2010 to September 2020, he served as a Communication Research Scientist at Nokia Bell Labs, NJ, USA. 

Dr. Maddah-Ali is the recipient of several honors, including the NSERC Postdoctoral Fellowship (2007), the Best Paper Award at the IEEE International Conference on Communications (ICC) in 2014, the IEEE Communications Society and IEEE Information Theory Society Joint Paper Award in 2015, and the IEEE Information Theory Society Paper Award in 2016. He served as an Associate Editor for the IEEE Transactions on Information Theory (2019–2022) and as Lead Editor for the IEEE Journal on Selected Areas in Information Theory. He is currently a distinguished lecturer of the IEEE Information Theory Society.
\end{IEEEbiographynophoto}

\appendices
\section{Definitions}\label{app:defs}

\begin{definition}[Weak derivative]\label{def:weak}
    Let $g$ and $h$ be functions on $\Omega=(a,b)$ for which the following integrals are well-defined. We say that $h$ is the weak derivative of $g$ if
\[
    \int^b_a g(t)\varphi'(t)\,dt
    =
    -
    \int_a^b h(t)\varphi(t)\,dt
\]
for every smooth test function $\varphi$ that $\varphi(a)=\varphi(b)=0$. In this case, we write $g'=h$. 

Higher-order weak derivatives are defined recursively, and for vector-valued functions they are defined component-wise. If $g$ is classically differentiable, then its weak derivative coincides with its classical derivative. Therefore, the condition $g\in H^2(\Omega;\mathbb{R}^M)$ means that $g$, $g'$, and $g''$ all belong to $L^2(\Omega;\mathbb{R}^M)$, where derivatives are understood in the weak sense.
\end{definition}

\begin{definition}[$L^p$ spaces and norms]
Let $\Omega\subset\mathbb{R}$ be an open interval and let $M\in\mathbb{N}$. For $1\le p<\infty$, we denote by
$\lpm{p}$ the space of measurable functions $g:\Omega\to\mathbb{R}^M$ such that
\[
\int_{\Omega} |g_j(t)|^p\,dt < \infty,\qquad \forall j\in[M],
\]
where $g(t)=[g_1(t),\dots,g_M(t)]^T$. The space $\lpm{p}$ is endowed with the norm
\[
\|g\|_{\lpm{p}} := \left(\sum^M_{j=1}\int_{\Omega}|g_i(t)|^p \,dt\right)^{\frac{1}{p}}.
\]
For $p=\infty$, we denote by $\lpm{\infty}$the space of measurable functions
$g:\Omega\to\mathbb{R}^M$ such that $\operatorname*{ess\,sup}_{t\in\Omega}|g_j(t)|<\infty$ for all $j\in[M]$,
equipped with the norm
\[
\|g\|_{\lpm{\infty}} := \max_{j \in [M]} \sup_{t\in \Omega} |g_j(t)|.
\]
\end{definition}

\begin{definition}[Local $L^p$ spaces]
A function $g:\Omega\to\mathbb{R}^M$ belongs to $L^p_{\mathrm{loc}}(\Omega;\mathbb{R}^M)$ if
$g\in L^p(V;\mathbb{R}^M)$ for every compact set $V\subset \Omega$.
\end{definition}

\begin{definition}[Local Sobolev spaces]
The local Sobolev space $W^{m,p}_{\mathrm{loc}}(\Omega;\mathbb{R}^M)$ is defined analogously by requiring that
$g\in L^p_{\mathrm{loc}}(\Omega;\mathbb{R}^M)$ and $g^{(i)}\in L^p_{\mathrm{loc}}(\Omega;\mathbb{R}^M)$
for all $i\in[m]$.
\end{definition}


\begin{definition}[Sobolev space with compact support]\label{def:sobz}
Let $\Omega=(a,b)$ and $m\in\mathbb{N}$. We define $W^{m,p}_0(\Omega;\mathbb{R}^M)$ as the set of functions
$g\in \smp$ satisfying the homogeneous boundary conditions
\[
g(a)=\mathbf{0},\ g'(a)=\mathbf{0},\ \dots,\ g^{(m-1)}(a)=\mathbf{0}.
\]
On this subspace, it is often convenient to use the seminorm induced by the highest-order derivative:
\[
\|g\|_{W^{m,p}_0(\Omega;\mathbb{R}^M)}
:=\|g^{(m)}\|_{\lpm{p}},\qquad 1\le p<\infty,
\]
and, for $p=\infty$,
\[
\|g\|_{W^{m,\infty}_0(\Omega;\mathbb{R}^M)}
:=\|g^{(m)}\|_{\lpm{\infty}}.
\]
\end{definition}

\section{Sobolev Spaces}\label{app:sobolev_props}

We first recall a standard Gagliardo-Nirenberg interpolation inequality for one-dimensional Sobolev spaces, which bounds the $L^r$ norm of a function using its $L^q$ norm and the $L^p$ norm of its derivative.

\begin{theorem}[Theorem 7.34, \cite{leoni2024first}] \label{th:interp_ineq} 
Let $\Omega \subseteq \mathbb{R}$ be an open interval and let $g \in \soblocm{1}{1}$. Assume $1 \leqslant p,q,r \leqslant \infty$ and $r \geqslant q$. Then:
\begin{align}
    \norm{g}_{\lpm{r}} \leqslant \ell^{\frac{1}{r}-\frac{1}{q}}\norm{g}_{\lpm{q}} + \ell^{1 - \frac{1}{p}+\frac{1}{r}}\norm{g'}_{\lpm{p}},
\end{align}
for every $0 < \ell < \mathcal{L}^1(\Omega)$, where $\mathcal{L}^1(\Omega)$ denotes the Lebesgue measure (length) of the interval $\Omega$.
\end{theorem}

By specializing this theorem to the $L^2$ and $L^\infty$ norms, we obtain the following useful corollary for bounding the supremum norm.

\begin{corollary}\label{col:interp_ineq}
Suppose $g \in \soblocm{1}{1}$ and $\Omega \subseteq \mathbb{R}$ is an open interval. If $\frac{\norm{g}_{\lpm{2}}}{\norm{g'}_{\lpm{2}}} < \mathcal{L}^1(\Omega)$, then:
\begin{align}
    \norm{g}_{\lpm{\infty}} \leqslant 2\sqrt{\norm{g}_{\lpm{2}}\cdot\norm{g'}_{\lpm{2}}}.
\end{align}
\end{corollary}
\begin{proof}
Substituting $p=2$, $q=2$, and $r=\infty$ into Theorem~\ref{th:interp_ineq} yields:
\begin{align}
    \norm{g}_{\lpm{\infty}} \leqslant \ell^{-\frac{1}{2}}\norm{g}_{\lpm{2}} + \ell^{\frac{1}{2}}\norm{g'}_{\lpm{2}}.
\end{align}
Minimizing the right-hand side with respect to $\ell$, we find that the optimal value, denoted by $\ell^*$, is:
\begin{align}
    \ell^* = \frac{\norm{g}_{\lpm{2}}}{\norm{g'}_{\lpm{2}}}.
\end{align}
Because we assume $\frac{\norm{g}_{\lpm{2}}}{\norm{g'}_{\lpm{2}}} < \mathcal{L}^1(\Omega)$, this optimal value $\ell^*$ strictly falls within the valid interval $(0, \mathcal{L}^1(\Omega))$ required by Theorem~\ref{th:interp_ineq}. Substituting $\ell^*$ back into the inequality completes the proof.
\end{proof}

\begin{corollary}[Corollary 7.36, \cite{leoni2024first}] \label{cor:interp_ineq_2}
Let $\Omega=(a, b)$, and let $1 \leq p, q, r \leq \infty$ such that $1+1 / r \geq 1 / p$ and $r \geq q$. Suppose $g \in \soblocm{1}{1}$ with $g^{\prime} \in \lpm{p}$. Let $x_0 \in[a, b]$ be a point such that $|g(x_0)|=\min_{x \in [a, b]}|g(x)|$. Then:
\begin{align}
    \norm{g-g(x_0)}_{\lpm{r}} \leq 8\norm{g}_{\lpm{q}}^\alpha \norm{g^{\prime}}_{\lpm{p}}^{1-\alpha},
\end{align}
where $\alpha:=0$ if $r=q$ and $1-1 / p+1 / r=0$; otherwise,
\begin{align}
    \alpha:=\frac{1-1 / p+1 / r}{1-1 / p+1 / q}.
\end{align}
\end{corollary}

While the preceding results are stated for $\soblocm{1}{1}$, the encoded functions in our proposed scheme reside in $\sobm{2}{2}$. The following theorem formalizes that these interpolation inequalities natively apply to our setting.

\begin{theorem}\label{th:interp_ineq2}
$\sobm{2}{2} \subseteq \sobm{1}{1} \subseteq \soblocm{1}{1}$.
\end{theorem}
\begin{proof}

Let $g = [g_1, \dots, g_m]^T \in \sobm{2}{2}$. Applying the Cauchy-Schwarz inequality to the $L^1$ norm of $g$, we have:
\begin{align}\label{eq:col:interp_ineq2_1}
    \norm{g}_{\lpm{1}} &= \sum_{j=1}^m \int_{\Omega} |g_j(t)|\,dt \nonumber \\
    &\leqslant \sum_{j=1}^m \left(\int_{\Omega} 1^2\,dt \cdot \int_{\Omega} |g_j(t)|^2\,dt\right)^{\frac{1}{2}} \nonumber \\ 
    &\leqslant \left(\mathcal{L}^1(\Omega)\right)^\frac{1}{2} \sum_{j=1}^m \left(\int_{\Omega} |g_j(t)|^2\,dt\right)^{\frac{1}{2}} \lec{(a)}{<} \infty,
\end{align}
where (a) follows because $\Omega$ has a finite length and $g \in \sobm{2}{2}$. Similarly, bounding the $L^1$ norm of the first derivative yields:
\begin{align}\label{eq:col:interp_ineq2_2}
    \norm{g'}_{\lpm{1}} &= \sum_{j=1}^m \int_{\Omega} |g_j'(t)|\,dt \nonumber \\
    &\leqslant \sum_{j=1}^m \left(\int_{\Omega} 1^2\,dt \cdot \int_{\Omega} |g_j'(t)|^2\,dt\right)^{\frac{1}{2}} \nonumber \\ 
    &\leqslant \left(\mathcal{L}^1(\Omega)\right)^\frac{1}{2} \sum_{j=1}^m \left(\int_{\Omega} |g_j'(t)|^2\,dt\right)^{\frac{1}{2}} \lec{(a)}{<} \infty.
\end{align}
Equations \eqref{eq:col:interp_ineq2_1} and \eqref{eq:col:interp_ineq2_2} confirm that $g$ and its derivative are absolutely integrable, establishing that $\sobm{2}{2} \subseteq \sobm{1}{1}$. 

Finally, to show that $g \in \soblocm{1}{1}$, consider any closed subset $[c,d] \subset \Omega$. We observe that:
\begin{align}
    \sum_{j=1}^m \left(\int^d_c |g_j(t)|\,dt\right) &\lec{(a)}{\leqslant} \sum_{j=1}^m \left(\int_\Omega |g_j(t)|\,dt\right) \lec{(b)}{<} \infty, \\
    \sum_{j=1}^m \left(\int^d_c |g_j'(t)|\,dt\right) &\lec{(a)}{\leqslant} \sum_{j=1}^m \left(\int_\Omega |g_j'(t)|\,dt\right) \lec{(b)}{<} \infty,
\end{align}
where (a) is due to $[c,d] \subset \Omega$ and the non-negativity of the integrand, and (b) follows directly from our finding that $g \in \sobm{1}{1}$. Therefore, $g$ is locally absolutely integrable, meaning $g \in \soblocm{1}{1}$, which completes the proof.
\end{proof}


\subsection{Sobolev Equivalent Norms} \label{app:sob_equi_norm}
Various norms have been defined on Sobolev spaces in the literature that are equivalent to the standard norm $\norm{\cdot}_{\sobm{p}{m}}$ (see \cite{adams2003sobolev}, \cite[Ch. 7]{berlinet2011reproducing}, and \cite[Sec. 10.2]{wahba1990spline}). Note that two norms $\norm{\cdot}_{W_1}$ and $\norm{\cdot}_{W_2}$ are considered equivalent if there exist positive constants $\eta_1, \eta_2$ such that:
$$
\eta_1\cdot\norm{g}_{W_2} \leqslant \norm{g}_{W_1} \leqslant \eta_2\cdot\norm{g}_{W_2}.
$$
The specific equivalent norm we utilize is the one introduced in \cite{kimeldorf1971some}. Let $\Omega = (a,b) \subset \mathbb{R}$. Let $\smpeq$ as the Sobolev space endowed with the following norm:
\begin{align}\label{eq:sob_norm_eq}
    \norm{g}_{\smpeq} := \left(\sum^M_{j=1} \left(|g_j(a)|^p +\sum^{m-1}_{i=1} \left|g_j^{(i)}(a)\right|^p \right) + \norm{g^{(m)}}_{\lpm{p}}^p \right)^\frac{1}{p}.
\end{align}

\begin{proposition}\label{prop:sob_hilb_comp} 
(See \cite[Section 7.2]{leoni2024first}, \cite[Theorem 121]{berlinet2011reproducing}, and \cite{wahba1990spline}). For any open interval $\Omega \subseteq \mathbb{R}$ and $m, M \in \mathbb{N}$, the spaces endowed with the following norm:
\begin{align}
    \hilmtilde{m} &:= \sobmeq{m}{2},
\end{align}
are Reproducing Kernel Hilbert Spaces (RKHSs).
\end{proposition}

\section{Smoothing Spline Basics}
\label{app:smoothspline}

Consider the data model $y_i = f(t_i) + \epsilon_i$ for $i=1,\dots,n$, where the knots $t_i \in \Omega=(a,b) \subset \mathbb{R}$, and the noise terms satisfy $\mathbb{E}[\epsilon_i] = 0$ and $\mathbb{E}[\epsilon_i^2] \leqslant \sigma_0^2$. 
Assuming the true function $f \in \sobeq{m}{2}$ and $\mathcal{T}:=\{t_i\}^n_{i=1}$, the smoothing spline estimator is defined as the solution to the following regularized optimization problem:
\begin{align}\label{eq:sm_spline_obj}
    \splineCC{\lambda}{\mathcal{T}}{1}(\mathbf{y}) := \underset{g \in \sobeq{m}{2}}{\operatorname{argmin}} \frac{1}{n} \sum^n_{i=1}\left(g(t_i) - y_i\right)^2 + \lambda \int_{\Omega} \left(g^{(m)}(t)\right)^2\,dt,
\end{align}
where $\mathbf{y} = [y_1,\dots, y_n]^T$ is the vector of observations and $\lambda > 0$ is the smoothing parameter. Based on Proposition~\ref{prop:sob_hilb_comp}, the space $\hiltilde{m} := \sobeq{m}{2}$ equipped with the norm $\|\cdot\|_{\sobeq{m}{2}}$ is a RKHS associated with a reproducing kernel $\phi(\cdot,\cdot)$. Therefore, by the reproducing property, for any $v \in \sobeq{m}{2}$, we have:
\begin{align}\label{eq:smothspline_kernel}
    v(t) = \langle v(\cdot), \phi(\cdot, t)\rangle_{\hiltilde{m}}.
\end{align}
It is shown that the kernel decomposes as $\phi(t,s) = R^P(t, s) + R^0(t,s)$, where $R^0(t, s)$ is the reproducing kernel for the subspace $\hilz{m}$, and $R^P(t,s)$ corresponds to the null space of the penalty functional, which spans all polynomials of degree strictly less than $m$ \cite{wahba1975smoothing,wahba1990spline}.

By the Representer Theorem \cite{scholkopf2001generalized}, the unique minimizer of \eqref{eq:sm_spline_obj} admits the finite-dimensional representation:
\begin{align}
    u^*(\cdot) = \sum^{m}_{i=1} d_i \zeta_i(\cdot) + \sum^n_{j=1} c_j \nu_j(\cdot),
\end{align}
where $\nu_j(\cdot) = R^0(\cdot, t_j)$ for $j \in [n]$, and $\{\zeta_i(\cdot)\}_{i=1}^m$ forms a basis for the space of polynomials of degree at most $m-1$. Substituting this representation back into \eqref{eq:sm_spline_obj} and optimizing over the coefficient vectors $\mathbf{c} = [c_1,\dots, c_n]^T$ and $\mathbf{d} = [d_1,\dots, d_m]^T$ yields the closed-form solution for the fitted values at the design points \cite{wahba1990spline}:
\begin{align}\label{eq:spline_linear}
    \splineCC{\lambda}{\mathcal{T}}{1}(\mathbf{y})|_{\mathbf{t}} = \mathbf{Q}\left(\mathbf{Q}^T \mathbf{Q} + \lambda \mathbf{\Gamma}\right)^{-1} \mathbf{Q}^T \mathbf{y},
\end{align}
where the matrices are defined as follows:
\begin{align}
\mathbf{Q}_{n \times(n+m)} &= \left[\begin{array}{ll}
\mathbf{T}_{n \times m} & \boldsymbol{\Sigma}_{n \times n}
\end{array}\right], \nonumber \\
\bm{\Gamma}_{(n+m) \times(n+m)} &= \left[\begin{array}{cc}
\mathbf{0}_{m \times m} & \mathbf{0}_{m \times n} \\
\mathbf{0}_{n \times m} & \boldsymbol{\Sigma}_{n \times n}
\end{array}\right], \nonumber \\
\mathbf{T}_{ij} &= \zeta_j(t_i), \nonumber \\
\boldsymbol{\Sigma}_{ij} &= R^0(t_i, t_j).
\end{align}
Equation \eqref{eq:spline_linear} demonstrates that the smoothing spline operator mapping the data $\mathbf{y}$ to the fitted values is strictly linear. More generally, the estimated function itself can be expressed via a linear operator $\mathbf{A}_\lambda$ applied to any input vector $\mathbf{z} \in \mathbb{R}^n$:
\begin{align}\label{eq:ssfunction_def}
    \splineCC{\lambda}{\mathcal{T}}{1}(\mathbf{z}) := \mathbf{A_{\lambda}} \mathbf{z}.
\end{align}
Furthermore, $u^*(\cdot)$ is known to be a natural spline \cite{wahba1975smoothing, wahba1990spline}. Consequently, if $\{b_i(\cdot)\}_{i=1}^n$ denotes a basis for the $m$-th order natural splines (such as truncated power or B-spline basis functions) corresponding to the knots $\{t_i\}_{i=1}^n$, the solution is equivalently written as:
\begin{align}\label{eq:spline_bspline_1}
     u^*(t) &= \sum^n_{i=1} \xi_i b_i(t), \\
     \boldsymbol{\xi} &= \left(\mathbf{N}^T\mathbf{N} + \lambda\Phi\right)^{-1}\mathbf{N}^T\mathbf{y},
\end{align}
where $\boldsymbol{\xi} := [\xi_1,\dots,\xi_n]^T$, $\mathbf{N}_{ij} = b_i(t_j)$, and $\Phi_{ij} = \int_\Omega b''_i(t)b''_j(t)\,dt$ for $i, j \in [n]$.

\section{Proof of Lemmas}\label{app:proof_lemmas}

\subsection{Proof of Lemma~\ref{lem:composition_sobolev}}\label{app:proof_lemma_omposition_sobolev}

All derivatives below are understood in the weak sense (Definition~\ref{def:weak}), and the chain-rule identities hold almost everywhere on $\Omega$. To prove that $f \circ \enc \in \hild{2}{m}$, we must show that the function and its first two derivatives have bounded $\lp{2}$ norms. For analytical convenience, let us define the shifted function $$f_0(\mathbf{x}) := f(\mathbf{x}) - f(\mathbf{0}).$$ Thus, $f_0(\mathbf{0}) = \mathbf{0}$. Since adding a constant vector does not affect the derivatives, demonstrating that $f_0 \circ \enc \in \hild{2}{m}$ is equivalent to showing $f \circ \enc \in \hild{2}{m}$. Let
\[
    q^2:=\sum_{j=1}^m q_j^2,
    \qquad
    \mu^2:=\sum_{j=1}^m \mu_j^2 .
\]
First, since $\|\nabla f_j(\mathbf{x})\|_2\le q_j$, each component $f_{0j}$ is $q_j$-Lipschitz. Therefore, for almost every $t\in\Omega$,
\begin{align}
    \|f_0(\enc(t))\|_2^2
    &=
    \sum_{j=1}^m
    |f_{0j}(\enc(t))-f_{0j}(\mathbf{0})|^2
    \nonumber\\
    &\le
    \sum_{j=1}^m
    q_j^2\|\enc(t)\|_2^2
    =
    q^2\|\enc(t)\|_2^2 .
    \label{eq:composition_l2_bound}
\end{align}
Integrating over $\Omega$ gives
\[
    \|f_0\circ\enc\|_{\lpd{2}{m}}^2
    \le
    q^2\|\enc\|_{\lpd{2}{d}}^2
    <
    \infty,
\]
because $\enc\in\hild{2}{d}$.

Next, for each $j\in[m]$, the weak chain rule gives
\[
    (f_{0j}\circ\enc)'(t)
    =
    \left\langle
        \nabla f_{0j}(\enc(t)),
        \enc'(t)
    \right\rangle
    \quad \text{a.e. } t\in\Omega .
\]
Thus,
\begin{align}
    \|(f_0\circ\enc)'\|_{\lpd{2}{m}}^2
    &=
    \sum_{j=1}^m
    \int_\Omega
    \left|
        \left\langle
            \nabla f_{0j}(\enc(t)),
            \enc'(t)
        \right\rangle
    \right|^2 dt
    \nonumber\\
    &\le
    \sum_{j=1}^m
    q_j^2
    \int_\Omega
    \|\enc'(t)\|_2^2 dt
    =
    q^2\|\enc'\|_{\lpd{2}{d}}^2
    <
    \infty .
    \label{eq:composition_first_derivative_bound}
\end{align}

It remains to control the second derivative. Using the expression for the second derivative, we obtain
\begin{align}
&\int_\Omega
\left|
    (f_{0j}\circ\enc)''(t)
\right|^2\, dt
\nonumber\\
&=
\int_\Omega
\left|
    \enc'(t)^T
    D^2 f_{0j}(\enc(t))
    \enc'(t)
    +
    \left\langle
        \nabla f_{0j}(\enc(t)),
        \enc''(t)
    \right\rangle
\right|^2 \,dt
\nonumber\\
&\overset{\mathrm{(a)}}{\le}
2
\int_\Omega
\left|
    \enc'(t)^T
    D^2 f_{0j}(\enc(t))
    \enc'(t)
\right|^2 \,dt
+
2
\int_\Omega
\left|
    \left\langle
        \nabla f_{0j}(\enc(t)),
        \enc''(t)
    \right\rangle
\right|^2 \,dt
\nonumber\\
&\overset{\mathrm{(b)}}{\le}
2
\int_\Omega
\|D^2 f_{0j}(\enc(t))\|_{\mathrm{op}}^2
\|\enc'(t)\|_2^4\, dt
+
2
\int_\Omega
\|\nabla f_{0j}(\enc(t))\|_2^2
\|\enc''(t)\|_2^2\, dt
\nonumber\\
&\overset{\mathrm{(c)}}{\le}
2\mu_j^2
\int_\Omega
\|\enc'(t)\|_2^4\, dt
+
2q_j^2
\int_\Omega
\|\enc''(t)\|_2^2\, dt .
\label{eq:composition_second_derivative_component}
\end{align}
Here, (a) follows from $(a+b)^2\le 2a^2+2b^2$, (b) uses operator norm inequality 
\(
    |\mathbf{v}^T\mathbf{A}\mathbf{v}|
    \le
    \|\mathbf{A}\|_{\mathrm{op}}\|\mathbf{v}\|_2^2
\)
and the Cauchy--Schwarz inequality for the inner-product term, and (c) follows from the uniform bounds
\(
    \|D^2 f_{0j}(\mathbf{y})\|_{\mathrm{op}}\le \mu_j
\)
and
\(
    \|\nabla f_{0j}(\mathbf{y})\|_2\le q_j .
\)
Summing over $j\in[m]$ yields
\begin{align}
    \|(f_0\circ\enc)''\|_{\lpd{2}{m}}^2
    &\le
    2\mu^2
    \int_\Omega
    \|\enc'(t)\|_2^4\, dt
    +
    2q^2
    \|\enc''\|_{\lpd{2}{d}}^2 .
    \label{eq:composition_second_derivative_bound}
\end{align}
Since $\enc\in\hild{2}{d}$ and $\Omega\subset\mathbb{R}$ is bounded, the one-dimensional Sobolev embedding~\cite{nocedal2006numerical} implies $\enc'\in L^\infty(\Omega;\mathbb{R}^d)$. Hence,
\[
    \int_\Omega \|\enc'(t)\|_2^4 dt
    \le
    \|\enc'\|_{\lpd{\infty}{d}}^2
    \|\enc'\|_{\lpd{2}{d}}^2
    <
    \infty .
\]
Moreover, $\|\enc''\|_{\lpd{2}{d}}<\infty$ by assumption. Therefore,
\[
    f_0\circ\enc,\quad
    (f_0\circ\enc)',\quad
    (f_0\circ\enc)''
    \in
    \lpd{2}{m}.
\]
Thus $f_0\circ\enc\in\hild{2}{m}$, and consequently $f\circ\enc\in\hild{2}{m}$.

\subsection{Proof of Lemma~\ref{lem:decoder_error_supnorm}}\label{app:proof_lem:decoder_error_supnorm}

We begin by proving the following lemma.
\begin{lemma}\label{lem:decoder_error_poincare}
Let $\Omega=(-1,1)$, and suppose that $|\mathcal{F}|\ge 1$. Then the decoder error function $h$ satisfies
\begin{align}\label{eq:decoder_error_poincare}
    \|h\|_{\lpd{2}{m}}
    \le
    \sqrt{2}\,
    \|h'\|_{\lpd{2}{m}} .
\end{align}
\end{lemma}

\begin{proof}
We prove the claim component-wise. Let $h_j$ be the $j$-th component of $h$. We first show that $h_j$ has at least one zero in $\Omega$.

If $\declamb=0$, the decoder is interpreted as the natural cubic interpolant, and hence
\[
    h_j(\beta_v)=0,
    \qquad v\in\mathcal{F}.
\]
Thus, the claim is immediate.

Now suppose $\declamb>0$. Assume, toward a contradiction, that $h_j$ has no zero on $\Omega$. Since $h_j\in H^2(\Omega)$, it is continuous; hence it must have a fixed sign on $\Omega$. Without loss of generality, suppose $h_j(t)>0$ for all $t\in\Omega$. Let
\[
    c:=\min_{v\in\mathcal{F}} h_j(\beta_v)>0,
\]
and define a shifted decoder component
\[
    \widetilde{\dec}_j(t):=\dec_j^*(t)-c .
\]
This shift does not change the second derivative, so $\|\widetilde{\dec_j}^{\prime\prime}\|^2_{\lpd{2}{m}}=\|(\dec_j^*)^{\prime\prime}\|^2_{\lpd{2}{m}}$. Moreover, for every $v\in\mathcal{F}$,
\[
    0\le h_j(\beta_v)-c < h_j(\beta_v),
\]
which implies
\[
    \sum_{v\in\mathcal{F}}
    \left(
        \widetilde{\dec}_j(\beta_v)-g_j(\beta_v)
    \right)^2
    <
    \sum_{v\in\mathcal{F}}
    \left(
        \dec_j^*(\beta_v)-g_j(\beta_v)
    \right)^2 .
\]
Thus, $\widetilde{\dec}_j$ achieves a strictly smaller objective value than $\dec_j^*$ for the $j$-th scalar smoothing-spline problem, contradicting the optimality of $\dec_j^*$. The case $h_j(t)<0$ for all $t$ is handled by shifting upward. Hence, there exists $x_{0,j}\in\Omega$ such that $h_j(x_{0,j})=0$.

For any $x\in\Omega$, the fundamental theorem of calculus gives
\[
    h_j(x)
    =
    \int_{x_{0,j}}^x h_j'(s)\,ds .
\]
By Cauchy--Schwarz,
\[
    |h_j(x)|^2
    \le
    |x-x_{0,j}|
    \int_{\Omega} |h_j'(s)|^2\,ds .
\]
Integrating over $x\in\Omega$ yields
\begin{align}
    \|h_j\|_{\lp{2}}^2
    &\le
    \|h_j'\|_{\lp{2}}^2
    \int_{\Omega} |x-x_{0,j}|\,dx
    \nonumber\\
    &=
    \|h_j'\|_{\lp{2}}^2
    \left(
        \frac{(x_{0,j}+1)^2}{2}
        +
        \frac{(1-x_{0,j})^2}{2}
    \right)
    \nonumber\\
    &=
    (x_{0,j}^2+1)
    \|h_j'\|_{\lp{2}}^2
    \le
    2\|h_j'\|_{\lp{2}}^2 .
\end{align}
Summing over $j\in[m]$ gives
\[
    \|h\|_{\lpd{2}{m}}^2
    \le
    2\|h'\|_{\lpd{2}{m}}^2,
\]
which proves the result.
\end{proof}

\begin{proof}
If $\|h'\|_{\lpd{2}{m}}=0$, then Lemma~\ref{lem:decoder_error_poincare} implies $\|h\|_{\lpd{2}{m}}=0$, and \eqref{eq:decoder_error_supnorm} is trivial. Otherwise, Lemma~\ref{lem:decoder_error_poincare} gives
\[
    \frac{\|h\|_{\lpd{2}{m}}}{\|h'\|_{\lpd{2}{m}}}
    \le
    \sqrt{2}
    <
    2
    =
    |\Omega|.
\]
Since $h \in \hild{2}{m}$, Theorem~\ref{th:interp_ineq2} ensures that $h \in \soblocmd{1}{1}{m}$. Applying the interpolation inequality in Corollary~\ref{col:interp_ineq} completes the proof.


\end{proof}

\subsection{Proof of Lemma~\ref{lem:lit_spline_noiseless}}\label{app:proof_lem:lit_spline_noiseless}

This is the specialization of the smoothing-spline error bound of ~\cite[Theorem~4.10]{ragozin1983error} to the second-order case on $\Omega=(-1,1)$. 

\subsection{Proof of Lemma~\ref{lem:composition_second_derivative_bound}}\label{app:proof_lem:composition_second_derivative_bound}

From the proof of Lemma~\ref{lem:composition_sobolev}, for $g=f\circ\enc$ we have
\begin{align}
    \|g''\|_{\lpd{2}{m}}^2
    &\le
    2\mu^2
    \int_\Omega
    \|\enc'(t)\|_2^4\,dt
    +
    2q^2
    \int_\Omega
    \|\enc''(t)\|_2^2\,dt
    \nonumber\\
    &\le
    2\max\{\mu^2,q^2\}
    \left(
        \|\enc'\|_{\lpd{4}{d}}^4
        +
        \|\enc''\|_{\lpd{2}{d}}^2
    \right).
    \label{eq:composition_cor_start}
\end{align}
By the Sobolev interpolation inequality provided in Theorem~\ref{th:interp_ineq} applied to $\enc'$ with $p=q=2$, $r=\infty$, and $\ell=1$ we obtain:
\[
    \|\enc'\|_{\lpd{4}{d}}
    \le
    \|\enc'\|_{\lpd{2}{d}}
    +
    \|\enc''\|_{\lpd{2}{d}} .
\]
Therefore,
\begin{align}
    \|g''\|_{\lpd{2}{m}}^2
    &\le
    C
    \max\{\mu^2,q^2\}
    \left[
        \|\enc''\|_{\lpd{2}{d}}^2
        +
        \left(
            \|\enc'\|_{\lpd{2}{d}}
            +
            \|\enc''\|_{\lpd{2}{d}}
        \right)^4
    \right]
    \nonumber\\
    &\lec{(a)}{\le}
    C
    \max\{\mu^2,q^2\}
    \left(
        \|\enc\|_{\hild{2}{d}}^2
        +
        4\|\enc\|_{\hild{2}{d}}^4
    \right)
    \nonumber\\
    &=
    C
    \max\{\mu^2,q^2\}
    \psi\!\left(
        \|\enc\|_{\hild{2}{d}}^2
    \right),
\end{align}
where (a) follows from the fact that $\max\{ \|\enc\|_{\lpd{2}{d}},  \|\enc'\|_{\lpd{2}{d}},  \|\enc''\|_{\lpd{2}{d}}\} \leq \|\enc\|_{\hild{2}{d}}^2$. Absorbing the numerical constant into $C_{\mathrm{comp}}$ proves the result.

\subsection{Proof of Lemma~\ref{lem:efficient_encoder_uniform_bound}}
\label{app:proof_uniform_encoder_bound}

Fix $\lambda>0$, and let $u_\lambda$ be a minimizer of
\[
    \operatorname*{arg\,min}_{u\in\hild{2}{d}}
    \left\{
        \frac{1}{K}
        \sum_{k=1}^K
        \|u(\alpha_k)-\mathbf{x}_k\|_2^2
        +
        \lambda
        \|u''\|_{\lpd{2}{d}}^2
    \right\}.
\]
We will show that $\|u_\lambda\|_{\hild{2}{d}}^2$ is bounded by a constant depending only on the encoder design points and the input data.

Let $\mathcal{N}_{\alpha}$ denote the space of natural cubic splines associated with the fixed encoder design points $\{\alpha_k\}_{k=1}^K$. For each $i\in[K]$, let $\ell_i\in\mathcal{N}_{\alpha}$ be the natural cubic spline satisfying
\[
    \ell_i(\alpha_k)=\mathbf{1}_{\{i=k\}},
    \qquad k\in[K].
\]
Thus, each $\ell_i$ is a natural cubic spline and belongs to $H^2(\Omega)$. Moreover, for any vector $\mathbf{y}=[y_1,\ldots,y_K]^T\in\mathbb{R}^K$, the natural cubic spline interpolating the values $\{y_k\}_{k=1}^K$ at the encoder design points $\{\alpha_k\}_{k=1}^K$ is
\(
    \sum_{i=1}^K y_i\ell_i(t).
\)

We now explain why $u_\lambda$ can be represented using this basis. For any function $u\in\hild{2}{d}$, replacing each component of $u$ by the natural cubic spline interpolant with the same values at $\{\alpha_k\}_{k=1}^K$ leaves the empirical fitting term unchanged. By the variational property of natural cubic splines, this replacement does not increase the second-derivative norm. Hence, a minimizer may be chosen in $(\mathcal{N}_{\alpha})^d$. Therefore, for each coordinate $j\in[d]$, we can write
\[
    (u_\lambda)_j(t)
    =
    \sum_{i=1}^K
    y_{\lambda,i}^{j}\ell_i(t),
\]
where
\[
    \mathbf{y}_{\lambda}^{j}
    :=
    \big[(u_\lambda)_j(\alpha_1),\ldots,(u_\lambda)_j(\alpha_K)\big]^T
\]
is the vector of values of the $j$-th coordinate of $u_\lambda$ at the encoder design points. Similarly, define
\[
    \mathbf{x}^{j}
    :=
    \big[x_1^{j},\ldots,x_K^{j}\big]^T .
\]

For the $j$-th coordinate, the objective as a function of $\mathbf{y}^{j}$ is
\[
    J_j(\mathbf{y}^{j})
    =
    \frac{1}{K}
    \left\|
        \mathbf{y}^{j}-\mathbf{x}^{j}
    \right\|_2^2
    +
    \lambda
    (\mathbf{y}^{j})^T
    \Omega
    \mathbf{y}^{j},
\]
where
\[
    \Omega_{i\ell}
    :=
    \int_\Omega
    \ell_i''(t)\ell_\ell''(t)\,dt .
\]
The matrix $\Omega$ is positive semidefinite, since for every $\mathbf{a}\in\mathbb{R}^K$,
\[
    \mathbf{a}^T\Omega\mathbf{a}
    =
    \int_\Omega
    \left(
        \sum_{i=1}^K a_i\ell_i''(t)
    \right)^2 dt
    \ge 0 .
\]

Since $J_j$ is a quadratic function of $\mathbf{y}^{j}$, its minimizer satisfies the first-order optimality condition
\[
    \nabla J_j(\mathbf{y}_{\lambda}^{j})=\mathbf{0}.
\]
Computing the gradient gives
\[
    \nabla J_j(\mathbf{y}^{j})
    =
    \frac{2}{K}
    \left(
        \mathbf{y}^{j}-\mathbf{x}^{j}
    \right)
    +
    2\lambda\Omega\mathbf{y}^{j}.
\]
Therefore,
\[
    \frac{2}{K}
    \left(
        \mathbf{y}_{\lambda}^{j}-\mathbf{x}^{j}
    \right)
    +
    2\lambda\Omega\mathbf{y}_{\lambda}^{j}
    =
    \mathbf{0}.
\]
By rearranging, we obtain
\begin{align}\label{eq:11}
    \left(
        I_K+\bar{\lambda}\Omega
    \right)
    \mathbf{y}_{\lambda}^{j}
    =
    \mathbf{x}^{j},
\end{align}
where $\bar{\lambda}:=K\lambda$. Since $\Omega$ is positive semi-definite, all eigenvalues of $I_K+\bar{\lambda}\Omega$ are at least one. Hence, for every vector $\mathbf{z}\in\mathbb{R}^K$,
\[
    \left\|
        (I_K+\bar{\lambda}\Omega)\mathbf{z}
    \right\|_2
    \ge
    \|\mathbf{z}\|_2.
\]
Applying this with $\mathbf{z}=\mathbf{y}_{\lambda}^{j}$ and using \eqref{eq:11}
\[
    \|\mathbf{y}_{\lambda}^{j}\|_2
    \le
    \|\mathbf{x}^{j}\|_2.
\]
Summing over $j\in[d]$ yields
\begin{align}\label{eq:value_uniform_bound}
    \sum_{j=1}^d
    \|\mathbf{y}_{\lambda}^{j}\|_2^2
    \le
    \sum_{j=1}^d
    \|\mathbf{x}^{j}\|_2^2
    =
    \sum_{k=1}^K
    \|\mathbf{x}_k\|_2^2 .
\end{align}

It remains to translate this bound on the values of $u_\lambda$ at the encoder design points into a Sobolev-norm bound. For each coordinate $j\in[d]$, using the representation above and the triangle inequality followed by Cauchy--Schwarz, we have
\begin{align}
    \|(u_\lambda)_j\|_{H^2(\Omega)}
    &=
    \left\|
        \sum_{i=1}^K
        y_{\lambda,i}^{j}\ell_i
    \right\|_{H^2(\Omega)}
    \nonumber\\
    &\le
    \sum_{i=1}^K
    |y_{\lambda,i}^{j}|
    \|\ell_i\|_{H^2(\Omega)}
    \nonumber\\
    &\le
    \|\mathbf{y}_{\lambda}^{j}\|_2
    \left(
        \sum_{i=1}^K
        \|\ell_i\|_{H^2(\Omega)}^2
    \right)^{1/2}.
\end{align}
Squaring both sides gives
\[
    \|(u_\lambda)_j\|_{H^2(\Omega)}^2
    \le
    \|\mathbf{y}_{\lambda}^{j}\|_2^2
    \sum_{i=1}^K
    \|\ell_i\|_{H^2(\Omega)}^2 .
\]
Summing over $j\in[d]$, we obtain
\begin{align}
    \|u_\lambda\|_{\hild{2}{d}}^2
    &=
    \sum_{j=1}^d
    \|(u_\lambda)_j\|_{H^2(\Omega)}^2
    \nonumber\\
    &\le
    \left(
        \sum_{i=1}^K
        \|\ell_i\|_{H^2(\Omega)}^2
    \right)
    \left(
        \sum_{j=1}^d
        \|\mathbf{y}_{\lambda}^{j}\|_2^2
    \right)
    \nonumber\\
    &\le
    \left(
        \sum_{i=1}^K
        \|\ell_i\|_{H^2(\Omega)}^2
    \right)
    \left(
        \sum_{k=1}^K
        \|\mathbf{x}_k\|_2^2
    \right),
\end{align}
where the last inequality follows from \eqref{eq:value_uniform_bound}.

The first factor depends only on the encoder design points $\{\alpha_k\}_{k=1}^K$, while the second depends only on the input data $\{\mathbf{x}_k\}_{k=1}^K$. Therefore, setting
\[
    R_{\mathrm e}
    :=
    \left(
        \sum_{i=1}^K
        \|\ell_i\|_{H^2(\Omega)}^2
    \right)
    \left(
        \sum_{k=1}^K
        \|\mathbf{x}_k\|_2^2
    \right)
\]
gives
\[
    \|u_\lambda\|_{\hild{2}{d}}^2
    \le
    R_{\mathrm e}
\]
for every $\lambda>0$. This completes the proof.

\subsection{Proof of Lemma~\ref{lem:longest_run_moment}}\label{app:proof_longest_run_2}

To simplify notation throughout this proof, let $X := R_{\stset}$ denote the non-negative integer random variable representing the longest run of ones, and let its expectation be $\mu := \mathbb{E}[R_{\stset}]$. 

\begin{lemma}(\cite[Theorem 2]{gordon1986extreme}) \label{lem:longest_run}
Let $Z_{p, N}$ be the maximal sub-sequence length of consecutive ones (the longest run) in an i.i.d. Bernoulli sequence $B_1,\dots, B_N$ with $\mathbb{E}[B_i] = p$ for $i \in [N]$. Let $\theta=\frac{\pi^2}{\ln (1/p)}$. Then:
\begin{align}
    \left|\mathbb{E}\left[Z_{p, N}\right]-\left(\log _{1/p}((1-p)N)+\frac{\gamma}{\ln (1/p)}-\frac{1}{2}\right)\right| &< g_1(\theta)+o(1),\\
    \Pr\left(Z_{p, N} - \mathbb{E}\left[Z_{p, N}\right] > t\right) &\leqslant C_3p^t,\\
    \Pr\left(Z_{p, N} - \mathbb{E}\left[Z_{p, N}\right] < -t\right) &\leqslant \exp(-C_4p^{-t}) + \exp\left(\frac{-(1-p)N}{8}\right),
\end{align}
where $\gamma \approx 0.577$ is the Euler-Mascheroni constant, $g_1(\theta):=\frac{\theta^{1/2}}{2 \pi e^\theta(1-e^{-\theta})^2}$, and $C_3, C_4$ are positive constants.
\end{lemma}

From Lemma~\ref{lem:longest_run}, there exists a constant $C_3 > 0$ such that the upper tail probability is bounded by $\Pr(X - \mu > s) \leqslant C_3 p^s$ for any real $s > 0$.

\begin{remark}\label{rem:guibas_error_terms}
Regarding the $o(1)$ asymptotic error term in the expectation bound of Lemma~\ref{lem:longest_run}, the authors of \cite{guibas1980long} established that this residual converges to zero at least at a rate of $\mathcal{O}(\log(N)^3/N)$. Because this term decays rapidly as the number of worker nodes $N$ increases, its numerical contribution becomes negligible even for moderately small values of $N$. Consequently, this term can be securely upper-bounded by an absolute constant that is independent of $N$.
\end{remark}

For any non-negative integer-valued random variable $Y$, the expectation of $Y^4$ can be expressed via its survival function:
\begin{equation}
    \mathbb{E}[Y^4] = \sum_{y=1}^{\infty} \left(y^4 - (y-1)^4\right) \Pr(Y \geqslant y).
\end{equation}
Let $Y = X + 1$. Since $X$ takes values in $\{0, 1, 2, \dots\}$, $Y$ takes values in $\{1, 2, 3, \dots\}$. Substituting $x = y - 1$, we obtain:
\begin{equation}
    \mathbb{E}[(X+1)^4] = \sum_{x=0}^{\infty} \left((x+1)^4 - x^4\right) \Pr(X \geqslant x).
\end{equation}

To leverage the tail bound, we partition the infinite sum into a core component and a tail component. To strictly satisfy the condition $s > 0$ in the assumed bound, we define the integer split point as $m = \lfloor \mu \rfloor + 2$.
\begin{equation}
    \mathbb{E}[(X+1)^4] = \sum_{x=0}^{m} \left((x+1)^4 - x^4\right) \Pr(X \geqslant x) + \sum_{x=m+1}^{\infty} \left((x+1)^4 - x^4\right) \Pr(X \geqslant x).
\end{equation}

For the core sum, we apply the trivial probability bound $\Pr(X \geqslant x) \leqslant 1$. This yields a telescoping series:
\begin{equation}
    \sum_{x=0}^{m} \left((x+1)^4 - x^4\right) \Pr(X \geqslant x) \leqslant \sum_{x=0}^{m} \left((x+1)^4 - x^4\right) = (m+1)^4 - 0^4 = (m+1)^4.
\end{equation}
Because $m = \lfloor \mu \rfloor + 2 \leqslant \mu + 2$, we have:
\begin{equation}\label{eq:core_sum_bound}
    \sum_{x=0}^{m} \left((x+1)^4 - x^4\right) \Pr(X \geqslant x) \leqslant (\mu + 3)^4.
\end{equation}

For the tail sum, we analyze the term $\Pr(X \geqslant x)$ for $x \geqslant m + 1$. Let $t = x - m$, so that $x = m + t$. Since $x \geqslant m + 1$, the index $t$ ranges from $1$ to $\infty$. Because $X$ is integer-valued, the event $\{X \geqslant x\}$ is identical to $\{X > x - 1\}$. We rewrite the probability relative to the center $\mu$:
\begin{equation}
    \Pr(X \geqslant x) = \Pr(X > x - 1) = \Pr(X - \mu > x - 1 - \mu).
\end{equation}
Substitute $x = \lfloor \mu \rfloor + 2 + t$:
\begin{equation}
    x - 1 - \mu = (\lfloor \mu \rfloor - \mu) + t + 1.
\end{equation}
By the definition of the floor function, $\lfloor \mu \rfloor > \mu - 1$, which means $\lfloor \mu \rfloor - \mu > -1$. Therefore:
\begin{equation}
    (\lfloor \mu \rfloor - \mu) + t + 1 > -1 + t + 1 = t.
\end{equation}
Let $s = x - 1 - \mu$. We have established that $s > t \geqslant 1 > 0$. We can now apply the upper tail bound from Lemma~\ref{lem:longest_run}. Since $p \in (0, 1)$ and $s > t$, it strictly follows that $p^s < p^t$. Thus:
\begin{equation}
    \Pr(X \geqslant m + t) \leqslant \Pr(X - \mu > s) \leqslant C_3 p^s < C_3 p^t.
\end{equation}

We now substitute $x = m + t$ and this bound into the tail sum:
\begin{equation}
    \sum_{x=m+1}^{\infty} \left((x+1)^4 - x^4\right) \Pr(X \geqslant x) \leqslant \sum_{t=1}^{\infty} \Big( (m+t+1)^4 - (m+t)^4 \Big) C_3 p^t.
\end{equation}
Expanding the difference of the fourth powers yields:
\begin{equation}
    (m+t+1)^4 - (m+t)^4 = 4(m+t)^3 + 6(m+t)^2 + 4(m+t) + 1.
\end{equation}
Applying the binomial expansion to $(m+t)^k$ and grouping by powers of $m$, we get:
\begin{equation}
    (m+t+1)^4 - (m+t)^4 = 4m^3 + (12t+6)m^2 + (12t^2+12t+4)m + (4t^3+6t^2+4t+1).
\end{equation}
We distribute the summation and the constant $C_3 p^t$ across these terms. Defining the convergent arithmetico-geometric series $S_k = \sum_{t=1}^{\infty} t^k p^t$, we substitute $S_k$ into our expanded polynomial:
\begin{align}\label{eq:tail_sum_bound}
    \sum_{x=m+1}^{\infty} \left((x+1)^4 - x^4\right) \Pr(X \geqslant x) &\leqslant C_3 \Big[ 4m^3 \sum_{t=1}^{\infty} p^t + m^2 \sum_{t=1}^{\infty} (12t+6)p^t \nonumber \\
    &\quad + m \sum_{t=1}^{\infty} (12t^2+12t+4)p^t + \sum_{t=1}^{\infty} (4t^3+6t^2+4t+1)p^t \Big] \nonumber \\
    &= C_3 \Big[ 4S_0 m^3 + 6(2S_1 + S_0)m^2 \nonumber \\
    &\quad + 4(3S_2 + 3S_1 + S_0)m + (4S_3 + 6S_2 + 4S_1 + S_0) \Big].
\end{align}
Note that the constants $S_k$ depend only on $p$ and evaluate to finite values: $S_0 = \frac{p}{1-p}$, $S_1 = \frac{p}{(1-p)^2}$, $S_2 = \frac{p(1+p)}{(1-p)^3}$, and $S_3 = \frac{p(1+4p+p^2)}{(1-p)^4}$.

Combining the bounds for the core sum \eqref{eq:core_sum_bound} and the tail sum \eqref{eq:tail_sum_bound}, we arrive at the exact inequality:
\begin{align}\label{eq:exact_fourth_moment_bound}
    \mathbb{E}[(X+1)^4] &\leqslant (\mu + 3)^4 + C_3 \Big[ 4S_0 m^3 + 6(2S_1 + S_0)m^2 \nonumber \\
    &\quad + 4(3S_2 + 3S_1 + S_0)m + (4S_3 + 6S_2 + 4S_1 + S_0) \Big].
\end{align}

Finally, we establish the asymptotic bound. Since $m \leqslant \mu + 2$, the right-hand side of \eqref{eq:exact_fourth_moment_bound} is a polynomial in $\mu$ of degree $4$, meaning $\mathbb{E}[(X+1)^4] = \mathcal{O}(\mu^4)$. From Lemma~\ref{lem:longest_run}, we know that $\mu \leqslant \log_{1/p}((1-p)N) + \mathcal{O}(1) + o(1)$. According to Remark~\ref{rem:guibas_error_terms}, there exists an integer $n_0 \in \mathbb{N}$ such that for all $N > n_0$, the $o(1)$ term is strictly less than $1$. Therefore, for $N > n_0$, we can absorb this term to obtain $\mu \leqslant \log_{1/p}((1-p)N) + \mathcal{O}(1)$, which in turn implies $\mu^4 = \mathcal{O}(\log_{1/p}((1-p)N)^4)$. Consequently, there exists an absolute constant $C > 0$ such that for all $N > n_0$, the fourth moment is bounded by:
\begin{equation}
    \mathbb{E}[(R_{\stset}+1)^4] \leqslant C\log_{1/p}((1-p)N)^4.
\end{equation}
This concludes the proof.

\end{document}